\documentclass{article}

\usepackage[main, final]{neurips_2026}

\usepackage[utf8]{inputenc}
\usepackage[T1]{fontenc}
\usepackage{hyperref}
\usepackage{url}
\usepackage{booktabs}
\usepackage{amsfonts}
\usepackage{nicefrac}
\usepackage{microtype}
\usepackage[table]{xcolor}

\definecolor{myhighlight}{RGB}{220,220,220}

\usepackage[utf8]{inputenc} %
\usepackage[T1]{fontenc}    %
\usepackage{hyperref}       %
\hypersetup{
    colorlinks,
    breaklinks,
    linkcolor = blue,
    citecolor = blue,
    urlcolor  = blue,
}
\usepackage{url}            %
\usepackage{amsfonts}       %
\usepackage{nicefrac}       %
\usepackage{microtype}      %
\usepackage{amsmath,bm}
\usepackage{amsthm}
\usepackage[algo2e, ruled, vlined]{algorithm2e}
\usepackage{algorithmicx}
\usepackage{algorithm}
\usepackage[noend]{algcompatible}
\usepackage{bbm}      
\usepackage{mathtools}  \usepackage{matlab-prettifier}
\usepackage{amsopn}
\usepackage{amssymb}
\usepackage{graphicx}
\usepackage{xcolor}
\usepackage{footnote}
\usepackage{makecell}
\makesavenoteenv{tabular}
\makesavenoteenv{table}
\usepackage[table]{xcolor}
\usepackage{booktabs}

\usepackage{xcolor}
\usepackage{soul}
\usepackage{changepage}

\usepackage[nohints]{minitoc}

\newcommand{\Term}{\textsc{Term}}
\newcommand{\Err}{\textsc{Err}}
\newcommand{\Variance}{\texttt{Var}}

\newcommand{\Alg}{\textbf{Alg}}
\newcommand{\UnifPart}{\texttt{UnifPart}}

\newcommand{\PKL}{\mathrm{PKLCal}}

\newcommand{\Caldist}{\mathrm{Cal}}
\newcommand{\PCaldist}{\mathrm{PCal}}

\newcommand{\FSR}{\mathrm{PSReg}}

\newcommand{\ExtReg}{\mathrm{ExtReg}}

\newcommand{\optp}{p^\star}
\newcommand{\optq}{q^\star}

\sethlcolor{gray}

\newcommand{\Ind}[1]{ \field{I}{\left\{{#1}\right\}} }

\newcommand{\wh}[1]{\widehat{#1}}

\newcommand{\Betadist}{\texttt{Beta}}

\newcommand{\Clip}{\text{Clip}}

\newcommand{\naturalnum}{\mathbb{N}}

\newcommand{\bi}{\begin{itemize}}
\newcommand{\ei}{\end{itemize}}

\newcommand\numberthis{\addtocounter{equation}{1}\tag{\theequation}}

\theoremstyle{theorem} 
\newtheorem{theorem}{Theorem}[section]
\newtheorem{lemma}[theorem]{Lemma}

\newtheorem{corollary}[theorem]{Corollary}

\newtheorem{definition}[theorem]{Definition}
\newtheorem{condition}[theorem]{Condition}
\newtheorem{problem}[theorem]{Problem}

\usepackage{amssymb}%
\usepackage{pifont}%

\allowdisplaybreaks

\DeclareMathOperator*{\argmin}{\arg\!\min}

\usepackage{amsmath}
\usepackage{amssymb}
\usepackage{graphicx}
\usepackage{mathtools}
\usepackage{enumerate}
\usepackage{enumitem}
\usepackage{footnote}
\usepackage{float}
\usepackage{xspace}
\usepackage{multirow}
\usepackage{nicefrac}
\usepackage{wrapfig}
\usepackage{framed}
\usepackage{url}
\usepackage{tikz}
\usetikzlibrary{arrows,chains,matrix,positioning,scopes}

\newcommand{\calL}{{\mathcal{L}}}

\newcommand{\calS}{{\mathcal{S}}}

\newcommand{\calE}{{\mathcal{E}}}
\newcommand{\calR}{{\mathcal{R}}}
\newcommand{\calT}{{\mathcal{T}}}
\newcommand{\calP}{{\mathcal{P}}}
\newcommand{\calZ}{{\mathcal{Z}}}

\newcommand{\calN}{{\mathcal{N}}}
\newcommand{\calF}{{\mathcal{F}}}

\newcommand{\calQ}{{\mathcal{Q}}}

\newcommand{\loss}{\ell}

\newcommand{\field}[1]{\mathbb{#1}}

\newcommand{\fR}{\field{R}}

\newcommand{\E}{\field{E}}

\renewcommand{\P}{\field{P}}

\newcommand{\KL}{{\text{\rm KL}}}

\newcommand{\wt}{\widetilde}

\newcommand{\order}{\ensuremath{\mathcal{O}}}
\newcommand{\otil}{\ensuremath{\widetilde{\mathcal{O}}}}

\newcommand{\rbr}[1]{\left(#1\right)}

\newcommand{\sbr}[1]{\left[#1\right]}

\newcommand{\cbr}[1]{\left\{#1\right\}}

\newcommand{\abr}[1]{\left|#1\right|}

\DeclareFontFamily{OMX}{MnSymbolE}{}
\DeclareFontShape{OMX}{MnSymbolE}{m}{n}{
    <-6>  MnSymbolE5
   <6-7>  MnSymbolE6
   <7-8>  MnSymbolE7
   <8-9>  MnSymbolE8
   <9-10> MnSymbolE9
  <10-12> MnSymbolE10
  <12->   MnSymbolE12}{}
\DeclareSymbolFont{mnlargesymbols}{OMX}{MnSymbolE}{m}{n}
\SetSymbolFont{mnlargesymbols}{bold}{OMX}{MnSymbolE}{b}{n}
\DeclareMathDelimiter{\llangle}{\mathopen}{mnlargesymbols}{'164}{mnlargesymbols}{'164}
\DeclareMathDelimiter{\rrangle}{\mathclose}{mnlargesymbols}{'171}{mnlargesymbols}{'171}

\usepackage{prettyref}
\newcommand{\pref}[1]{\prettyref{#1}}

\newcommand{\savehyperref}[2]{\texorpdfstring{\hyperref[#1]{#2}}{#2}}
\newrefformat{eq}{\savehyperref{#1}{Eq.~\eqref{#1}}}
\newrefformat{eqn}{\savehyperref{#1}{Equation~\eqref{#1}}}
\newrefformat{lem}{\savehyperref{#1}{Lemma~\ref*{#1}}}
\newrefformat{def}{\savehyperref{#1}{Definition~\ref*{#1}}}
\newrefformat{line}{\savehyperref{#1}{line~\ref*{#1}}}
\newrefformat{thm}{\savehyperref{#1}{Theorem~\ref*{#1}}}
\newrefformat{cond}{\savehyperref{#1}{Condition~\ref*{#1}}}
\newrefformat{corr}{\savehyperref{#1}{Corollary~\ref*{#1}}}
\newrefformat{cor}{\savehyperref{#1}{Corollary~\ref*{#1}}}
\newrefformat{col}{\savehyperref{#1}{Corollary~\ref*{#1}}}
\newrefformat{sec}{\savehyperref{#1}{Section~\ref*{#1}}}
\newrefformat{app}{\savehyperref{#1}{Appendix~\ref*{#1}}}
\newrefformat{assum}{\savehyperref{#1}{Assumption~\ref*{#1}}}
\newrefformat{ex}{\savehyperref{#1}{Example~\ref*{#1}}}
\newrefformat{fig}{\savehyperref{#1}{Figure~\ref*{#1}}}
\newrefformat{tab}{\savehyperref{#1}{Table~\ref*{#1}}}
\newrefformat{alg}{\savehyperref{#1}{Algorithm~\ref*{#1}}}
\newrefformat{rem}{\savehyperref{#1}{Remark~\ref*{#1}}}
\newrefformat{conj}{\savehyperref{#1}{Conjecture~\ref*{#1}}}
\newrefformat{prop}{\savehyperref{#1}{Proposition~\ref*{#1}}}
\newrefformat{proto}{\savehyperref{#1}{Protocol~\ref*{#1}}}
\newrefformat{prob}{\savehyperref{#1}{Problem~\ref*{#1}}}
\newrefformat{claim}{\savehyperref{#1}{Claim~\ref*{#1}}}
\newrefformat{que}{\savehyperref{#1}{Question~\ref*{#1}}}
\newrefformat{obs}{\savehyperref{#1}{Observation~\ref*{#1}}}

\title{Adaptive Calibration in Non-Stationary Environments}

\author{Junyan Liu \\
  University of Washington \\
  \texttt{junyanl1@cs.washington.edu} \\
  \And
  Haipeng Luo \\
  University of Southern California \\
  \texttt{haipengl@usc.edu} \\
  \AND
  Lillian J. Ratliff \\
  University of Washington \\
 \texttt{ratliffl@uw.edu}
}

\begin{document}
\doparttoc
\faketableofcontents

\maketitle

\begin{abstract}
Making calibrated online predictions is a central challenge in modern AI systems. Much of the existing literature focuses on fully adversarial environments where outcomes may be arbitrary, leading to conservative algorithms that can perform suboptimally in more benign settings, such as when outcomes are nearly stationary. This gap raises a natural question: can we design online prediction algorithms whose calibration error automatically adapts to the degree of non-stationarity in the environment, smoothly interpolating between i.i.d. and adversarial regimes? We answer this question in the affirmative and develop a suite of algorithms that achieve adaptive calibration guarantees under multiple calibration measures. Specifically, with $T$ being the number of rounds, $W$ being the unknown number of i.i.d. segments of the environment, and $C\in[0,T]$ being another unknown non-stationary measure defined as the minimal $\ell_1$ deviation of the mean outcomes, our algorithms attain $\otil(\min\{\sqrt{T}+(TC)^{\frac{1}{3}}, \sqrt{WT}\})$ for $\ell_1$ calibration error and $\otil(\min\{(1+C)^{\frac{1}{3}}, W\})$ for both $\ell_2$ and pseudo KL calibration error. These bounds match the optimal rates in the stationary case ($C=0$ and $W=1$) and recover known guarantees in the fully adversarial regime ($C, W=\Omega(T)$). Our approach builds on and extends prior work~\citep{hu2025efficient, luo2025simultaneous}, introducing an epoch-based scheduling together with a novel non-uniform partition of the prediction space that allocates finer resolution near the underlying ground truth.
\end{abstract}

\section{Introduction}\label{sec:intro}
Calibrated predictions are a cornerstone of reliable decision-making in modern AI systems. In high-stakes applications such as medical diagnosis, weather forecasting, and risk assessment, it is often not enough for predictions to be accurate on average; they must also be \textit{well-calibrated}, meaning that predicted probabilities faithfully reflect empirical frequencies. For example, among all events assigned probability $0.7$, approximately $70\%$ should occur. This requirement and its variants have led to a rich literature on calibration, starting from~\citep{dawid1982well}.

In many real-world applications, predictions are made sequentially, and the underlying data-generating process may evolve over time. This motivates the study of \textit{online calibration}, where a forecaster repeatedly outputs probabilistic predictions and observes outcomes in a potentially changing environment. Starting from~\citet{foster1998asymptotic}, a prominent line of work analyzes this problem under the fully adversarial regime, providing algorithms that guarantee vanishing calibration error regardless of how the outcomes are generated. While these algorithms offer strong robustness guarantees, they are often too conservative at the same time and may fail to exploit benign structure when the environment exhibits regularity, such as stationarity or slow variation.

For example, when outcomes are i.i.d. generated, a simple algorithm that proceeds in epochs with doubling length and predicts the empirical average outcomes of all previous epochs enjoys $\order(\log T)$ $\ell_2$ calibration error after $T$ rounds (as we will show), even though it completely fails if outcomes are arbitrary.
On the other hand, the best existing algorithms, designed for the adversarial regimes, achieve $\otil(T^{1/3})$ $\ell_2$ calibration error generally and fail to adapt to i.i.d. data~\citep{foster2023calibeating, fishelson2025full, hu2025efficient}.
This motivates a natural and fundamental question that our work focuses on: \textit{can we design online calibration algorithms that adapt to the degree of non-stationarity in the environment?} Ideally, such algorithms would achieve strong worst-case guarantees while automatically improving their performance in more stationary settings, interpolating smoothly between i.i.d. and adversarial regimes without any prior knowledge of the environment.

\paragraph{Contributions.}
In this work, we answer this question affirmatively, 
developing a suite of algorithms with adaptive calibration guarantees under different calibration measures, including the standard $\ell_1$ and $\ell_2$ calibration error, denoted as $\Caldist_1$ and $\Caldist_2$, as well as the recently proposed (pseudo) KL calibration error $\PKL$~\citep{luo2025simultaneous} (see \pref{sec:pre} for formal definitions).
More specifically, 
our contributions are summarized below; see also \pref{tab:summary_result} for an overview of our results and comparisons with existing work.

\begin{itemize}[leftmargin=*]
\item 
To establish a baseline, we first consider the i.i.d.~setting, where $\Caldist_1 = \Omega(\sqrt{T})$ is a folklore; see e.g.,~\citet{qiao2021stronger}.
We extend the ideas to show an $\Omega(\log T)$ lower bound for both $\Caldist_2$ and $\PKL$.
We then show that all these bounds are tight up to logarithmic factors via a simple epoch-based algorithm that predicts the empirical average outcomes of previous epochs (or a clipped version of it in the case of $\PKL$) (see \pref{app:simple_block} for details).
It is clear that this algorithm does not work in the adversarial regime, and we include it solely to establish the benchmark rates attainable in the i.i.d. regime. 

\item To smoothly interpolate between i.i.d.~and adversarial regimes, we propose a non-stationarity measure $C \in [0,T]$ that is the minimal $\ell_1$ deviation of the mean outcomes and naturally shows up as an overhead due to non-stationarity in standard concentration bounds.
The i.i.d.~regime corresponds to $C=0$ while the adversarial regime corresponds to $C=T$.
Additionally, we also consider another natural non-stationarity measure $W \in \{1, \ldots, T\}$ which counts the number of i.i.d. segments of a piecewise stationary environment.
The exact values of $C$ and $W$ are unknown ahead of time, but we provide a reduction based on a doubling trick and a stationarity test
that converts any algorithm that requires the knowledge of $C$ into an algorithm that does not require the knowledge of $C$ or $W$, incurring only an additional logarithmic factor in the calibration error. 
We therefore assume the knowledge of $C$ in subsequent discussions.

\item Our first result (\pref{sec:cal1}) concerns $\Caldist_1$, arguably the most canonical calibration measure in the literature.
While it has been recently shown in a non-constructive way that $\Caldist_1=\order(T^{\frac{2}{3}-\epsilon})$ is achievable for a small fixed constant $\epsilon >0$~\citep{dagan2025breaking},
the best bound achieved by existing explicit algorithms is $\otil(T^{\frac{2}{3}})$.
In particular, we show that one of these algorithms, developed by~\citet{hu2025efficient}, is able to achieve $\Caldist_1 = \tilde{O}\big(\sqrt{T} + (T C)^{\frac{1}{3}}\big)$ (or more generally, $\Caldist_1 = \tilde{O}\big(\min\{\sqrt{T} + (T C)^{\frac{1}{3}}, \sqrt{WT}\}\big)$ after using the aforementioned reduction) when the parameter is set appropriately, which smoothly interpolates between the optimal $\otil(\sqrt{T})$ rate in the i.i.d.~regime and the existing $\otil(T^{\frac{2}{3}})$ rate in the adversarial regime.
While algorithmically this only requires changing the parameter's value, we emphasize that analytically this requires a novel proof that bounds the number of times a prediction is made in terms of its distance to a certain ground truth and the non-stationarity $C$.
This is a property specifically tied to the algorithm of~\citet{hu2025efficient}, and we are unable to prove the same adaptive $\Caldist_1$ bound using other algorithms.

\item Our next result (\pref{sec:cal2_approach1}) is a better bound for  $\Caldist_2$, which receives increasing interest recently since it is statistically easier to optimize compared to $\Caldist_1$ while still ensuring many strong decision-theoretic properties for downstream applications that $\Caldist_1$ enjoys.
We argue that in this case, the algorithm of~\citet{hu2025efficient} cannot achieve our goal by simply setting the parameter differently, since their uniform partition of the prediction space is too wasteful when $C$ is small.
Motivated by this, we propose a general epoch-based framework, where in each epoch, a carefully constructed non-uniform partition is created based on previous observations, so that finer resolution is allocated near the ground truth. 
We then apply the algorithm of~\citet{hu2025efficient} in each epoch using these non-uniform partitions and show that the final algorithm achieves $\Caldist_2 = \otil(\min\{(1+C)^{\frac{1}{3}}, W\})$, again matching the optimal $\otil(1)$ i.i.d.~rate and the best existing adversarial rate of $\otil(T^{\frac{1}{3}})$~\citep{foster2023calibeating, fishelson2025full, hu2025efficient}. 
In fact, our framework is versatile and also allows us to achieve the same $\Caldist_1$ result previously mentioned while improving the time complexity. 

\item Finally, in \pref{sec:PKL}, we further consider a recently proposed measure called (pseudo) KL calibration~\citep{luo2025simultaneous}, which is stronger than $\Caldist_2$ and enjoys an even broader range of applications.
\citet{luo2025simultaneous} show an $\otil(T^{\frac{1}{3}})$ upper bound for KL calibration in a non-constructive way and the same bound for the closely related pseudo KL calibration measure $\PKL$ using an explicit swap-regret-based approach.
We apply their algorithm to our epoch-based framework and show that the final algorithm achieves $\PKL = \otil(\min\{(1+C)^{\frac{1}{3}}, W\})$, matching the optimal i.i.d.~rate and the best known adversarial rate again.
This also shows the versatility of our epoch-based framework for adapting to non-stationarity.
Our analysis, however, is different from that for $\Caldist_2$ and requires a novel combination of the swap-regret guarantee and properties of our non-uniform partition.
\end{itemize}

\begin{table}[t] \label{tab:summary_result}
\centering
\caption{Comparison of calibration guarantees, with $T$ being the number of rounds, $W \in \{1,\ldots, T\}$ being the number of i.i.d.~segments of the environment, and $C\in[0,T]$ being a non-stationarity measure. Results from this paper are highlighted in gray, where the $W$-dependent bounds are due to a reduction from \pref{app:reduction_known2unknown}. The result $\Caldist_1=\otil(T^{\frac{2}{3}-\epsilon})$ by \citet{dagan2025breaking} holds for some small but unspecified constant $\epsilon>0$ without an explicit algorithm (thus no time complexity).
The best $\Caldist_1$ bound for explicit algorithms is $\otil(T^{\frac{2}{3}})$, achieved by for example those algorithms with $\Caldist_2=\otil(T^{\frac{1}{3}})$ since $\Caldist_1 \leq \sqrt{T\Caldist_2}$.}
\resizebox{\textwidth}{!}{
\begin{tabular}{ccccc}
\toprule
Metric & Reference & Bound &\makecell{Outcome \\ assumption}& \makecell{Time complexity \\ per round} \\
\midrule
\multirow{5}{*}{$\Caldist_1$} 
& folklore 
& $\Omega \big( \sqrt{T}  \big)$ &i.i.d. & -- \\
& \cellcolor{myhighlight} \pref{thm:simple_block_Cal1}  
& \cellcolor{myhighlight} $\otil \big( \sqrt{T}  \big)$ & \cellcolor{myhighlight}i.i.d. &  $ \cellcolor{myhighlight} \order(1)$\\
& \citet{dagan2025breaking}
& $\Omega \big( T^{0.54} \big)$ & none &-- \\
& \citet{dagan2025breaking}
& $\otil \big( T^{\frac{2}{3}-\epsilon} \big)$ & none &-- \\
&\cellcolor{myhighlight}  \pref{thm:CalL1_knownC} 
&\cellcolor{myhighlight} $\otil \big( \min \big\{\sqrt{T} +(TC)^{\frac{1}{3}},\sqrt{WT} \big\} \big)$ &\cellcolor{myhighlight}none &\cellcolor{myhighlight} 
$\otil \big(\min \big\{  (T^2/C)^{\frac{1}{3}},\sqrt{WT} \big\} \big)$  \\
&\cellcolor{myhighlight} \pref{thm:cal1_better} 
&\cellcolor{myhighlight} $\otil \big( \min \big\{\sqrt{T} +(TC)^{\frac{1}{3}},\sqrt{WT} \big\} \big)$ &\cellcolor{myhighlight}none &\cellcolor{myhighlight} 
$\otil \big( 1+ \min \big\{ C^{\frac{1}{3}},W^{\frac{1}{2}}T^{\frac{1}{6}} \big\} \big)$  \\
\midrule
\multirow{7}{*}{$\Caldist_2$} 
&\cellcolor{myhighlight} \pref{thm:lower_bound_Cal2}
&\cellcolor{myhighlight} $\Omega \big( \log T \big)$ &\cellcolor{myhighlight}i.i.d. &\cellcolor{myhighlight} -- \\
&\cellcolor{myhighlight} \pref{thm:simple_block_Cal2}  
&\cellcolor{myhighlight} $\order \big( \log^2 T \big)$ &\cellcolor{myhighlight}i.i.d. &\cellcolor{myhighlight} $\order(1)$ \\
& \citet{dagan2025breaking}
& $\Omega \big( T^{0.086} \big)$ & none &-- \\
& \citet{fishelson2025full}   
& $\otil \big( T^{\frac{1}{3}}  \big)$ & none & $\otil \big( T^{\frac{2}{3}}  \big)$\\
&  \citet{hu2025efficient}   
& $\otil \big( T^{\frac{1}{3}}  \big)$ & none & $\otil(T^{\frac{1}{3}})$  \\
&\cellcolor{myhighlight} \pref{thm:cal2_approach1}  
&\cellcolor{myhighlight} $\otil \big( \min \big\{(1+C)^{\frac{1}{3}} ,W\big\}  \big)$ &\cellcolor{myhighlight}none &\cellcolor{myhighlight} $\otil \big( \min \big \{  (1+C)^{\frac{1}{3}},W \big\}  \big)$ \\
&\cellcolor{myhighlight} \pref{corr:Cal2_bound_implied_by_PCal2} 
&\cellcolor{myhighlight} $\otil \big( \min \big\{(1+C)^{\frac{1}{3}},W \big\}  \big)$ &\cellcolor{myhighlight}none &\cellcolor{myhighlight} $\otil \big( \min \big \{  (1+C)^{\frac{2}{3}},W^2 \big\}  \big) $   \\
\midrule
\multirow{4}{*}{$\PKL$} 
& \cellcolor{myhighlight} \pref{corr:lower_bound_PKL}
&\cellcolor{myhighlight} $\Omega \big( \log T \big)$ &\cellcolor{myhighlight}i.i.d. &\cellcolor{myhighlight} -- \\
&\cellcolor{myhighlight} \pref{thm:simple_block_PKL}
&\cellcolor{myhighlight} $\order \big( \log^2 T  \big)$ &\cellcolor{myhighlight}i.i.d. &\cellcolor{myhighlight} $\order(1)$\\
& \citet{luo2025simultaneous}  
& $\otil \big( T^{\frac{1}{3}}  \big)$ & none &$\otil \big( T^{\frac{2}{3}}  \big) $ \\
&\cellcolor{myhighlight} \pref{thm:PKLCal_bound_knownC} 
&\cellcolor{myhighlight} $\otil \big( \min \big\{(1+C)^{\frac{1}{3}},W \big\}  \big)$ &\cellcolor{myhighlight}none & \cellcolor{myhighlight} $\otil \big( \min \big \{  (1+C)^{\frac{2}{3}},W^2 \big\}  \big)$ \\
\bottomrule
\end{tabular}
}
\end{table}

\paragraph{Related work.}
There is a vast literature on calibration; we only focus on results for the most relevant non-contextual online setting here.
\citet{foster1998asymptotic} are the first to show that $\Caldist_1=\otil(T^{\frac{2}{3}})$ is possible even without any assumption on how the outcomes are generated.
For more than two decades, the best lower bound for this problem remains to be $\Omega(\sqrt{T})$, coming from a folklore construction using i.i.d. data.
The first breakthrough in bridging this gap is by~\citet{qiao2021stronger}, who show a surprising $\Omega(T^{0.528})$ lower bound for $\Caldist_1$.
This is recently improved to $\Omega(T^{0.54})$ by~\citet{dagan2025breaking}, who, as mentioned, also develop an $\otil(T^{\frac{2}{3}-\epsilon})$ upper bound (for some small constant $\epsilon > 0$) in a non-constructive way, illustrating that the minimax rate for this problem is somewhere between $\Omega(T^{0.54})$ and  $\otil(T^{\frac{2}{3}-\epsilon})$, a surprising phenomenon in online learning.

Since the work of~\citet{foster1998asymptotic}, there have been many other calibration measures proposed in the literature to overcome different issues of $\Caldist_1$.
For example, one of the reasons that $\Caldist_2$ attracts increasing interest recently is that it is statistically easier to optimize while still sharing many decision-theoretic properties of $\Caldist_1$. 
Indeed, \citet{foster2023calibeating, fishelson2025full, luo2025simultaneous, hu2025efficient} show that $\Caldist_2 = \otil(T^{\frac{1}{3}})$ (a much better rate than those for $\Caldist_1$) is achievable, and this implies that all downstream agents using these calibrated predictions to make their decisions enjoy $\otil(T^{\frac{1}{3}})$ swap regret for any bounded proper losses with a smooth univariate form.
Similarly, \citet{luo2025simultaneous} also propose (pseudo) KL calibration, which upper bounds $\Caldist_2$ (thus stronger) and implies swap regret guarantees for any proper losses with
a twice continuously differentiable univariate form (that is not necessarily smooth).
Despite being stronger, \citet{luo2025simultaneous} show that the same bound $\otil(T^{\frac{1}{3}})$ is also achievable for KL calibration.
There are currently no better lower bounds for $\Caldist_2$ and KL calibration than simply converting the lower bound $\Caldist_1 = \Omega(T^{0.54})$ into $\Caldist_2 = \Omega(T^{0.086})$ via the connection $\Caldist_1 \leq \sqrt{T\Caldist_2}$.

While we only consider these three calibration measures (due to their fundamental roles in calibration), we believe that our work serves as an important first step in developing more adaptive calibration algorithms, a direction that has not been explored yet to our knowledge.
We also expect that our epoch-based framework with non-uniform partition to be useful for studying other calibration measures, such as smooth calibration~\citep{kakade2008deterministic} and its subsample version~\citep{haghtalab2024truthfulness}, distance from calibration~\citep{blasiok2023unifying, qiao2024distance},  U-calibration~\citep{kleinberg2023u, luo2024optimal}, and calibration decision loss~\citep{hu2024calibration}. 

Finally, we point out while our specific problem has not been studied before, adapting to non-stationarity is generally a heavily-studied topic in online learning.
In particular, the non-stationarity measure we use in this work (that is, minimal $\ell_1$ deviation) is most related to similar measures used in~\citet{jin2023no, liu2026online} for reinforcement learning.

\textbf{Discussion with concurrent and independent work \citep{huang2026instance}.} Concurrent with our work, \citet{huang2026instance} also study adaptive calibration in non-stationary environments. They propose an alternative algorithm achieving a bound of $\Caldist_1=\otil \big(\min \big\{T^{2/3}, \sqrt{WT} \big\} \big)$ which interpolates between the stationary and fully adversarial regimes. 
This result is incomparable to the bound $\Caldist_1=\otil(\sqrt{T}+(TC)^{1/3})$ established in the first version of our paper.
After becoming aware of each other's work, however, we observed that both approaches can in fact yield both guarantees for $\Caldist_1$.
Specifically, by adapting \citep[Algorithm 2]{liu2026online} to the calibration setting, any algorithm designed for \textit{known} $C$ can be extended to simultaneously adapt to both unknown $W$ and unknown $C$. 
This powerful reduction, originally introduced by \citet{liu2026online} for the uninformed Markov games, readily applies here; see \pref{app:reduction_known2unknown} for details. Applying this framework to our results immediately yields $\Caldist_1=\otil \big(\min \big\{  \sqrt{T}+(TC)^{1/3} ,\sqrt{WT}\big\} \big)$ and improved per-round time complexity $\otil \big( 1+ \min \big\{ C^{\frac{1}{3}},W^{\frac{1}{2}}T^{\frac{1}{6}} \big\} \big)$, along with analogous improvements for $\Caldist_2$ and $\PKL$.
Beyond these overlapping results, the two works pursue different directions. Our paper develops a general framework that adapts to a broader class of calibration metrics, while focusing on the non-contextual setting. In contrast, \citet{huang2026instance} extend their algorithmic approach to the contextual setting (that is, multicalibration), but focus exclusively on the $\Caldist_1$ metric.

\section{Preliminaries} \label{sec:pre}

\textbf{Notations.} For an integer $N \in \naturalnum$, we denote $[N]=\{1,\ldots,N\}$.
Let $\calF_t$ be the history up to and including round $t$.
We use $\E_t$ to denote the expectation conditioning on $\calF_{t-1}$.
For a set $\calS$, we use $\Delta(\calS)$ to denote the set of all probability distributions over $\calS$.
We use $\texttt{Ber}(p)$ to denote the Bernoulli distribution with mean $p$.
We write $\KL(p,q)=p\log(p/q)+(1-p)\log((1-p)/(1-q))$ to denote the KL divergence of two Bernoulli distributions with means $p,q$. For any $a \leq b$, we write $\Clip_{[a,b]}(x)=\min\{\max\{x,a\},b\}$.
Throughout the paper, we use $\iota$ to denote a factor of order $\Theta(\log(T/\delta))$ where $\delta$ is a failure probability.

\paragraph{Problem setup.}
We consider the following sequential prediction problem between a forecaster and an adversary.
Ahead of time, the adversary decides $\{q_t\}_{t \in [T]}$ where $q_t \in [0,1]$ is chosen based on the knowledge of the forecaster's algorithm.\footnote{We consider such an oblivious adversary for ease of presentation; our analysis directly generalizes to an adaptive adversary who decides $q_t$ based on the history $\calF_{t-1}$.}
At each round $t \in [T]$, the adversary draws an outcome $y_t \sim \texttt{Ber}(q_t)$, and simultaneously, the forecaster predicts $p_t \in [0,1]$, the probability of the outcome being $1$, after which the true outcome $y_t$ is revealed to the forecaster.
It is well known that randomness is generally required to achieve calibration, and we let $\calP_t \in \Delta([0,1])$ be the forecaster's conditional distribution of $p_t$.

Given outcomes $y_1,\ldots,y_T$ and predictions $p_1,\ldots,p_T$, the forecaster's $\ell_r$-calibration error for $r\geq 1$ is defined as 
\begin{equation}\label{eq:cal}
\Caldist_r := \sum_{\substack{p \in [0,1]: \\ n(p)> 0}} n(p) \abr{ \frac{1}{n(p)} \sum_{t=1}^T \rbr{ y_t -p_t } \Ind{p_t = p} }^r
= \sum_{\substack{p \in [0,1]: \\ n(p)> 0}} n(p) \abr{ \frac{1}{n(p)} \sum_{t: p_t = p} y_t - p}^r,
\end{equation}
where $n(p) = \sum_{t=1}^T \Ind{p_t=p}$ is the number of times prediction $p$ is made. 
Note that the term $\abr{ \frac{1}{n(p)} \sum_{t: p_t = p} y_t - p}^r$ measures how close the forecaster's prediction $p$ is compared to the empirical average outcomes conditioning on all the rounds where the prediction $p$ is made.
In this paper, we focus on $r$ being $1$ or $2$, arguably the two most studied values.  
A closely related concept is \textit{pseudo} calibration error, first mentioned explicitly in~\citet{luo2025simultaneous}.
In particular, they show that minimizing $\Caldist_2$ can be achieved by minimizing the pseudo $\ell_2$-calibration error, defined as $\PCaldist_2:= \sum_{t=1}^T \E_{p \sim \calP_t}\big[ \rbr{\bar{\rho}_p-p}^2 \big]$, where $\bar{\rho}_p = \frac{ \sum_{t=1}^T \calP_t(p) y_t}{\sum_{t=1}^T \calP_t(p)}$. 
In other words, compared to \pref{eq:cal}, the pseudo version removes the random variable $p_t$ by replacing $n(p)$ with a conditional expected version $\sum_{t=1}^T \calP_t(p)$ and similarly $\frac{1}{n(p)} \sum_{t: p_t = p} y_t$ with $\bar{\rho}_p$, making it intuitively easier to minimize.

Another natural calibration measure of interest is KL-calibration~\citep{luo2025simultaneous}, which replaces the $\ell_r$ distance  with the KL divergence. 
\citet{luo2025simultaneous} show that this is stronger than $\Caldist_2$, but only provide an explicit algorithm for the pseudo version, defined as
$
    \PKL = \sum_{t=1}^T \E_{p \sim \calP_t } \sbr {   \KL \rbr{ \bar{\rho}_p , p } }.
$
Since our work focuses on designing explicit algorithms with non-stationarity adaptivity, we also only consider this pseudo version of KL-calibration.
  
\paragraph{Non-stationarity measure.}
To quantify the non-stationarity of the adversary, we propose the following minimal $\ell_1$ deviation measure:
\begin{equation} \label{eq:def4corrupt}
    C := \sum_{t=1}^T c_t,\quad \text{where } c_t= \abr{q_t  - \optp} \text{ and } \optp \in \argmin_{q \in [0,1]} \sum_{t=1}^T  \abr{ q_t  - q} .
\end{equation}
It is clear that $\optp$ is the median of the mean outcome sequence $q_1, \ldots, q_T$ (which we also sometimes refer to as the ``ground truth''), and thus $C$ can be viewed as the total amount of deviation from the median.
Such an $\ell_1$ deviation naturally shows up in concentration bounds (see e.g.,~\pref{lem:corruption_analysis_Cal2_Approach1}), which is also why it has been used in prior work for different problems~\citep{jin2023no, liu2026online}.
It is clear that $C$ is between $0$ and $T$:
when $C=0$, all $q_t$'s are equal and the outcomes $y_1, \ldots, y_T$ are i.i.d. samples of a fixed Bernoulli distribution, making the problem much easier than the worst-case when $q_t$'s are arbitrary and $C=\Omega(T)$.

Additionally, inspired by~\cite{huang2026instance}, we also consider another natural stationarity measure $W=1+\sum_{t=1}^{T-1} \Ind{q_t \neq q_{t+1}}$, which counts the number of i.i.d. segments of the environment.

\paragraph{Baselines for $C=0$.}
To identify the appropriate target for calibration guarantees that adapt to $C$, we need to examine the right bounds in the two extremes: $C=0$ and $C=\Omega(T)$.
While the latter has been the focus in the literature (which we review in \pref{sec:intro} already), we are unable to find explicit discussions for the former.
Therefore, as a baseline, in \pref{app:simple_block}, we establish matching lower and upper bounds for all three calibration measures we consider.
Specifically, we show that a simple epoch-based algorithm that predicts the average outcomes of previous epochs (or a clipped version of it for the case of $\PKL$) achieves $\otil(\sqrt{T})$ for $\Caldist_1$ and $\order(\log^2 T)$ for both $\Caldist_2$ and $\PKL$, matching the corresponding lower bounds up to logarithmic factors.
All our final adaptive calibration guarantees recover these results when $C=0$.

\paragraph{From unknown $C$ and $W$ to known $C$.}
While our goal is to adapt to the unknown non-stationarity level $C$ and the unknown number of stationary segments $W$, it suffices to first consider the case where $C$ is known. This is because the forecaster observes the outcome at the end of each round and can therefore monitor its calibration error online. Thus, adapting to unknown $C$ alone only requires a standard doubling trick.
To further adapt to unknown $W$, we borrow the idea of \citet{liu2026online} and
equip the doubling trick with stationarity tests. Each failed test indicates that the corresponding interval contains a change point, so the total number of failed tests reflects the number of stationary segments. With a suitable choice of the number of tests, the resulting framework adapts simultaneously to both unknown $C$ and unknown $W$ while incurring only additional logarithmic factors.
See \pref{app:reduction_known2unknown} for details. In light of this reduction, the rest of the paper focuses on the known-$C$ setting and the $C$-dependent bounds only.

\section{Achieving $\Caldist_1=\otil(\sqrt{T}+(TC)^{1/3})$} \label{sec:cal1}
We consider $\Caldist_1$ in this section.
It is natural to start with an existing algorithm that achieves $\Caldist_1 = \otil(T^{\frac{2}{3}})$ and study what modifications are needed to adapt to non-stationarity.
While there are several such algorithms in the literature,
we are only able to find one that achieves our goal: the algorithm of~\citet{hu2025efficient}.
We briefly review their algorithm in a partition-generic form (important for subsequent discussions on $\Caldist_2$); see \pref{alg:cal1_knownC} in the appendix. The algorithm takes as input a finite partition $\Pi=\{J_1,\ldots,J_N\}$ of the prediction space $[0,1]$, where each interval $J_i$ is associated with a grid prediction $z_{J_i}:=\sup_{x\in J_i}x.$
For $\Caldist_1$, we instantiate $\Pi$ as the uniform partition with discretization size $N\in\naturalnum$, namely $J_i=  [\frac{i-1}{N},\frac{i}{N} )$ for each $i \in [N-1]$ and $J_N=[1-\frac{1}{N},1]$.
In this case, $z_{J_i}=i/N$ for each $i\in[N]$.
The algorithm then runs an expert subroutine \textsc{MsMwC} \citep{chen2021impossible} over $2N$ experts
indexed by $(i,\sigma)\in [N]\times\{\pm 1\}$.
At each round $t$, the algorithm invokes \textsc{MsMwC} to obtain the expert weights $\{\omega_{t,i,\sigma}\}_{(i,\sigma)}$, which further induces a distribution $P_t$ over $[0,1]$ such that $\max_{y \in \{0,1\}}\E_{u \sim P_t}\big[\sum_{(i,\sigma)}  \omega_{t,i,\sigma} \cdot \sigma  \Ind{ u \in J_i} (u-y)   \big]= O(\nicefrac{1}{T})$.
The forecaster samples $\wt{p}_t\sim P_t$ and predicts $p_t =z_{J_{i_t}} = \frac{i_t}{N}$, where $i_t\in[N]$ is the unique index such that
$\wt{p}_t\in J_{i_t}$.
After observing $y_t$, the expert algorithm
is updated using the expected loss $  \loss_{t,i,\sigma} =    \E_{u \sim P_t} \sbr{ \sigma \Ind{u \in J_i} (u-y_t) \mid \calF_{t-1} }$.

To achieve $\Caldist_1 = \otil(T^{\frac{2}{3}})$, \citet{hu2025efficient} set the partition size $N$ to be of order $T^{\frac{1}{3}}$.
Our first main result is to show that simply changing the partition size $N$ to order $\min\cbr{\sqrt{T}, (T^2/C)^{\frac{1}{3}}}$ achieves our goal, even though it requires a novel refined analysis to prove so.

\begin{theorem} \label{thm:CalL1_knownC}
\pref{alg:cal1_knownC} with a uniform partition $\Pi$ for some appropriate value of $N$ (see \pref{eq:choice_N_cal1}) ensures that with probability at least $1-\delta$, $\Caldist_1 = \order \rbr{ \sqrt{\iota T \log T} + (\iota TC \log T)^{\frac{1}{3}}  }$. The computational complexity of the algorithm is $\otil \big(\min \big\{ \sqrt{T},(T^2/C)^{\frac{1}{3}} \big\}\big)$ per round.
\end{theorem}

\textbf{Proof sketch.}
We outline the key ideas of the proof for \pref{thm:CalL1_knownC} and highlight the novelty of our analysis, despite the algorithm being almost the same.
We refer readers to \pref{app:proof_Cal1_theorem} for details and the analysis of computational complexity.
As $p_t \in \{z_J\}_{J \in \Pi}$ for all $t \in [T]$ and each interval $J$ uniquely corresponds to a point $z_J$, we rewrite $\Caldist_1$ as $\sum_{J \in \Pi}      \big| \sum_{t=1}^T \rbr{ y_t - z_{J} } \Ind{p_t = z_J } \big|$.
Let $n_J = \sum_{t=1}^T \Ind{p_t = z_J}$ denote the number of times prediction $z_J$ is made and $\Delta_J=\sup_{u,v\in J}|u-v|$.
We start with the following lemma to bound each $\big|  \sum_{t=1}^T \rbr{ y_t - z_J } \Ind{p_t = z_J } \big|$, which is adapted from \citet[Section 2.1]{hu2025efficient} and presented in a more general form.

\begin{lemma} \label{lem:abs_cal_error_byHu}
For any fixed partition $\Pi=\{J_i\}_{i \in [N]}$ of $[0,1]$ with $N=\order(\text{poly}(T))$, with probability at least $1-\delta/2$,
$
  \forall J \in \Pi:\     \big| \sum_{t=1}^T \rbr{ y_t - z_{J} } \Ind{p_t = z_{J} } \big| \leq \order \rbr{  \sqrt{  \iota n_{J}} + \iota + n_{J} \Delta_{J} }.
$
\end{lemma}

From here, summing over all $J\in\Pi$ and applying Cauchy-Schwarz inequality with $\sum_J n_{J}=T$ would bound $\Caldist_1$ by $\sum_{J \in \Pi} \otil \big( \sqrt{  n_{J}} +  \frac{n_{J}}{N}  \big) \leq \otil \big( \sqrt{NT} +\frac{T}{N} \big)$ (which then leads to $\otil(T^{\frac{2}{3}})$ with $N$ set appropriately).
However, the clear issue of doing so is that the bound does not adapt to $C$.
To address this, we propose a refined analysis showing that the number of times a prediction $z_J$ is made (that is, $n_J$) is controlled by its distance to $\optp$ as well as the non-stationarity level $C$: the larger the distance, the smaller the $n_J$; on the other hand, the larger the non-stationarity, the larger the $n_J$.
Specifically,
let $j^* \in [N]$ be such that $\optp \in J_{j^*}$ and its neighborhood be $\calN= \{j^*-1,j^*,j^*+1\} \cap [N]$. For any interval $J$, further define
$d_{J} = \inf_{u \in J} \abr{\optp-u}$.
Then we have:

\begin{lemma} \label{lem:visit_number_control_Cal1_main}
For any fixed partition $\Pi=\{J_i\}_{i \in [N]}$ of $[0,1]$ with $N=\order(\text{poly}(T))$, with probability at least $1-\delta/2$,
\begin{equation} \label{eq:bound4ni}
\forall i \in [N] \backslash  \{j^*\}:\  \sqrt{n_{J_i}} \leq   \order \rbr{  \frac{\sqrt{\iota}}{d_{J_i}} + \sqrt{\frac{C_{J_i}}{d_{J_i}}  }    }, \ \text{where}\ \forall J \in \Pi: C_J = \sum_{t=1}^T c_t  \calP_t \rbr{z_{J}}.
\end{equation}
\end{lemma}

\pref{lem:visit_number_control_Cal1_main} thus allows us to do a more careful analysis as:
\begin{align*}
\Caldist_1 &\leq  \otil \rbr{\sum_{i \in \calN}  \sqrt{ n_{J_i}}  +\sum_{i \in [N] \backslash \calN}  \sqrt{ n_{J_i}} +\frac{T}{N}   } \leq   \otil \rbr{\sqrt{T }  +\sum_{i \in [N] \backslash \calN}    \rbr{\frac{1}{d_{J_i}} +\sqrt{\frac{C_{J_i}}{d_{J_i}}} }  +\frac{T}{N}  } .
\end{align*}

Further, for any $i \in  [N] \backslash \calN$,  $|i-j^*|\geq 2$ and $d_{J_i} \geq \frac{|i-j^*|-1}{N} \geq \frac{|i-j^*|}{2N}$ hold.
Thus, we have $\sum_{i \in  [N] \backslash \calN}  \frac{1}{d_{J_i}}  \leq \order \rbr{     \sum_{i \in  [N] \backslash \calN}  \frac{N}{|i-j^*|}   }  \leq \order \rbr{     \sum_{z=2}^N  \frac{N}{z}   } \leq \otil \rbr{ N}$, and also (by Cauchy-Schwarz inequality) $\sum_{i \in  [N] \backslash \calN}  \sqrt{C_{J_i}/d_{J_i}}   \leq   \sqrt{  \sum_{i \in  [N] \backslash \calN} \frac{1}{d_{J_i}}}  \sqrt{  \sum_{i \in  [N] \backslash \calN} C_{J_i}} \leq \otil  \big(  \sqrt{NC}  \big)$. Hence, we arrive at $\Caldist_1 \leq \otil \big(  \sqrt{T } +\frac{T}{N} + \sqrt{ NC}     \big)$. Choosing $N$ optimally completes the proof.

We note that \pref{lem:visit_number_control_Cal1_main} is tied to the algorithm's property (in particular, its usage of the \textsc{MsMwC} algorithm), and as mentioned, we are unable to achieve similar bounds using other algorithms.
In particular, we also point out that while $\Caldist_1 = \otil(T^{\frac{2}{3}})$ can be achieved by showing $\Caldist_2 = \otil(T^{\frac{1}{3}})$ since $\Caldist_1 \leq \sqrt{T\Caldist_2}$,
we cannot achieve our $\Caldist_1=\otil(\sqrt{T}+(TC)^{\frac{1}{3}})$ result using our $\Caldist_2 = \otil((1+C)^{\frac{1}{3}})$ results from the following sections in the same way --- doing so only achieves $\Caldist_1=\otil(\sqrt{T}(1+C)^{\frac{1}{6}})$, which is strictly weaker.

\section{Achieving $\Caldist_2 = \otil \rbr{(1+C)^{1/3} }$}
 \label{sec:cal2_approach1}

Given that the algorithm of \citet{hu2025efficient} also achieves $\Caldist_2 = \otil(T^{\frac{1}{3}})$, it is natural to suspect that by appropriately setting $N$ in terms of $C$ again, their algorithm might also achieve an adaptive $\Caldist_2$ guarantee.
However, we emphasize that it is highly unclear whether this is feasible.
Indeed, since $\Caldist_2$ can be written as $\sum_{J \in \Pi} \frac{1}{n_J}\big| \sum_{t=1}^T \rbr{ y_t - z_{J} } \Ind{p_t = z_J } \big|^2$,
we can bound each summand using \pref{lem:abs_cal_error_byHu} again by $\otil(1 + n_J \Delta_J^2)$, but this is already too loose and unable to lead to $C$-dependent bounds because
summing over $J\in\Pi$ would lead to exactly $\otil(N + \frac{T}{N^2})$ (as $\Delta_J = \frac{1}{N}$ for a uniform partition and $\sum_{J\in\Pi} n_J = T$).

\setcounter{AlgoLine}{0}
\begin{algorithm}[t]
\DontPrintSemicolon
\caption{General framework with non-uniform partitions for non-stationarity adaptation}
\label{alg:framework_non_uniform}
\textbf{Input}: confidence $\delta \in (0,1)$, non-stationarity $C$, base algorithm $\Alg$, number of grids $\{N_m\}_m$ for inner regions, and number of grids $K$ for each outer band.

Let $s_m = 2^m$. Predict arbitrarily for the first epoch that lasts for $s_1$ rounds.

\For{epoch $m=2,\ldots$}{

Let $\wh{y}_m$ be the average outcome of the previous epoch.

Define inner region $I_m=[a_m, b_m] =  \sbr{ \wh{y}_m - 2r_m,\wh{y}_m + 2r_m  } \cap [0,1]$ with  $r_m=\sqrt{ \frac{ \iota  }{s_{m-1}} }  + \frac{C}{s_{m-1}}$.

Define inner partition $\Pi_{m}^{\text{in}} = \UnifPart(I_m,N_m)$ where $\UnifPart$ is from \pref{def:unifpart}.

Define left and right outer bands as (for each $q=0,\ldots,Q$ where $Q = \lceil \log_2 T \rceil$):
\begin{equation} \label{eq:outer_division_Cal2}
\begin{aligned}
        L_{m,q} &= [a_m -(2^{q+1}-1) r_m , a_m -(2^{q}-1) r_m ) \cap [0,1],\\
        R_{m,q} &=  (  b_m +( 2^{q} -1)r_m , b_m + (2^{q+1}-1) r_m ] \cap [0,1].
\end{aligned}
\end{equation}

Define
$\calL_{m,q} = \UnifPart(L_{m,q},K)$ and $\calR_{m,q} = \UnifPart(R_{m,q},K)$ for each $q=0,\ldots,Q$.

Define outer partition $\Pi_{m}^{\text{out}}= \{\calL_{m,q}\}_{q=0}^{Q} \cup \{\calR_{m,q}\}_{q=0}^{Q}$ and overall partition $\Pi_{m} = \Pi_{m}^{\text{in}} \cup \Pi_{m}^{\text{out}}$.

Run $\Alg$ over the partition $\Pi_m$ from scratch for $s_m$ rounds.
}
\end{algorithm}

\paragraph{High-level idea.}
We suspect that this is not fixable with a uniform partition since $\ell_2$ distance is more sensitive than $\ell_1$ and requires a partition with finer resolution around the ground truth $\optp$, especially when $C$ is small.
Indeed, since $n_J$ is larger when $J$ is closer to $\optp$ (\pref{lem:visit_number_control_Cal1_main}),
if we could accordingly use a shorter length $\Delta_J$ (hence finer resolution), then the term $n_J \Delta_J^2$ could be overall smaller.
In particular, for the special case of $C=0$ such that $n_J$ is of order $1/d_J^2$ based on \pref{lem:visit_number_control_Cal1_main},
we would want $d_J$ and $\Delta_J$ to be of the same order so that $n_J\Delta_J^2 = \otil(1)$, leading to a non-uniform partition with $\Delta_J$ growing exponentially as one moves away from the region containing $\optp$, both to the left and to the right.
To handle the case with a general $C$, we could further uniformly partition the aforementioned intervals to provide better resolution (since for a large $C$, these intervals could have a large width).
Finally, since $\optp$ is unknown, it is natural to deploy an epoch-based schedule to track the region where $\optp$ lies in using an exponentially increasing amount of data.
Putting all these ideas together leads to our general framework presented in \pref{alg:framework_non_uniform}.

\paragraph{General framework with non-uniform partition.}
\pref{alg:framework_non_uniform}  proceeds in epochs with doubling length.
In each epoch, we run a certain base calibration algorithm (from scratch) for $s_m = 2^m$ rounds using a partition $\Pi_{m}$ constructed based on observations from the previous epoch.
This base algorithm does not have to be the one of~\citet{hu2025efficient}; indeed, we will instantiate it differently for different results to be presented.
The construction of $\Pi_{m}$ is as follows (see also \pref{fig:my_figure} for an illustration).

\begin{itemize}[leftmargin=*]
\item
(\textbf{Inner region and partitions})
First, we create an inner region $I_m=[a_m, b_m]$ that covers $\optp$ with high probability and uniformly partition it into $\Pi_{m}^{\text{in}}$ with size $N_m$ that is supposed to be relatively large.
This inner region $I_m$ is defined as $\sbr{ \wh{y}_m - 2r_m,\wh{y}_m + 2r_m  } \cap [0,1]$, where $\wh{y}_m$ is the empirical average outcome of epoch $m-1$ and $r_m = \sqrt{ \frac{ \iota  }{s_{m-1}} }  + \frac{C}{s_{m-1}}$ is such that $\abr{\wh{y}_m -\optp} \leq  r_m$ holds with high probability based on standard concentration (\pref{lem:corruption_analysis_Cal2_Approach1}).

\item
(\textbf{Outer region, bands, and partitions})
We refer to $[0,a_m)$ and $(b_m,1]$ as the left and right outer regions respectively.
They are handled in the same way, so below we focus on discussing the left outer region.
First, we partition it into a few \textit{bands} $\{L_{m,q}\}$ for $q$ ranging from $0$ to at most $Q =\lceil\log_2 T\rceil$.
The first band $L_{m,0}$ is right next to the inner region and has length $r_m$, while the remaining bands have a length that is doubling as $q$ increases; see \pref{eq:outer_division_Cal2} for the formal definition.\footnote{In this definition, we allow $q$ to be as large as $Q = \lceil \log_2 T\rceil$, which is an overestimate of the number of bands. It makes the presentation more concise and some bands empty, but clearly does not affect the algorithm.}
However, as discussed, merely using these logarithmically many bands is too coarse for a large $C$, and we thus further uniformly partition each band $L_{m,q}$ into $\calL_{m,q}$ with size $K$.
We denote the collection of $\calL_{m,q}$ as well as the counterpart $\calR_{m,q}$ for the right outer region as the final outer partition $\Pi_{m}^{\text{out}}$.
The final overall partition to be used in this epoch is $\Pi_{m} = \Pi_{m}^{\text{in}} \cup \Pi_{m}^{\text{out}}$.
\end{itemize}

\begin{figure}[t]
    \centering
    \includegraphics[width=0.9\textwidth]{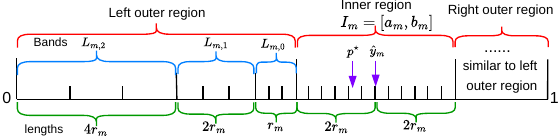}
    \caption{Illustration of the non-uniform partition created by \pref{alg:framework_non_uniform} for epoch $m$.}
    \label{fig:my_figure}
\end{figure}

\paragraph{Main results for $\Caldist_2$.}
Our result for $\Caldist_2$ is obtained by applying the algorithm of~\citet{hu2025efficient} as the base algorithm in our general framework \pref{alg:framework_non_uniform} and picking the parameters appropriately.

\begin{theorem} \label{thm:cal2_approach1}
With  $\Alg$ instantiated as \pref{alg:cal1_knownC}, $N_m=\big \lceil \rbr{s_m}^{1/3} |I_{m}|^{2/3} \iota^{-1/3} \big \rceil$,
and $K = \big\lceil \rbr{ (1+C)/ \iota }^{1/3}  \big \rceil$,
\pref{alg:framework_non_uniform} ensures that with probability at least $1-\delta$, $\Caldist_2 = \order \rbr{  C^{1/3}\iota^{2/3} \log^2 T     + \iota  \log^2  T  }.$ The computational complexity per round is $\otil \big( (1+C)^{1/3}\big)$.
\end{theorem}

\paragraph{Proof sketch.}
We fix an epoch $m$ and illustrate the idea to bound the $\ell_2$-calibration error for this epoch, defined as $\Caldist_2^{(m)} = \sum_{J \in \Pi_m:n_{m,J}>0 } \frac{1}{n_{m,J} }\rbr{   \sum_{t \in \calT_m } \rbr{ y_t -p_t } \Ind{p_t = z_{J}} }^2$ where $n_{m,J} = \sum_{t \in \calT_m} \Ind{p_t=z_{J}}$ and $\calT_m$ denotes the set of all the $s_m=2^m$ rounds in this epoch.

Based on \pref{lem:abs_cal_error_byHu} and the fact
$|\Pi_{m}|=|\Pi_{m}^{\text{in}}|+|\Pi_{m}^{\text{out}}| = N_m + \otil(K)$, we proceed as
\begin{align*}
    \Caldist_2^{(m)} \leq \otil \rbr{   |\Pi_{m}|  + \sum_{ J \in \Pi_{m} }  n_{m,J}  \Delta_J^2    } = \otil \rbr{ N_{m}  + \sum_{ J \in \Pi_{m}^{\text{in}} }  n_{m,J}  \Delta_J^2   + K  + \sum_{ J \in \Pi_{m}^{\text{out}} }  n_{m,J}  \Delta_J^2   }.
\end{align*}

Using the facts $\Delta_J = \frac{|I_{m}|}{N_{m}}$ for any $J \in \Pi_{m}^{\text{in}}$
and $\sum_{J \in \Pi_{m}^{\text{in}}} n_{m,J} \leq s_{m}$, the sum of the first two terms above is at most $\otil   \rbr{N_{m}  +s_{m} \frac{|I_{m}|^2}{N_{m}^2}}$,
which explains the choice of the value of $N_{m}$ since it minimizes the last bound to $\otil \rbr{  s_{m}^{1/3}  |I_{m}|^{2/3}   } = \otil \rbr{  (1+C)^{1/3}  }$ (last step uses the definition of $r_{m}$, roughly the width of $I_{m}$).

It remains to bound $\sum_{ J \in \Pi_{m}^{\text{out}} }  n_{m,J}  \Delta_J^2$, which can be decomposed as $ \sum_{ q =0}^{Q} \sum_{J \in \calL_{m,q}}  n_{m,J}  \Delta_J^2    +\sum_{ q =0}^{Q} \sum_{J \in \calR_{m,q}}  n_{m,J}  \Delta_J^2$ (note that the summation over $J$ is naturally and implicitly only over those nonempty $J$'s).
Due to symmetry, we only focus on the term related to the left outer region and apply the high-level idea mentioned at the beginning of the section:
use \pref{lem:visit_number_control_Cal1_main} (more specifically, an epoch version of it stated in \pref{lem:visit_number_control_Cal2_block_main}) to bound
$n_{m,J}$ by $\order \rbr{ \frac{\iota}{d_{J}^2} +\frac{C_{m,J}}{d_{J}}}$ where $C_{m,J} = \sum_{t\in \calT_m} c_t  \calP_t \rbr{z_{J}}$;
then, using the facts $\Delta_J \leq \frac{2^q r_{m}}{K}$  and $d_J \geq 2^q r_{m}$ for any nonempty $J \in \calL_{m,q}$,
we see that the order of $\sum_{J \in \calL_{m,q}}  n_{m,J}  \Delta_J^2$ is
\begin{align*}
&\sum_{J \in \calL_{m,q}}   \rbr{ \frac{1}{d_{J}^2} +\frac{C_{m,J}}{d_{J}}   }\rbr{   \frac{2^q r_m}{K} }^2
\leq\sum_{J \in \calL_{m,q}} \rbr{ \frac{1}{ (2^q r_m)^2 } +\frac{C_{m,J}}{2^q r_m}   } \rbr{   \frac{2^q r_m}{K} }^2  \\
& \leq \sum_{J \in \calL_{m,q}}\frac{1}{ K^2 } +\frac{C_{m,J} 2^q r_m}{K^2}    \leq   1 +\frac{C}{K^2}   ,
\end{align*}
where
the last step uses $|\calL_{m,q}|= K$, $\sum_{J \in \calL_{m,q}} C_{m,J} \leq C$, and $2^q r_{m} \leq d_J \leq  1$.
Noting that $Q = \otil(1)$ and combining with the earlier term $|\Pi_{m}^{\text{out}}|=\otil(K)$,
we finally arrive at $K  + \sum_{ J \in \Pi_{m}^{\text{out}} }  n_{m,J}  \Delta_J^2
=\otil\rbr{K + \frac{C}{K^2}}$.
This explains the choice of the value of $K$, which minimizes the last bound to $\otil \rbr{  (1+C)^{1/3} }$.
Finally, since $\Caldist_2$ is subadditive (\pref{corr:decomp_Cal2_into_block}) and there are at most $\order(\log T)$ epochs, summing over all epochs proves our result.

Finally, we mention that with a different value of $K$, this algorithm in fact also achieves $\Caldist_1 = \otil(\sqrt{T}+(TC)^{\frac{1}{3}})$, the same guarantee achieved by the simpler algorithm from \pref{sec:cal1}.
The advantage of using the more complicated \pref{alg:framework_non_uniform}, however, is that it enjoys better time complexity that in particular is only $\otil(1)$ per round when $C=0$, as opposed to $\otil(\sqrt{T})$;
see \pref{thm:cal1_better} for details.
This further illustrates the usefulness of our framework.

\section{Achieving $\PKL=\otil \big( (1+C)^{1/3} \big)$}\label{sec:PKL}

Finally, in this section, we switch our focus to pseudo KL calibration, a notion that is strong than $\Caldist_2$ but still enjoys an $\otil(T^{\frac{1}{3}})$ bound in the worst case~\citep{luo2025simultaneous}.
The only existing algorithm is through minimizing \textit{pseudo swap regret} ($\FSR$) with respect to log loss, which we show can be used in combination of our general framework \pref{alg:framework_non_uniform} to achieve $\PKL=\otil \big( (1+C)^{1/3} \big)$.
The algorithm of~\citet{luo2025simultaneous} builds on the idea of \citet{fishelson2025full}, who showed that minimizing $\FSR$ with respect to the squared loss is equivalent to minimizing pseudo $\ell_2$-calibration error $\PCaldist_2$ (recall the definition from \pref{sec:pre}).
Since these two works share very similar ideas
but the algorithm/analysis for $\PKL$ is more involved,
we will use $\PCaldist_2$ as a warm-up to illustrate our key ideas, with most details related to $\PKL$ deferred to the appendix.

More specifically, our \pref{alg:framework_non_uniform}
achieves $\PCaldist_2=\otil((1+C)^{1/3})$ when  instantiated with the algorithm of \citet{fishelson2025full} as the base algorithm (and $N_m$ and $K$ set to similar values as \pref{thm:cal2_approach1}). We briefly review the algorithm of \citet{fishelson2025full} in the context of our non-uniform partition. Given a partition $\Pi_m$ of $[0,1]$ in epoch $m$ and letting $\calZ_m$ be the set of all the endpoints of the intervals from $\Pi_m$,
each point $s \in \calZ_m$ is associated with an instance of online gradient descent (OGD). At each round $t$, each OGD instance $s \in \calZ_m$ outputs an action $a_{s,t} \in [0,1]$, which is then randomized and leads to a certain distribution $H(a_{s,t}) \in \Delta(\calZ_m)$ supported on the two endpoints of the interval $J_{s,t}$ containing $a_{s,t}$.
Next, the algorithm constructs a row-stochastic matrix $\calQ_t \in \mathbb{R}^{|\calZ_m|\times|\calZ_m|}$, where $\calQ_t(s,\cdot)=H(a_{s,t})$, and computes its stationary distribution $\calP_t$.
Finally, the algorithm samples a prediction $p_t$ from $\calP_t$, and each OGD instance $s \in \calZ_m$ is  updated using the weighted squared loss $\calP_t(s)\loss(a_{s,t},y_t)$ where $\ell(a,y) = (a-y)^2$.

\textbf{Analysis.}
Fix an epoch $m$.
The pseudo $\ell_2$-calibration error for this epoch is equivalent to
the pseudo swap regret for this epoch: $\FSR^{(m)}  = \sup_{\phi:  [0,1] \to  [0,1]} \sum_{t \in \calT_m} \E_{p \sim \calP_t  }  \sbr{ \rbr{\loss(p,y_t)  - \loss(\phi(p),y_t)} }$.
Following an analysis similar to \citet{fishelson2025full}, we can obtain $ \FSR^{(m)} \leq \otil \big(|\Pi_m|+ \sum_{t \in \calT_m} \sum_{s \in \calZ_m}   \calP_t(s) \Delta^2_{J_{s,t}}\big)$, which is very similar to the $\otil \rbr{   |\Pi_{m}|  + \sum_{ J \in \Pi_{m} }  n_{m,J}  \Delta_J^2    }$ term in the analysis of \pref{thm:cal2_approach1}.
Indeed, we can also decompose it into an inner-region-related term $|\Pi_m^{\text{in}}| + \sum_{t \in \calT_m} \sum_{s \in \calZ_m}  \calP_t(s) \Delta^2_{J_{s,t}} \Ind{J_{s,t}  \in \Pi_m^{\text{in}}}$,
which can be handled in the exact same way as \pref{thm:cal2_approach1},
and an outer-region-related term $|\Pi_m^{\text{out}}| + \sum_{t \in \calT_m} \sum_{s \in \calZ_m}  \calP_t(s) \Delta^2_{J_{s,t}} \Ind{J_{s,t}  \in \Pi_m^{\text{out}}}$,
which requires a different analysis since $J_{s,t}$ is the interval that contains prediction $a_{s,t}$ instead of $s$
and thus there is no clear relationship between $\calP_t(s) $ and $\Delta_{J_{s,t}}$.
Instead, we propose to first bound the term $\sum_{t \in \calT_m} \sum_{s \in \calZ_m}  \calP_t(s)  \Delta^2_{J_{s,t}}  \Ind{J_{s,t}  \in \Pi_m^{\text{out}} }$ by $\otil \rbr{ \sum_{t \in \calT_m} \sum_{s \in \calZ_m}  \frac{\calP_t(s)  (a_{s,t}-\optp)^2 }{K^2} }$ due to $\Delta_{J_{s,t}} \leq \frac{|a_{s,t}-\optp|}{K}$, and then apply the following critical lemma that is our key novelty.
\begin{lemma} \label{lem:bound_occupancy_PCal2}
With probability at least $1-\delta$, for any epoch $m$, we have $\sum_{t \in \calT_m}   \sum_{s \in \calZ_m}  \calP_t(s) (a_{s,t}-\optp)^2 \leq \order \rbr{|\Pi_m|\log T +C_m +\iota }$ where $C_m = \sum_{t \in \calT_m} c_t$.
\end{lemma}
This lemma is proven using the no-regret guarantee of OGD and a certain self-bounding argument; see \pref{app:PCal2}.
Put together, this shows that the outer-region-related term $|\Pi_m^{\text{out}}| + \sum_{t \in \calT_m} \sum_{s \in \calZ_m}  \calP_t(s) \Delta^2_{J_{s,t}} \Ind{J_{s,t}  \in \Pi_m^{\text{out}}}$ is of order $\otil\rbr{\frac{C_m}{K^2}+K}$, which, when summed over all epochs and plugged in the (optimal) value of $K$, is of the desired order $\otil((1+C)^{1/3})$.

\paragraph{Extension to $\PKL$}
Finally, we briefly mention the extra elements needed to achieve the same bound for $\PKL$.
The key difference is that, instead of working directly on the prediction space $[0,1]$, the base algorithm operates on the transformed domain $[0,\pi]$, and maps each point $z$ back to $[0,1]$ via $\sin^2(z/2)$. This transformation aligns the geometry of the problem with the KL divergence.
Recall that in each epoch $m$, the inner and outer regions are determined by the empirical estimate $\wh{y}_{m}$. To adapt our framework to this transformed setting, we instead work with the transformed estimate $\theta(\wh{y}_{m})$, where for any $y \in [0,1]$, $\theta(y) = 2\arcsin(\sqrt{y})$, and then construct a confidence interval around $\theta(\wh{y}_{m})$ to define the inner and outer regions.
See \pref{app:PKL} for the complete algorithm/proof.

\section{Conclusion and Future Work}

In this paper, we study online calibration in non-stationary environments.
We consider a natural non-stationarity measure $W$ that is the number of i.i.d.~segments of the environment and also propose another non-stationarity measure $C$, defined as the minimal $\ell_1$ deviation of the mean outcomes $q_1,\ldots,q_T$.
We develop algorithms whose calibration guarantees adapt to both $W$ and $C$ without requiring their knowledge in advance.
For $\Caldist_1$, we show that an existing efficient calibration algorithm can be refined to achieve $\otil( \min \{\sqrt{T}+(TC)^{1/3},\sqrt{WT}\} )$.
For $\Caldist_2$ and $\PKL$, we introduce an epoch-based framework with non-uniform partitions that allocate finer resolution near the underlying ground truth, leading to $\otil(\min \{(1+C)^{1/3},W\})$ guarantees.
We further show that the same framework can also recover the adaptive $\Caldist_1$ guarantee while improving the per-round time complexity to $\otil(1+\min \{C^{1/3},\sqrt{W}T^{1/6}\})$.
These bounds recover the optimal stationary rates up to logarithmic factors when $C=0$ or $W=1$. They also match the best known adversarial rates when $C$ and $W$ are linear in $T$.

Our work leaves several directions for future investigation. First, it would be interesting to extend the framework to other calibration notions, such as smooth calibration~\citep{kakade2008deterministic} and its subsample version~\citep{haghtalab2024truthfulness}, distance from calibration~\citep{blasiok2023unifying, qiao2024distance},  U-calibration~\citep{kleinberg2023u, luo2024optimal}, and calibration decision loss~\citep{hu2024calibration}. Second, while our results establish adaptive upper bounds for several fundamental calibration measures, sharper lower bounds as a function of $C$ and $W$ would help clarify the optimality of these rates in the intermediate non-stationary regimes. Finally, extending these ideas to contextual, multiclass, or more structured prediction settings is an important next step toward a more general theory of adaptive calibration.

\paragraph{Acknowledgment}
HL thanks Spandan Senapati for helpful discussion on the stationary setting.

\bibliographystyle{abbrvnat}
\bibliography{refs}

\newpage
\clearpage
\appendix
\begingroup
\part{Appendix}
\let\clearpage\relax
\parttoc[n]
\endgroup

\newpage

\section{Omitted Proofs in \pref{sec:pre}} \label{app:reduction}

\subsection{Lower Bounds for $\Caldist_2,\PCaldist_2$, and $\PKL$ for $C=0$}
In this section, we reduce the problem of lower bounding $\Caldist_2$ and $\PKL$ to the problem of lower bounding the squared error for estimating an unknown Bernoulli mean.
For all problems in this section with $C=0$, an instance is completely specified by the unknown
Bernoulli mean $q\in[0,1]$. Therefore, the usual minimax risk $\inf_{\Alg}\sup_{\text{instance}}$ can be written equivalently as $\inf_{\Alg} \sup_{q \in [0,1]}$.

\begin{problem} \label{prob:Bernoulli_est_mean_problem}
Ahead of time, $y_1,\ldots,y_T \overset{i.i.d.}{\sim} \texttt{Ber}(q)$  where $q \in [0,1]$ is unknown.
    At each round $t=1,\ldots,T$, the forecaster predicts $p_t \in [0,1]$ based on past observations $y_1,\ldots,y_{t-1}$ and then observes $y_t$. The goal is to minimize $\sum_{t=1}^T\E[ (p_t-q)^2]$ where the expectation is taken over the algorithm's randomness and observed samples.
\end{problem}

\begin{theorem} \label{thm:Bernoulli_est_mean_problem_lower_bound}
For the problem stated in \pref{prob:Bernoulli_est_mean_problem},  $\inf_{\text{Alg}} \sup_{q \in [0,1]}  \sum_{t=1}^T \E_{\Alg,q} [ (p_t-q)^2] \geq \Omega(\log T)$.
\end{theorem}

\begin{proof}
Fix an arbitrary algorithm $\Alg$.
We introduce a prior $H \sim \texttt{Unif}[0,1]$, and then $y_t|H \sim \texttt{Ber}(H)$. We have
\[
\sup_{q \in [0,1]}  \sum_{t=1}^T \E_{\Alg,q} [ (p_t-q)^2] \geq  \E_{H \sim \texttt{Unif}[0,1] } \sbr{ \sum_{t=1}^T \E_{\Alg,H} [ (p_t-H)^2]   }.
\]

For any forecaster, conditioning on the past $y_1,\ldots,y_{t-1}$, the best possible squared-error prediction of $H$ is the posterior mean. Thus,
\[
 \E[(p_t-H)^2 ] \geq \E \sbr{ \Variance(H|y_1,\ldots,y_{t-1}) }.
\]

For shorthand, let $S_{t}=\sum_{i \leq t}y_i$. Under the uniform prior, the posterior is $H|S_{t-1}=s \sim \Betadist(s+1,t-s)$. Thus,
\[
\Variance(H|S_{t-1}=s) = \frac{(s+1)(t-s)}{(t+1)^2 (t+2)}.
\]

Under the uniform prior, $S_{t-1}$ is uniform on $\{0,\ldots,t-1\}$, and thus
\[
\P(S_{t-1}=s) = \binom{t-1}{s} \int_{0}^1 q^s(1-q)^{t-1-s} dq =\frac{1}{t},
\]
which implies that
\[
\E_{H} \sbr{ \Variance(H|S_{t-1}) } = \frac{1}{t} \sum_{s=0}^{t-1}   \frac{(s+1)(t-s)}{(t+1)^2 (t+2)} = \frac{1}{6(t+1)}.
\]

Therefore,
\[
 \E_{H \sim \texttt{Unif}[0,1] } \sbr{ \sum_{t=1}^T \E_{\Alg,H} [ (p_t-H)^2]   } \geq \sum_{t=1}^T \frac{1}{6(t+1)} \geq \Omega(\log T).
\]

The claimed result thus follows.
\end{proof}

\begin{theorem} \label{thm:lower_bound_Cal2}
For $C=0$, $\inf_{\text{Alg}} \sup_{q \in [0,1]} \E_{\text{Alg},q} \big[ \Caldist_2 \big] \geq \Omega(\log T)$.
\end{theorem}

\begin{proof}
Fix an arbitrary algorithm $\Alg$.
For each prediction $\alpha$, we define
\[
T(\alpha) = \cbr{t: p_t=\alpha},\qquad n(\alpha) = |T(\alpha)|,\qquad  \bar{y}(\alpha) = \frac{1}{n(\alpha)} \sum_{t \in T(\alpha)} y_t.
\]

Recall that $\Caldist_2 =\sum_{\alpha: n(\alpha)>0} n(\alpha) (\bar{y}(\alpha)-\alpha)^2$.
For any $\alpha$ such that $n(\alpha)>0$, we have
\begin{align*}
n(\alpha) (\bar{y}(\alpha)-\alpha)^2  &= \sum_{t \in T(\alpha)} (y_t -\alpha)^2 -\sum_{t \in T(\alpha)} (y_t -\bar{y}(\alpha))^2    \\
 &=  \sum_{t \in T(\alpha)} (y_t -\alpha)^2 -\min_{z \in [0,1]}\sum_{t \in T(\alpha)} (y_t -z)^2  .
\end{align*}

Summing over all $\alpha$ such that $n(\alpha)>0$,
\begin{align*}
    \Caldist_2  &=  \sum_{\alpha: n(\alpha)>0} \sum_{t \in T(\alpha)} (y_t -\alpha)^2-  \sum_{\alpha: n(\alpha)>0}  \min_{z \in [0,1]}\sum_{t \in T(\alpha)} (y_t -z)^2  \\
    &=  \sum_{t =1}^T (y_t - p_t)^2- \sum_{\alpha: n(\alpha)>0}  \min_{z \in [0,1]}\sum_{t \in T(\alpha)} (y_t -z)^2 \\
    &\geq   \sum_{t =1}^T (y_t - p_t)^2-   \min_{z \in [0,1]} \sum_{\alpha: n(\alpha)>0} \sum_{t \in T(\alpha)} (y_t -z)^2 \\
     &=   \sum_{t =1}^T (y_t - p_t)^2-   \min_{z \in [0,1]} \sum_{t=1}^T (y_t -z)^2.
\end{align*}

For any $q \in [0,1]$, we have
\[
 \Caldist_2 \geq \sum_{t =1}^T (y_t - p_t)^2-  \sum_{t=1}^T(y_t -q)^2.
\]

Taking expectation under the stochastic model $y_t \sim \texttt{Ber}(q)$ and using $\E_{\Alg,q}[\sum_{t =1}^T (y_t - p_t)^2-  \sum_{t=1}^T(y_t -q)^2 |\calF_{t-1},p_t] = (p_t-q)^2$ ($p_t$ is chosen before $y_t$), we have
\begin{align*}
    \E_{\Alg,q}[ \Caldist_2] \geq \sum_{t=1}^T \E_{\Alg,q}  \sbr{ (p_t-q)^2 } .
\end{align*}

It suffices to give a lower bound for the
\[
\inf_{\text{Alg}} \sup_{q \in [0,1]}  \sum_{t=1}^T \E_{\text{Alg},q} [ (p_t-q)^2] .
\]

Invoking \pref{thm:Bernoulli_est_mean_problem_lower_bound} completes the proof.
\end{proof}

\begin{theorem} \label{thm:lower_bound_PCal2}
For $C=0$, $\inf_{\text{Alg}} \sup_{q \in [0,1]} \E_{\text{Alg},q} \big[ \PCaldist_2 \big] \geq \Omega(\log T)$.
\end{theorem}
\begin{proof}
Fix an arbitrary algorithm $\Alg$.
Recall that $\bar{\rho}_p = \frac{ \sum_{t=1}^T \calP_t(p) y_t}{\sum_{t=1}^T \calP_t(p)}$.
We write for any $q \in [0,1]$
\begin{align*}
   \PCaldist_2 &= \sum_{t=1}^T \E_{p \sim \calP_t} \sbr{  \rbr{p-\bar{\rho}_p}^2 } \\
   &= \sum_{t=1}^T \E_{p \sim \calP_t} \sbr{  \rbr{p-y_t}^2 } -\sum_{t=1}^T \E_{p \sim \calP_t} \sbr{  \rbr{y_t-\bar{\rho}_p}^2 } \\
    &\geq  \sum_{t=1}^T \E_{p \sim \calP_t} \sbr{  \rbr{p-y_t}^2 } -\sum_{t=1}^T \E_{p \sim \calP_t} \sbr{  \rbr{y_t-q}^2 } .
\end{align*}

Suppose that $y_t \sim \texttt{Ber}(q)$ for all $t$. As $\calP_t$ is chosen before $y_t$, we have
\begin{align*}
   & \E_{\Alg,q} \sbr{ \E_{p \sim \calP_t} \sbr{   \rbr{p-y_t}^2-  \rbr{y_t-q}^2  } |\calF_{t-1}   } \\
    &= \E_{p \sim \calP_t} \sbr{ \E_{\Alg,q} \sbr{    \rbr{p-y_t}^2-  \rbr{y_t-q}^2  |\calF_{t-1} ,p}   } \\
      &= \E_{p \sim \calP_t} \sbr{ (p-q)^2  }.
\end{align*}

Thus, we have $ \E_{\Alg,q}[ \PCaldist_2] \geq \sum_{t=1}^T \E_{\Alg,q}  \sbr{ \E_{p \sim \calP_t} \sbr{ (p-q)^2  }} = \sum_{t=1}^T \E_{\Alg,q}  \sbr{  (p_t-q)^2  }$. Invoking \pref{thm:Bernoulli_est_mean_problem_lower_bound} completes the proof.
\end{proof}

\begin{corollary} \label{corr:lower_bound_PKL}
For $C=0$,
    $\inf_{\text{Alg}} \sup_{q \in [0,1]} \E_{\text{Alg},q} \big[ \PKL \big] \geq  \Omega(\log T)$.
\end{corollary}
\begin{proof}
Since $\PKL \geq \PCaldist_2$ by \citep[Proposition 3]{luo2025simultaneous}, the lower bound on $\PCaldist_2$ in \pref{thm:lower_bound_PCal2} immediately implies the claimed result.
\end{proof}

\subsection{Technical Lemmas}

We here first show that the calibration measures used in this paper indeed satisfy the subadditivity.

\begin{lemma} \label{lem:toollem4decomp_Cal}
    Let $\{\calT_m\}_{m=1}^M$ be a partition of $[T]$ and let $\calZ$ be a finite set. For every $t \in [T]$, $\omega_t(z) \geq 0$ for all $z \in \calZ$.
    Let $N_z = \sum_{t=1}^T \omega_t(z)$ and $N_{m,z} = \sum_{t \in \calT_m} \omega_t(z)$.
    Let $F_z: [0,1] \to \fR \cup \{+\infty\}$ be convex for every $z \in \calZ$. Then,
    \[
    \sum_{z \in \calZ:  N_z>0 } N_z F_z\rbr{ \frac{\sum_{t=1}^T \omega_t(z) y_t  }{ N_z}  } \leq \sum_{m=1}^M    \sum_{z \in \calZ:  N_{m,z}>0 } N_{m,z} F_z\rbr{ \frac{\sum_{t \in \calT_m} \omega_t(z) y_t  }{ N_{m,z}}  } .  
    \]
\end{lemma}
\begin{proof}
Consider any $z \in \calZ$ such that $N_z >0$. We have 
\begin{align*}
&\sum_{z \in \calZ:  N_z>0 } N_z F_z\rbr{ \frac{\sum_{t=1}^T \omega_t(z) y_t  }{ N_z}  }  \\
&=\sum_{z \in \calZ:  N_z>0 } N_z F_z\rbr{\sum_{m\in [M]: N_{m,z}>0}  \frac{N_{m,z}}{N_z}   \frac{\sum_{t \in \calT_m} \omega_t(z) y_t  }{ N_{m,z}  }  }  \\
&\leq \sum_{z \in \calZ:  N_z>0 } N_z  \sum_{m\in [M]: N_{m,z}>0}   \frac{N_{m,z}}{N_z}  F_z\rbr{  \frac{\sum_{t \in \calT_m} \omega_t(z) y_t  }{ N_{m,z}  }  }  \\
&= \sum_{z \in \calZ:  N_z>0 }  \sum_{m\in [M]: N_{m,z}>0}  N_{m,z}  F_z\rbr{  \frac{\sum_{t \in \calT_m} \omega_t(z) y_t  }{ N_{m,z}  }  }  \\
&=\sum_{m=1}^M    \sum_{z \in \calZ:  N_{m,z}>0 } N_{m,z} F_z\rbr{ \frac{\sum_{t \in \calT_m} \omega_t(z) y_t  }{ N_{m,z}}  } ,
\end{align*}
where the inequality follows from Jensen's inequality. The claimed result thus follows.
\end{proof}

For any interval $I \subseteq [T]$, let $n_{I,\alpha}=\sum_{t \in I} \Ind{p_t=\alpha}$ and 
\[
\Caldist_1(I)= \sum_{\alpha \in [0,1]: n_{I,\alpha}>0}  \abr{ \sum_{t \in  I } \rbr{ y_t -p_t } \Ind{p_t = \alpha} } .
\]

We also similarly define $\Caldist_2(I),\PCaldist_2(I),\PKL(I)$. When, $I=[T]$, we have $\Caldist_1([T]) =\Caldist_1$, $\Caldist_2([T]) =\Caldist_2$, $\PCaldist_2([T]) =\PCaldist_2$, and $\PKL([T]) =\PKL$.

\begin{corollary} \label{corr:decomp_Cal2_into_block}
For any partition $\{I_i\}_{i}$ of $[T]$.
$\Caldist_1  \leq \sum_{i} \Caldist_1(I_i)$.
$\Caldist_2  \leq \sum_{i} \Caldist_2(I_i)$, $\PCaldist_2 \leq \sum_{i} \PCaldist_2(I_i)$, and $\PKL \leq \sum_{i} \PKL(I_i)$.
\end{corollary}
\begin{proof}
We first show $\Caldist_1  \leq \sum_{i} \Caldist_1(I_i)$.
Recall that $n_{\alpha}=\sum_{t \in [T]} \Ind{p_t=\alpha}$.
\begin{align*}
\Caldist_1 
&= \sum_{\alpha \in [0,1]: n_{\alpha}>0}  \abr{  \sum_{i}  \sum_{t \in I_i} \rbr{ y_t -p_t } \Ind{p_t = \alpha} } \\
&\leq \sum_{\alpha \in [0,1]: n_{\alpha}>0}   \sum_{i}  \abr{   \sum_{t \in I_i} \rbr{ y_t -p_t } \Ind{p_t = \alpha} } \\
&= \sum_{i} \sum_{\alpha \in [0,1]: n_{I,\alpha}>0}    \abr{   \sum_{t \in I_i} \rbr{ y_t -p_t } \Ind{p_t = \alpha} } \\
& =\sum_{i}  \Caldist_1(I_i).
\end{align*}

Setting $\omega_t(z)=\Ind{p_t=z}$ and $F_z(y)=(y-z)^2$ in \pref{lem:toollem4decomp_Cal} gives $\Caldist_2 \leq \sum_{i} \Caldist_2(I_i)$.
Setting $\omega_t(z)=\calP(z)$ and $F_z(y)=(y-z)^2$ in \pref{lem:toollem4decomp_Cal} gives $\PCaldist_2 \leq \sum_{i} \PCaldist_2(I_i)$.
Setting $\omega_t(z)=\calP(z)$ and $F_z(y)=\KL(y,z)$ in \pref{lem:toollem4decomp_Cal} gives $\PKL \leq \sum_{i} \PKL(I_i)$.
\end{proof}

\subsection{Adapting to Unknown $C$ and Piecewise Stationary Environments}
\label{app:reduction_known2unknown}

\setcounter{AlgoLine}{0}
\begin{algorithm}[t]
\DontPrintSemicolon
\caption{Framework of adapting to unknown $C$ and piecewise stationary environments}
\label{alg:meta_alg_piecewise}
\textbf{Input}: confidence $\delta \in (0,1)$, an algorithm $\Alg$ that satisfies \pref{cond:anytime_adaptive_alg} with function $U_{\cdot}(\cdot,\cdot)$, error measure $\Err$, sub-block scheduling $\{G_b\}_b$.

Define $\delta' =\frac{\delta}{10T^2 \lceil \log_8(8+T) \rceil }$.

\For{block $b=1,2,\ldots$}{

\For{$\ell=1,2\ldots,G_b+1$}{

Set the guess $\wh{C}^{(b)} = 2^{3b}$.

If $\ell\leq G_b$, restart $\Alg$ with input $(\delta', 0 )$; otherwise, restart $\Alg$ with input $(\delta', \wh{C}^{(b)})$.

Initialize $\calT_{\ell}^{(b)} \gets  \emptyset$, $\Err \big(\calT_{\ell}^{(b)} \big) \gets 0$ and $D \gets 0$.

\While{
$\Err \big(\calT_{\ell}^{(b)} \big) \leq D$
}{

Let $t$ be the current round.

Run $\Alg$ to make a prediction $p_t$ and observe outcome $y_t$.

Update $\calT_{\ell}^{(b)} \gets \calT_{\ell}^{(b)} \cup \{t\}$ and $\Err \big(\calT_{\ell}^{(b)} \big)$.

If $\ell\leq G_b$, set $D \gets  U_{\delta'}\big(\calT_{\ell}^{(b)}, 0 \big)$; otherwise, set $D \gets  U_{\delta'}\big([T], \wh{C}^{(b)} \big)$.

}

}
}    
\end{algorithm}

In this section, we present a generic adaptive-restart framework, \pref{alg:meta_alg_piecewise}, that converts
known-\(C\) base algorithms into algorithms that are simultaneously adaptive to the unknown
non-stationarity level \(C\) and to piecewise stationary structure. The framework is adapted
from the meta-algorithm of \citet{liu2026online}, originally developed for
non-stationarity adaptation in uninformed Markov games. Their algorithm uses a block and
sub-block schedule to combine a doubling search over an unknown non-stationarity level with
tests for infrequent environment switches. 

The framework takes four inputs, including a confidence level $\delta \in (0,1)$, an error measure  $\Err$ (e.g., $\Err=\Caldist_1$), a base
algorithm $\Alg$, and sub-block scheduling $\{G_b\}_b$.
For any contiguous interval $I \subseteq [T]$, let $\Err(I)$ denote the cumulative error incurred in interval $I$, and let $C(I) =\min_{q \in [0,1]}\sum_{t \in I} |q_t-q|$ be the non-stationarity for interval $I$. 
We assume that $\Err$ is \textit{subadditive}, in the sense that for any partition $\{I_i\}_i$ of $[T]$, $\Err([T]) \leq \sum_i \Err(I_i)$.
Indeed, this property holds for all calibration measures studied in this paper, including $\Caldist_1,\Caldist_2,\PCaldist_2,\PKL$ (see \pref{corr:decomp_Cal2_into_block}).
The base algorithm takes confidence $\delta \in (0,1)$ and a guess on $C$ as inputs, and it is required to satisfy the following condition.
\begin{condition}[Calibration certificate condition]\label{cond:anytime_adaptive_alg}
Under an error measure $\Err$, an algorithm satisfies the certificate condition with a certificate function $U_{\cdot}(\cdot,\cdot)$ if for every contiguous interval $I \subseteq [T]$, every confidence level $\delta\in(0,1)$, and every guess
$\widetilde C \ge C(I)$, a fresh run of this algorithm on $I$ with input $(\delta,\widetilde C)$ guarantees $\P \big( \Err(I) \leq U_{\delta}(I,\wt{C}) \big) \geq 1-\delta$, where $U_{\delta}(I,\wt{C}) $ is a non-negative function, non-decreasing with respect to both \(I\) and
\(\widetilde C\).\footnote{
For the first argument, which is an interval, monotonicity is understood with respect to set inclusion: if $I\subseteq I'$, then $U_\delta(I,\widetilde C)\le U_\delta(I',\widetilde C)$ for any $\widetilde C$.
}.
\end{condition}

\pref{alg:meta_alg_piecewise} proceeds in blocks $b=1,2,\ldots$. Each block $b$ contains
$G_b+1$ sub-blocks. The first $G_b$ sub-blocks serve as stationarity tests. In each such
sub-block, the base algorithm is restarted with the stationary guess $\widetilde C=0$, and the
sub-block continues as long as the realized error remains below the stationary certificate
$U_{\delta'}(\calT^{(b)}_\ell,0)$. Therefore, if one of these sub-blocks terminates, then its error is
too large to be explained by a stationary environment. On the high-probability event from
\pref{cond:anytime_adaptive_alg}, this implies that the sub-block must contain at least two segments of stationary environment. Hence, the number of
terminated stationary-test sub-blocks can be charged to the number of stationary segments.

The last sub-block of each block implements a doubling search over the unknown non-stationarity level $C$. In block $b$, the algorithm chooses a guess $\widetilde C_b =2^{3b}$ and runs the base algorithm with this guess. This sub-block terminates only if the realized error exceeds the block-level threshold $U_{\delta'}\big([T], \wh{C}^{(b)} \big)$. Once $b$ is large enough so that $\widetilde C_b$ is a valid upper bound on the true non-stationarity level, \pref{cond:anytime_adaptive_alg} prevents this sub-block from terminating with high probability. Thus, the last sub-blocks provide the usual unknown-$C$ adaptation, while the stationary-test sub-blocks provide the additional piecewise-stationary guarantee. The sub-block schedule $G_b$ balances these two sources of error. 

\textbf{Additional notations.} We use $B$ to denote the total number of blocks that \pref{alg:meta_alg_piecewise} proceeds during $T$ rounds. For each block $b$, we use $M_b$ to denote its total number of sub-blocks that the algorithm has ever entered.

\begin{definition}[Nice event] \label{def:high_prob_event_reduction}
    For any fixed algorithm $\Alg_{\delta}$ with an input confidence $\delta \in (0,1)$, which satisfies \pref{cond:anytime_adaptive_alg} under error measure $\Err$, let $\calE(\Alg_{\delta},\Err)$ be the event such that $\forall \text{ contiguous } I \subseteq [T]$ and $\forall \wt{C} \in \{ 2^{3b}: 1\leq b\leq \lceil \log_8(8+T) \rceil \} \cup \{0\}$, if $\wt{C} \geq C(I)$, then $\Err(I) \leq U_{\delta}(I,\wt{C}) $.
\end{definition}

\begin{lemma}
    For any fixed $\delta \in (0,1)$, $\P(\calE(\Alg_{ \delta' },\Err)) \geq 1-\delta$ where $\delta' =\frac{\delta}{10T^2\lceil \log_8(8+T) \rceil }$.
\end{lemma}
\begin{proof}
There are at most $2T^2$ contiguous intervals across $T$ rounds and at most $\lceil \log_8(8+T) \rceil+1 $ guesses. By repeating the union bound, we complete the proof.
\end{proof}

\begin{lemma} \label{lem:bound_block_reduction}
Suppose that \pref{alg:meta_alg_piecewise} runs algorithm $\Alg_{\delta}$ with input $\delta \in (0,1)$ under $\Err$. Suppose that $\calE(\Alg_{ \delta' },\Err)$ holds where $\delta' =\frac{\delta}{10T^2\lceil \log_8(8+T) \rceil }$.
We have $B \leq \lceil \log_8(8+C) \rceil $.
\end{lemma}
\begin{proof}
For each block $b$, the guess is $\wh{C}^{(b)}=2^{3b}$. 
Let $B' = \inf\{b \geq 1: 2^{3b} \geq C\}$. Then, the termination condition will never be met in the last sub-block of block $B'$. Thus, we have $B \leq B' \leq \lceil \log_8(8+C) \rceil $.
\end{proof}

\subsubsection{Analysis for Achieving $\Caldist_1=\otil \rbr{ \min \cbr{ \sqrt{WT}, \sqrt{T}+(CT)^{1/3} }  }$}

For $\Caldist_1$, we choose the sub-block scheduling $G_b$ as
\begin{align} \label{eq:Gb_choice_Cal1}
    G_b = \left \lceil \rbr{    \frac{U_{\delta'}([T],2^{3b})}{U_{\delta'}([T],0)}    }^2       \right \rceil .
\end{align}

\begin{theorem} \label{thm:combine_piecewise_cal1}
Let $\delta \in (0,1)$ and $\delta' =\frac{\delta}{10T^2\lceil \log_8(8+T) \rceil }$.
If the input algorithm of \pref{alg:meta_alg_piecewise}, denoted by $\Alg^{\Caldist_1}$ satisfies \pref{cond:anytime_adaptive_alg} with $U_{\delta'}(I,\wt{C}) = \alpha_1 \sqrt{|I|}+ \alpha_2 (\wt{C}|I|)^{1/3}$ and enjoys the per-round time complexity $\alpha_3 + \alpha_4 \wt{C}^{1/3}$, for any contiguous interval $I \subseteq [T]$, and any guess $\wt{C} \geq C(I)$, where $\alpha_1,\alpha_2,\alpha_3,\alpha_4 >0$ are polynomial in $\log T$ or $\iota$, then running \pref{alg:meta_alg_piecewise} with confidence $\delta \in (0,1)$, error measure $\Err=\Caldist_1$, and $G_b$ is chosen as \pref{eq:Gb_choice_Cal1} for each block $b$, ensures that with probability at least $1-\delta$,
\[
\Caldist_1  =\order \rbr{  \min \cbr{  \alpha_1 \log T \sqrt{T }+\alpha_2  ((1+C)T)^{1/3}  ,\alpha_1 \sqrt{W T \log T }  } }.
\]

Moreover, the per-round time complexity is
\[
\order \rbr{   \min \cbr{  \alpha_3 +\alpha_4 (1+C)^{1/3} , \alpha_3 + \frac{\alpha_1 \alpha_4}{\alpha_2} \sqrt{W} T^{1/6} }     }.
\]
\end{theorem}

\pref{thm:combine_piecewise_cal1} implies that using the algorithm that achieves the result in \pref{thm:cal1_better} as a base algorithm of \pref{alg:meta_alg_piecewise} yields a bound of $\Caldist_1=\otil \big( \min\{\sqrt{WT},\sqrt{T}+(CT)^{1/3}\} \big)$ with per-round time complexity $\otil(1+\min \{C^{1/3},\sqrt{W}T^{1/6}\})$. 

\begin{lemma} \label{lem:piecewise_Cbound_Cal1}
Suppose that \pref{alg:meta_alg_piecewise} is run under the same condition in \pref{thm:combine_piecewise_cal1} and $\calE(\Alg_{ \delta'}^{\Caldist_1},\Caldist_1)$ holds where $\calE(\cdot,\cdot)$ is given in \pref{def:high_prob_event_reduction} and $\Alg^{\Caldist_1}_{\delta'}$ is the input algorithm of \pref{alg:meta_alg_piecewise}  with confidence $\delta'=\frac{\delta}{10T^2\lceil \log_8(8+T) \rceil }$. We have
\[
\Caldist_1  = \order \rbr{  \alpha_1  \sqrt{T} \log T +\alpha_2 (T(1+C))^{1/3} }.
\]
\end{lemma}
\begin{proof}
For any block $b$, by the termination condition, we have
\begin{align*}
   \sum_{\ell=1}^{M_b} \Caldist_1 \Big(\calT^{(b)}_{\ell} \Big) & \leq \order \rbr{   \sum_{\ell=1}^{G_b} \alpha_1 \sqrt{|\calT^{(b)}_{\ell}|} + U_{\delta'} \rbr{[T],\wh{C}^{(b)}} } \\
      & \leq \order \rbr{   \sum_{\ell=1}^{G_b} \alpha_1 \sqrt{|\calT^{(b)}_{\ell}|} + \alpha_1 \sqrt{T} +\alpha_2 (\wh{C}^{(b)} T)^{1/3}   } \\
   &\leq \order \rbr{    \alpha_1 \sqrt{ G_b T}+ \alpha_2 \cdot 2^b T^{1/3}  }\\
   &\leq \order \rbr{    \alpha_1 \sqrt{T}+ \alpha_2 \cdot 2^b T^{1/3}  },
\end{align*}
where the first inequality holds because $M_b \leq G_b+1$ and if the algorithm enters the last sub-block of block $b$, then the termination condition implies that the error in the last sub-block can be bounded by $\order \Big( U_{\delta'} \rbr{[T],\wh{C}^{(b)}} \Big)$, the third inequality uses the Cauchy--Schwarz inequality, and the last inequality follows from the following:
\begin{align*}
G_b = \left \lceil  \rbr{\frac{ U_{\delta'}([T],\wh{C}^{(b)}) }{ U_{\delta'}([T],0) }}^2  \right \rceil  \leq 1+ \rbr{\frac{ U_{\delta'}([T],\wh{C}^{(b)}) }{ U_{\delta'}([T],0) }}^2  \leq  \order  \rbr{ \rbr{1+ \frac{\alpha_2 2^{b}  }{\alpha_1 T^{1/6}}  }^2   }.
\end{align*}

Then, we have
\begin{align*}
  \Caldist_1 &\leq \sum_{b=1}^B \sum_{\ell=1}^{M_b} \Caldist_1 \Big(\calT^{(b)}_{\ell} \Big)  \\
  & \leq \order \rbr{    \alpha_1 \sum_{b=1}^B\sqrt{ T}+  \alpha_2 T^{1/3} \sum_{b=1}^B 2^b   } \\
    & \leq \order \rbr{    \alpha_1 \log T\sqrt{T}+\alpha_2  2^{B} T^{1/3}   } \\
  & \leq \order \rbr{    \alpha_1 \log T \sqrt{T }+\alpha_2  ((1+C)T)^{1/3}   },
\end{align*}
where the first inequality uses the subadditivity of $\Caldist_1$ based on \pref{corr:decomp_Cal2_into_block}, and the last inequality holds due to \pref{lem:bound_block_reduction}.
\end{proof}

\begin{lemma}\label{lem:piecewise_KTbound_Cal1}
Suppose that \pref{alg:meta_alg_piecewise} is run under the same condition in \pref{thm:combine_piecewise_cal1} and $\calE(\Alg_{ \delta'}^{\Caldist_1},\Caldist_1)$ holds where $\calE(\cdot,\cdot)$ is given in \pref{def:high_prob_event_reduction} and $\Alg^{\Caldist_1}_{\delta'}$ is the input algorithm of \pref{alg:meta_alg_piecewise}  with confidence $\delta'=\frac{\delta}{10T^2\lceil \log_8(8+T) \rceil }$. We have
\[
\Caldist_1  = \order \rbr{ \alpha_1 \sqrt{W T \log T} }.
\]
\end{lemma}
\begin{proof}
Let us define:
\begin{equation} \label{eq:def_S_all_test_sub_blocks}
    \calS:= \cbr{ \calT^{(b)}_{\ell}: \ell \leq G_b, b \leq B  }.
\end{equation}

As the input algorithm satisfies \pref{cond:anytime_adaptive_alg}, for any sub-block $\ell \leq G_b$ of any block $b$, then on the nice event it cannot be contained in a single stationary segment. Hence it contains at least one change point. Since sub-blocks are disjoint in time, the number of terminated stationary-test sub-blocks is at most $\order(W)$, thereby $|\calS| \leq \order(W)$.
Then,
\begin{align*}
  \sum_{I \in \calS}  \Caldist_1 ( I ) \leq \order  \rbr{  \alpha_1 \sum_{I \in \calS} \sqrt{ |I|}  }  \leq \order  \rbr{  \alpha_1  \sqrt{ WT}  },
\end{align*}
where the last inequality uses the Cauchy--Schwarz inequality together with facts that $\sum_{I \in \calS} |I| \leq T$ and $|\calS| \leq W$.
Notice that 
\[
\sqrt{G_b} \geq \frac{U_{\delta'} \rbr{[T],\wh{C}^{(b)}}}{U_{\delta'} \rbr{[T], 0  }} = \frac{U_{\delta'} \rbr{[T],\wh{C}^{(b)}}}{  \alpha_1 \sqrt{T}  },
\]
which gives
\begin{equation} \label{eq:U_bound_alpha_rootGT}
    U_{\delta'} \rbr{[T],\wh{C}^{(b)}} \leq \alpha_1 \sqrt{G_bT}.
\end{equation}

Thus, we can show that
\begin{align*}
    \sum_{b \leq B} \Caldist_1 ( \calT^{b}_{G_b+1} ) &\leq \order  \rbr{  \sum_{b \leq B}  U_{\delta'} \rbr{[T],\wh{C}^{(b)}} }\\
    &\leq \order \rbr{ \alpha_1 \sum_{b \leq B}  \sqrt{G_b T} } \leq \order \rbr{ \alpha_1 \sqrt{\log T \sum_{b \leq B} G_b T} }\leq \order \rbr{ \alpha_1 \sqrt{\log T WT} },
\end{align*}
where the first inequality uses the termination condition, and the last inequality follows from the facts that $B =\order(\log T)$ and for any block $b$, once the algorithm enters the last sub-block, the number of environment changes is at least $G_b$, thereby $\sum_b G_b \leq \order(W)$.

By the subadditivity of $\Caldist_1$ (\pref{corr:decomp_Cal2_into_block}), we have
\begin{align*}
     \Caldist_1 &\leq   \sum_{b \leq B} \Caldist_1 ( \calT^{b}_{G_b+1} )+\sum_{I \in \calS}  \Caldist_1 ( I )   \leq \order \rbr{ \alpha_1 \sqrt{ WT \log T} },
\end{align*}
which completes the proof.
\end{proof}

\begin{proof}[Proof of \pref{thm:combine_piecewise_cal1}]
Conditioning on event $\calE(\Alg_{ \delta'}^{\Caldist_1},\Caldist_1)$, which occurs with probability at least $1-\delta$, \pref{lem:piecewise_Cbound_Cal1} and \pref{lem:piecewise_KTbound_Cal1} jointly give the claimed result of $\Caldist_1$. Then, we analyze the per-round time complexity. 
For any sub-block $\ell \leq G_b$, the per-round time complexity is at most $\alpha_3$ since the input algorithm is with guess $0$.
For any block $b$, the last sub-block runs the base algorithm with guess $\wh{C}_b = 2^{3b}$, and thus the per-round time complexity is at most $\order(\alpha_3 +\alpha_4 2^{b})$. Recall that if the algorithm enters the last sub-block of block $b$, then we have $G_b \leq W$. Thus, using \pref{eq:U_bound_alpha_rootGT}, we have for any block $b$
\[
U_{\delta'} \rbr{[T],\wh{C}^{(b)}} = \alpha_1\sqrt{T} +\alpha_2 2^{b}T^{1/3} \leq \alpha_1 \sqrt{G_bT} \leq \alpha_1 \sqrt{WT} \Longrightarrow  2^b \leq \order \rbr{ \frac{\alpha_1}{\alpha_2} \sqrt{W} T^{1/6}}.
\]

Thus, one can bound the per-round time complexity by $\order \big(\alpha_3 + \frac{\alpha_1 \alpha_4}{\alpha_2} \sqrt{W} T^{1/6} \big)$.
As the total number of blocks is at most $1+\log_8(1+C)$, the per-round cost for the last sub-block of any block $b$ is at most $\order(\alpha_3 +\alpha_4 2^{b}) \leq \order(\alpha_3 +\alpha_4 (1+C)^{1/3})$. Combining all together, we get the claimed bound on the per-round time complexity.
\end{proof}

\subsubsection{Analysis for Achieving $\Caldist_2=\otil \rbr{ \min \cbr{W, (1+C)^{1/3} }  }$}
For $\Caldist_2$, we choose the sub-block scheduling $G_b$ as
\begin{align} \label{eq:Gb_choice_Cal2}
    G_b = \left \lceil    \frac{U_{\delta'}([T],2^{3b})}{U_{\delta'}([T],0)}         \right \rceil .
\end{align}

\begin{theorem} \label{thm:combine_piecewise_cal2}
Let $\delta \in (0,1)$, $\delta' =\frac{\delta}{10T^2\lceil \log_8(8+T) \rceil }$, and $\beta \in [0,1]$.
If the input algorithm of \pref{alg:meta_alg_piecewise}, denoted by $\Alg^{\Caldist_2}$ satisfies \pref{cond:anytime_adaptive_alg} with $U_{\delta'}(I,\wt{C}) = \alpha_1 + \alpha_2 \wt{C}^{1/3}$ and enjoys the per-round time complexity $\alpha_3 + \alpha_4 \wt{C}^{\beta}$, for any contiguous interval $I \subseteq [T]$, and any guess $\wt{C} \geq C(I)$, where $\alpha_1,\alpha_2,\alpha_3,\alpha_4>0$ are polynomial in $\log T$ or $\iota$, then running \pref{alg:meta_alg_piecewise} with confidence $\delta \in (0,1)$, error measure $\Err=\Caldist_2$, and $G_b$ is chosen as \pref{eq:Gb_choice_Cal2} for each block $b$, ensures that with probability at least $1-\delta$,
\[
\Caldist_2  =\order \rbr{  \min \cbr{\alpha_1 \log T +\alpha_2  (1+C)^{1/3}   ,\alpha_1 W  } }.
\]

Moreover, the per-round time complexity is
\[
\order \rbr{   \min \cbr{  \alpha_3 +\alpha_4 (1+C)^{\beta} , \alpha_3 + \alpha_4 \rbr{\frac{\alpha_1 }{\alpha_2} W }^{3\beta }  }     }.
\]
\end{theorem}

\pref{thm:combine_piecewise_cal2} implies that using the algorithm that achieves the result in \pref{thm:cal2_approach1} as a base algorithm of \pref{alg:meta_alg_piecewise} yields a bound of $\Caldist_2=\otil \big( \min\{W,(1+C)^{1/3}\} \big)$ with per-round time complexity $\otil \big( \min\{W,(1+C)^{1/3}\} \big)$ (here using $\beta=1/3$ in \pref{thm:combine_piecewise_cal2}). 
\pref{thm:combine_piecewise_cal2} implies that using the algorithm that achieves the result in \pref{corr:Cal2_bound_implied_by_PCal2} as a base algorithm of \pref{alg:meta_alg_piecewise} yields a bound of $\Caldist_2=\otil \big( \min\{W,(1+C)^{1/3}\} \big)$ with per-round time complexity $\otil \big( \min\{W^2,(1+C)^{2/3}\} \big)$ (here using $\beta=2/3$ in \pref{thm:combine_piecewise_cal2}). 

\begin{lemma} \label{lem:piecewise_Cbound_Cal2}
Suppose that \pref{alg:meta_alg_piecewise} is run under the same condition in \pref{thm:combine_piecewise_cal2}
and $\calE(\Alg_{ \delta' }^{\Caldist_2},\Caldist_2)$ holds where $\calE(\cdot,\cdot)$ is given in \pref{def:high_prob_event_reduction} and $\Alg_{ \delta'}^{\Caldist_2}$ is the input algorithm of \pref{alg:meta_alg_piecewise} with confidence $\delta' =\frac{\delta}{10T^2\lceil \log_8(8+T) \rceil }$.
We have
\[
\Caldist_2  = \order \rbr{ \alpha_1 \log T +\alpha_2  (1+C)^{1/3}   }.
\]
\end{lemma}
\begin{proof}
For any block $b$, by the termination condition, we have
\begin{align*}
   \sum_{\ell=1}^{M_b} \Caldist_2 \Big(\calT^{(b)}_{\ell} \Big)  &\leq \order \rbr{   G_b \alpha_1  +  U_{\delta'} \rbr{[T],\wh{C}^{(b)} } }\\
   &= \order \rbr{   G_b \alpha_1  +   \alpha_1 +\alpha_2 2^b }\\
   &\leq \order \rbr{    \alpha_1 +\alpha_2 2^b  } ,
\end{align*}
where the first inequality holds because $M_b \leq G_b+1$ and if the algorithm enters the last sub-block of block $b$, then the termination condition implies that the error in the last sub-block can be bounded by $\order \big(U_{\delta'} \rbr{[T],\wh{C}^{(b)} }  \big)$, and the last inequality uses the following:
\[
  G_b = \left \lceil    \frac{U_{\delta'}([T],2^{3b})}{U_{\delta'}([T],0)}         \right \rceil  \leq 1+ \frac{U_{\delta'}([T],2^{3b})}{U_{\delta'}([T],0)}    \leq \order \rbr{  1+ \frac{\alpha_2}{\alpha_1}2^b }.
\]

Then, we have
\begin{align*}
  \Caldist_2 &\leq \sum_{b=1}^B \sum_{\ell=1}^{M_b} \Caldist_2 \Big(\calT^{(b)}_{\ell} \Big)  
   \leq \order \rbr{     \sum_{b=1}^B \rbr{ \alpha_1 +\alpha_2 2^b} } 
   \leq \order \rbr{    \alpha_1 \log T +\alpha_2  (1+C)^{1/3}   },
\end{align*}
where the last inequality uses the fact that $B=\order(\log T)$ and applies \pref{lem:bound_block_reduction}.
\end{proof}

\begin{lemma}\label{lem:piecewise_KTbound_Cal2}
Suppose that \pref{alg:meta_alg_piecewise} is run under the same condition in \pref{thm:combine_piecewise_cal2}
and $\calE(\Alg_{ \delta' }^{\Caldist_2},\Caldist_2)$ holds where $\calE(\cdot,\cdot)$ is given in \pref{def:high_prob_event_reduction} and $\Alg^{\Caldist_2}_{\delta'}$ is the input algorithm of \pref{alg:meta_alg_piecewise}  with confidence $\delta'=\frac{\delta}{10T^2\lceil \log_8(8+T) \rceil }$.
We have
\[
\Caldist_2  = \order \rbr{ \alpha_1 W }.
\]
\end{lemma}
\begin{proof}
Recall the definition of $\calS$ in \pref{eq:def_S_all_test_sub_blocks}.
As the input algorithm satisfies \pref{cond:anytime_adaptive_alg}, for any sub-block $\ell \leq G_b$ of any block $b$, then on the nice event it cannot be contained in a single stationary segment. Hence it contains at least one change point. Since sub-blocks are disjoint in time, the number of terminated stationary-test sub-blocks is at most $\order(W)$, thereby $|\calS| \leq \order(W)$.
Then,
\begin{align*}
  \sum_{I \in \calS}  \Caldist_2 ( I ) \leq \order  \rbr{  \sum_{I \in \calS}  \alpha_1   }  \leq \order  \rbr{  \alpha_1  W }.
\end{align*}

We can use $\sum_b G_b \leq W$ to show that
\begin{align*}
    \sum_{b \leq B} \Caldist_2 ( \calT^{b}_{G_b+1} ) \leq \order  \rbr{  \sum_{b \leq B} U_{\delta'}([T],2^{3b})} = \order \rbr{ \sum_{b \leq B} \rbr{\alpha_1  +\alpha_2 2^b} }\leq \order \rbr{ \alpha_1 \sum_{b \leq B} G_b  } \leq \order (\alpha_1 W),
\end{align*}
where the second inequality uses the fact that $B=\order(\log T)$ and
\begin{equation} \label{eq:U_bound_alpha_Gb_cal2}
      G_b \geq      \frac{U_{\delta'}([T],2^{3b})}{U_{\delta'}([T],0)}             =   1+ \frac{\alpha_2}{\alpha_1}2^b .
\end{equation}

By the subadditivity of $\Caldist_2$ (\pref{corr:decomp_Cal2_into_block}), we have
\begin{align*}
     \Caldist_2 &\leq   \sum_{b \leq B} \Caldist_2 ( \calT^{b}_{G_b+1} )+\sum_{I \in \calS}  \Caldist_2 ( I )   \leq \order \rbr{ \alpha_1 W },
\end{align*}
which completes the proof.
\end{proof}

\begin{proof}[Proof of \pref{thm:combine_piecewise_cal2}]
Conditioning on event $\calE(\Alg_{ \delta' }^{\Caldist_2},\Caldist_2)$, which occurs with probability at least $1-\delta$, \pref{lem:piecewise_Cbound_Cal2} and \pref{lem:piecewise_KTbound_Cal2} jointly give the claimed result of $\Caldist_2$.
Then, we analyze the per-round time complexity. 
For any sub-block $\ell \leq G_b$, the per-round time complexity is at most $\alpha_3$ since the input algorithm is with guess $0$.
For any block $b$, the last sub-block runs the base algorithm with guess $\wh{C}_b = 2^{3b}$, and thus the per-round time complexity is at most $\order(\alpha_3 +\alpha_4 2^{3b \beta})$. Recall that if the algorithm enters the last sub-block of block $b$, then we have $G_b \leq W$. Thus, using \pref{eq:U_bound_alpha_Gb_cal2}, we have for any block $b$
\[
U_{\delta'} \rbr{[T],\wh{C}^{(b)}} = \alpha_1 +\alpha_2 2^{b} \leq \alpha_1 G_b \leq \alpha_1 W \Longrightarrow  2^b \leq \order \rbr{ \frac{\alpha_1}{\alpha_2} W }.
\]

Thus, one can bound the per-round time complexity by $\order \big(\alpha_3 + \alpha_4 \big(\frac{\alpha_1 }{\alpha_2} W \big)^{3\beta }\big)$.
As the total number of blocks is at most $1+\log_8(1+C)$, the per-round cost for the last sub-block of any block $b$ is at most $\order(\alpha_3 +\alpha_4 2^{3b\beta}) \leq \order(\alpha_3 +\alpha_4 (1+C)^{\beta})$. Combining all together, we get the claimed bound on the per-round time complexity.
\end{proof}

\subsubsection{Analysis for Achieving $\PKL=\otil \rbr{ \min \cbr{W, (1+C)^{1/3} }  }$}

\begin{theorem} \label{thm:combine_piecewise_PKL}
Let $\delta \in (0,1)$ and $\delta' =\frac{\delta}{10T^2\lceil \log_8(8+T) \rceil }$.
Suppose that the input algorithm of \pref{alg:meta_alg_piecewise}, denoted by $\Alg^{\PKL}$ satisfies \pref{cond:anytime_adaptive_alg} with $U_{\delta'}(I,\wt{C}) = H+\alpha_1 + \alpha_2 \wt{C}^{1/3}$ and enjoys the per-round time complexity $\alpha_3 + \alpha_4 \wt{C}^{2/3}$, for any contiguous interval $I \subseteq [T]$, and any guess $\wt{C} \geq C(I)$, where $\alpha_1,\alpha_2,\alpha_3,\alpha_4,H > 0$ are polynomial in $\log T$ or $\iota$. If  \pref{alg:meta_alg_piecewise} is run with confidence $\delta \in (0,1)$, error measure $\Err=\PKL$, $G_b$ is chosen as \pref{eq:Gb_choice_Cal2} for each block $b$, and $\Err(\{t\}) \leq H$ for each round $t \in [T]$, then with probability at least $1-\delta$,
\[
\PKL  =\order \rbr{  \min \cbr{     (H+\alpha_1)\log T + \alpha_2 (1+C)^{1/3}    ,(\alpha_1 +H) W  } }.
\]

Moreover, the per-round time complexity is
\[
\order \rbr{   \min \cbr{  \alpha_3 +\alpha_4 (1+C)^{2/3} , \alpha_3 +\alpha_4  \rbr{\frac{(\alpha_1+H) }{\alpha_2} W}^2  }     }.
\]
\end{theorem}

\pref{thm:combine_piecewise_PKL} implies that using the algorithm that achieves the result in \pref{thm:PKLCal_bound_knownC} as a base algorithm of \pref{alg:meta_alg_piecewise} yields a bound of $\PKL=\otil \big( \min\{W,(1+C)^{1/3}\} \big)$ with per-round time complexity $\otil \big( \min\{W^2,(1+C)^{2/3}\} \big)$. 

\begin{lemma} \label{lem:piecewise_Cbound_PKL}
Suppose that \pref{alg:meta_alg_piecewise} is run under the same condition in \pref{thm:combine_piecewise_PKL}
and $\calE(\Alg_{\delta'}^{\PKL},\PKL)$ holds where $\calE(\cdot,\cdot)$ is given in \pref{def:high_prob_event_reduction} and $\Alg^{\PKL}_{\delta'}$ is the input algorithm of \pref{alg:meta_alg_piecewise}  with confidence $\delta'=\frac{\delta}{10T^2\lceil \log_8(8+T) \rceil }$.
We have
\[
\PKL  = \order \rbr{     (H+\alpha_1)\log T + \alpha_2 (1+C)^{1/3}   }.
\]
\end{lemma}
\begin{proof}
For any block $b$, by the termination condition, we have
\begin{align*}
\sum_{\ell=1}^{M_b} \PKL \Big(\calT^{(b)}_{\ell} \Big) & \leq \order \rbr{   G_b (\alpha_1+H)  +  U_{\delta'} \rbr{[T],2^{3b}} +H}\\
&= \order \rbr{      G_b (\alpha_1+H)  +   \alpha_1 +\alpha_2 2^b +H } \\
&\leq \order \rbr{   \alpha_1 +\alpha_2 2^b +H },
\end{align*}
where the first inequality holds because $M_b \leq G_b+1$ and if the algorithm enters the last sub-block of block $b$, then the termination condition implies that the error in the last sub-block can be bounded by $\order \big(U_{\delta'} \rbr{[T],\wh{C}^{(b)} }  \big)$, and the last inequality uses the following:
\[
  G_b = \left \lceil    \frac{U_{\delta'}([T],2^{3b})}{U_{\delta'}([T],0)}         \right \rceil  \leq 1+ \frac{U_{\delta'}([T],2^{3b})}{U_{\delta'}([T],0)}    \leq \order \rbr{  1+ \frac{\alpha_2}{\alpha_1+H}2^b }.
\]

Then, we have
\begin{align*}
  \PKL &\leq \sum_{b=1}^B \sum_{\ell=1}^{M_b} \PKL \Big(\calT^{(b)}_{\ell} \Big)  \\
   &\leq \order \rbr{   \sum_{b=1}^B \rbr{ \alpha_1 +\alpha_2 2^b +H } } 
   \leq  \order \rbr{    (H+\alpha_1)\log T + \alpha_2 (1+C)^{1/3}   },
\end{align*}
where the last inequality uses the fact that $B=\order(\log T)$ and applies \pref{lem:bound_block_reduction}.
\end{proof}

\begin{lemma}\label{lem:piecewise_KTbound_PKL}
Suppose that \pref{alg:meta_alg_piecewise} is run under the same condition in \pref{thm:combine_piecewise_PKL}
and $\calE(\Alg_{  \delta'}^{\PKL},\PKL)$ holds where $\calE(\cdot,\cdot)$ is given in \pref{def:high_prob_event_reduction} and $\Alg^{\PKL}_{\delta'}$ is the input algorithm of \pref{alg:meta_alg_piecewise}  with confidence $\delta'=\frac{\delta}{10T^2 \lceil \log_8(8+T) \rceil }$.
We have
\[
\PKL  = \order \rbr{(\alpha_1 +H) W }.
\]
\end{lemma}
\begin{proof}
Recall the definition of $\calS$ in \pref{eq:def_S_all_test_sub_blocks}.
As the input algorithm satisfies \pref{cond:anytime_adaptive_alg}, for any sub-block $\ell \leq G_b$ of any block $b$, then on the nice event it cannot be contained in a single stationary segment. Hence it contains at least one change point. Since sub-blocks are disjoint in time, the number of terminated stationary-test sub-blocks is at most $\order(W)$, thereby $|\calS| \leq \order(W)$.
Then,
\begin{align*}
  \sum_{I \in \calS}  \PKL ( I ) \leq \order  \rbr{  \sum_{I \in \calS}  (\alpha_1 +H)  }  \leq \order  \rbr{ (\alpha_1 +H) W }.
\end{align*}

Moreover, we have
\begin{align*}
    \sum_{b \leq B} \PKL ( \calT^{b}_{G_b+1} ) &\leq \order  \rbr{  \sum_{b \leq B} U_{\delta'}([T],2^{3b})} \\
    &= \order \rbr{ \sum_{b \leq B} \rbr{H+\alpha_1  +\alpha_2 2^b} }\leq \order \rbr{ (H+\alpha_1 )\sum_{b \leq B} G_b  } \leq \order ((H+\alpha_1 ) W),
\end{align*}
where the second inequality uses the fact that $B=\order(\log T)$ and
\begin{equation} \label{eq:U_bound_alpha_Gb_PKL}
      G_b \geq      \frac{U_{\delta'}([T],2^{3b})}{U_{\delta'}([T],0)}             =   1+ \frac{\alpha_2}{H+\alpha_1}2^b .
\end{equation}

By the subadditivity of $\PKL$ (\pref{corr:decomp_Cal2_into_block}), we have
\begin{align*}
     \PKL &\leq   \sum_{b \leq B} \PKL ( \calT^{b}_{G_b+1} )+\sum_{I \in \calS}  \PKL ( I )   \leq \order \rbr{ (\alpha_1 +H) W },
\end{align*}
which completes the proof.
\end{proof}

\begin{proof}[Proof of \pref{thm:combine_piecewise_PKL}]
Conditioning on event $\calE(\Alg_{ \delta' }^{\PKL},\PKL)$, which occurs with probability at least $1-\delta$,
\pref{lem:piecewise_Cbound_PKL} and \pref{lem:piecewise_KTbound_PKL} jointly give the claimed result.
Then, we analyze the per-round time complexity. 
For any sub-block $\ell \leq G_b$, the per-round time complexity is at most $\alpha_3$ since the input algorithm is with guess $0$.
For any block $b$, the last sub-block runs the base algorithm with guess $\wh{C}_b = 2^{3b}$, and thus the per-round time complexity is at most $\order(\alpha_3 +\alpha_4 2^{2b})$. Recall that if the algorithm enters the last sub-block of block $b$, then we have $G_b \leq W$. Thus, using \pref{eq:U_bound_alpha_Gb_PKL}, we have for any block $b$
\[
U_{\delta'} \rbr{[T],\wh{C}^{(b)}} = H+\alpha_1 +\alpha_2 2^{b} \leq (\alpha_1+H) G_b \leq (\alpha_1+H) W \Longrightarrow  2^b \leq \order \rbr{ \frac{\alpha_1+H}{\alpha_2} W }.
\]

Thus, one can bound the per-round time complexity by $\order \big(\alpha_3 +\alpha_4 \big( \frac{(\alpha_1+H)}{\alpha_2} W \big)^2 \big)$.
As the total number of blocks is at most $1+\log_8(1+C)$, the per-round cost for the last sub-block of any block $b$ is at most $\order(\alpha_3 +\alpha_4 2^{2b}) \leq \order(\alpha_3 +\alpha_4 (1+C)^{2/3})$. Combining all together, we get the claimed bound on the per-round time complexity.
\end{proof}

\subsection{Omitted Details for Simple Epoch Algorithm in the Stochastic Setting}
\label{app:simple_block}

\subsubsection{Simple Epoch Algorithms}
The simple epoch algorithm is given in \pref{alg:simple_block}, which proceeds in epoch $m=1,2,\ldots$. In each epoch $m$, the algorithm always predicts $\wh{y}_{m-1}$, which is the empirical mean of outcomes from epoch $m-1$. We set $\wh{y}_0 =  \frac{1}{2}$.
We again use $\calT_m$ to denote the set of rounds in epoch $m$, and let $|\calT_m|=s_m=2^{m}$.

\setcounter{AlgoLine}{0}
\begin{algorithm}[H]
\DontPrintSemicolon
\caption{Epoch-based calibration for $\Caldist_1,\Caldist_2$}
\label{alg:simple_block}

\textbf{Initialize:} $\wh{y}_0\gets \frac{1}{2}$.

\For{epoch $m=1,2,\ldots$}{

Predict $\wh{y}_{m-1}$ for $s_m=2^m$ rounds and collect outcomes $\{y_t\}_{t \in \calT_m}$.

Set $\wh{y}_m \gets \frac{1}{|\calT_{m}|} \sum_{t \in \calT_m}  y_t$. 

}    
\end{algorithm}

For $\PKL$, the algorithm needs to be slightly modified. We define
\begin{align*}
    d_m = \min \cbr{\frac{\pi}{4}, 8\pi \sqrt{ \frac{\iota}{s_m} } }.
\end{align*}

\setcounter{AlgoLine}{0}
\begin{algorithm}[H]
\DontPrintSemicolon
\caption{Epoch-based calibration for $\PKL$}
\label{alg:simple_block_PKL}

\textbf{Initialize:} $\wh{y}_0\gets \frac{1}{2}$.

\For{epoch $m=1,2,\ldots$}{

Predict $\wh{p}_m \gets \psi( \Clip_{[d_m,\pi-d_m]}\theta(\wh{y}_{m-1}))$ for $s_m=2^m$ rounds and collect outcomes $\{y_t\}_{t \in \calT_m}$.

Set $\wh{y}_m \gets \frac{1}{|\calT_{m}|} \sum_{t \in \calT_m}  y_t$. 

}    
\end{algorithm}

\subsubsection{Achieving $\Caldist_1=\otil(\sqrt{T})$}

\begin{theorem} \label{thm:simple_block_Cal1}
For $C=0$, with probability at least $1-\delta$, \pref{alg:simple_block} ensures that $\Caldist_1 = \order \rbr{  \sqrt{  \iota T  }}$.
\end{theorem}

\begin{proof}
By Hoeffding's bound and a union bound over all epochs, with probability at least $1-\delta$,
\begin{align} \label{eq:concentration_simple_block}
\forall m:\quad \abr{  \sum_{t \in \calT_m} (y_t-\optp) } \leq \order \rbr{  \sqrt{s_m \iota} }.
\end{align}

The following analysis conditions on the event that \pref{eq:concentration_simple_block} holds.
We write
\begin{align*}
\Caldist_1 &=\sum_{\alpha \in [0,1]:n(\alpha)>0}  \abr{ \sum_{t=1}^T \rbr{ y_t -p_t } \Ind{p_t = \alpha } }\\
&= \sum_{\alpha \in [0,1]:n(\alpha)>0} \abr{  \sum_{m: \wh{y}_{m-1} =\alpha }  \sum_{t \in \calT_m} \rbr{ y_t - \wh{y}_{m-1} } \Ind{p_t = \wh{y}_{m-1}  } } \\
&\leq \sum_{\alpha \in [0,1]:n(\alpha)>0} \sum_{m: \wh{y}_{m-1} =\alpha } \abr{    \sum_{t \in \calT_m} \rbr{ y_t - \wh{y}_{m-1} } \Ind{p_t = \wh{y}_{m-1}  } } \\
&=  \sum_{m}  \abr{ \sum_{t \in \calT_m} \rbr{ y_t - \wh{y}_{m-1} }   } \\
&\leq \sum_{m}  \abr{ \sum_{t \in \calT_m} \rbr{ y_t - \optp }   }  +  \sum_{m}  \abr{ \sum_{t \in \calT_m} \rbr{ \optp - \wh{y}_{m-1} }   }  \\
&\leq \order \rbr { \sum_{m} \sqrt{ s_m \iota}   }+  \sum_{m}  \abr{ \sum_{t \in \calT_m} \rbr{ \optp - \wh{y}_{m-1} }   } \tag{By \pref{eq:concentration_simple_block}}  \\
&\leq \order \rbr {  \sqrt{ T \iota}   }+  \sum_{m}  \abr{ \sum_{t \in \calT_m} \rbr{ \optp - \wh{y}_{m-1} }   }  \numberthis{}  \label{eq:simple_block_step1},
\end{align*}
where the last inequality uses the fact that $s_m=2^m$.

We then use \pref{eq:concentration_simple_block} to show that for $m \geq 2$
\begin{align*}
\abr{ \optp -\wh{y}_{m-1} } & = \frac{1}{s_{m-1}} \abr{\sum_{t \in \calT_{m-1}}  \rbr{ \optp-y_t  } }   \leq \order \rbr{ \sqrt{ \frac{\iota}{s_{m-1}}    }  }.
\end{align*}

Plugging the above into \pref{eq:simple_block_step1} and using the fact that $s_m=2s_{m-1}$, we have 
\begin{align*}
    \Caldist_1 \leq \order \rbr{  \sqrt{  \iota T  }}.
\end{align*}

The proof is thus complete.
\end{proof}

\subsubsection{Achieving $\Caldist_2 = \order \rbr{\iota \log T }$}

\begin{theorem} \label{thm:simple_block_Cal2}
For $C=0$, with probability at least $1-\delta$, \pref{alg:simple_block} ensures that $\Caldist_2 = \order \rbr{\iota \log T }$.
\end{theorem}

\begin{proof}
The following analysis conditions on the event that \pref{eq:concentration_simple_block} holds.
Let $n_m(\alpha)$ be the number of times that $\alpha$ is predicted by the forecaster in epoch $m$.
By \pref{corr:decomp_Cal2_into_block}, we have that $\Caldist_2 \leq \sum_{m} \Caldist_2^{(m)}$, and then for all $m \geq 2$
\begin{align*}
\Caldist_2^{(m)} &=  \frac{1}{n_m(\wh{y}_{m-1})}  \rbr{  \sum_{t \in \calT_m} \rbr{ y_t - \wh{y}_{m-1} }   }^2 \\
&=  \frac{1}{s_m}  \rbr{  \sum_{t \in \calT_m} \rbr{ y_t - \wh{y}_{m-1} }   }^2 \\
&\leq    \frac{2}{s_m} \rbr{ \sum_{t \in \calT_m}  \rbr{ y_t  -  \optp }  }^2  +   \frac{2}{s_m} \rbr{ \sum_{t \in \calT_m}  \rbr{ \optp -  \wh{y}_{m-1} }  }^2  \\
&\leq \order \rbr{ \iota  } ,
\end{align*}
where the last inequality follows a similar argument in \pref{thm:simple_block_Cal1}. 

Since the total number of epochs is at most $\order(\log T)$, we complete the proof.
\end{proof}

\subsubsection{Achieving $\PKL = \order \rbr{\iota \log T }$}

\begin{lemma} \label{lem:sup_sin_bound}
Let $d \in[0,\pi/4]$.   
For any $u,v$ such that $v\in [d,\pi-d]$ and $|u-v| \leq 2d$, we have
\begin{align*}
    \sup_{\text{$w$ between $u,v$}} \sin w \leq (1+\pi)\sin v.
\end{align*}
\end{lemma}
\begin{proof}
Consider any point $w$ between $u,v$.
As $\sin(\cdot)$ is $1$-Lipschitz on $[0,\pi]$, 
\[
\sin w \leq \sin v +|w-v| \leq \sin v +2d.
\]

As $d \leq \pi/4$, we have $\sin d \geq 2d/\pi$. Also, $v \in [d,\pi-d]$ implies that $\sin v \geq \sin d$. Thus, $2d \leq \pi \sin v$, which further gives $\sin w \leq \sin v + \pi \sin v.$
The proof is thus complete.
\end{proof}

\begin{theorem} \label{thm:simple_block_PKL}
For $C=0$, with probability at least $1-\delta$, \pref{alg:simple_block_PKL} ensures that $\PKL = \order \rbr{\iota \log T }$.
\end{theorem}
\begin{proof}
Suppose that the high probability bounds in \pref{lem:concentration_optp_PKL} hold for $C=0$.
By \pref{corr:decomp_Cal2_into_block}, we have that $\PKL \leq \sum_{m} \PKL^{(m)}$.
Recall that $\bar{\rho}_{m,p} = \frac{ \sum_{t \in \calT_m} \calP_t(p) y_t}{\sum_{t \in \calT_m}  \calP_t(p)}$.
As the algorithm always predicts a single value $\wh{p}_m$ for all rounds in an epoch $m$, we have $\bar{\rho}_{m,\wh{p}_m} = \frac{ \sum_{t \in \calT_m}  y_t}{s_m} =\wh{y}_m$.
Then, we can write for all $m \geq 2$, 
\begin{align*}
\PKL^{(m)} &= s_m \KL \rbr{ \bar{\rho}_{m,\wh{p}_m} , \wh{p}_m} \\
&= s_m \KL \rbr{ \wh{y}_m , \wh{p}_m} .
\end{align*}

Recall that $\theta(x)=2\arcsin(\sqrt{x})$ and we define $\psi(z)=\theta^{-1}(z)=\sin^2(z/2)$.
We use the fact that for any $q\in [0,1]$ and $p \in(0,1)$, $\KL(q,p) \leq \frac{(q-p)^2}{p(1-p)}$, which implies that
\begin{align*}
   \KL \rbr{ \wh{y}_m , \wh{p}_m}  \leq \frac{ (\wh{y}_m - \wh{p}_m)^2 }{ \wh{p}_m(1-\wh{p}_m) }  = 4\frac{ \rbr{\psi( \theta(\wh{y}_m) )-\psi( \theta(\wh{p}_m) )}^2 }{\sin^2 (\theta(\wh{p}_m))},
\end{align*}
where the equality uses the fact that $   \wh{p}_m(1-\wh{p}_m) = \psi( \theta(\wh{p}_m) ) \rbr{ 1- \psi( \theta(\wh{p}_m) )} =\frac{1}{4} \sin^2 (\theta(\wh{p}_m))$.

Since $\psi'(z)=\frac{1}{2} \sin z$, by mean-value theorem, 
\begin{align*}
    \abr{ \psi( \theta(\wh{y}_m) )-\psi( \theta(\wh{p}_m) )  } & \leq \abr{\theta(\wh{y}_m)-\theta(\wh{p}_m )} \sup_{\text{$z$ between $\theta(\wh{p}_m)$ and $\theta(\wh{y}_m)$}} \frac{\sin z}{2} \\
   & \leq \order \rbr{ \abr{\theta(\wh{y}_m)-\theta(\wh{p}_m )} \sin(\theta(\wh{p}_m)) }, \numberthis{} \label{eq:simple_block_verification_PKL}
\end{align*}
where the reasoning of the second inequality is deferred to the end of this proof.
Therefore, 
\begin{align*}
\KL \rbr{ \wh{y}_m , \wh{p}_m} & \leq \order \rbr{\rbr{\theta(\wh{y}_m)-\theta(\wh{p}_m )}^2} \\
& \leq \order \rbr{  \rbr{\theta(\wh{y}_m)-\theta(\optp )}^2+\rbr{\theta(\optp)-\theta(\wh{p}_m )}^2    } \\
& \leq \order \rbr{ \frac{\iota}{s_m}   +\rbr{\theta(\optp)-\theta(\wh{p}_m )}^2     },
\end{align*}
where the last inequality uses \pref{lem:concentration_optp_PKL}. 
We then show that
\begin{align*}
  &\abr{\theta(\optp)-\theta(\wh{p}_m )} \\
  &=\abr{ \theta(\optp)- \Clip_{[d_m,\pi-d_m]}\theta(\wh{y}_{m-1}) } \\
  &\leq \abr{\theta(\optp)-\theta(\wh{y}_{m-1})} +\inf_{w \in [d_m,\pi-d_m]}  |w-\theta(\optp)| \\
    &\leq \abr{\theta(\optp)-\theta(\wh{y}_{m-1})} + d_m \\
    &\leq 10\pi \sqrt{  \frac{\iota}{s_m}  }  . \numberthis{} \label{eq:simple_block_tool1_PKL}
\end{align*}

As a result, we have $\KL \rbr{ \wh{y}_m , \wh{p}_m} \leq \order (\iota/s_m)$, which further implies that $\PKL^{(m)} \leq \order(\iota)$. Summing over all epochs, we get the claimed bound.

To show \pref{eq:simple_block_verification_PKL}, we consider two cases.
If $d_m =8\pi \sqrt{\iota/s_m} <\pi/4$, then
\begin{align*}
    \abr{\theta(\wh{p}_m) -\theta(\wh{y}_m) } & \leq   \abr{\theta(\wh{p}_m) -\theta(\optp) }  +  \abr{\theta(\optp) -\theta(\wh{y}_m) }  \\
     & \leq \abr{\theta(\wh{p}_m) -\theta(\optp) } + 2\pi \sqrt{  \frac{\iota}{s_m} }  \\
    &\leq 12\pi \sqrt{  \frac{\iota}{s_m}  } \le 2d_m ,
\end{align*}
where the second inequality applies \pref{lem:concentration_optp_PKL} and last inequality uses \pref{eq:simple_block_tool1_PKL}.

Since $\abr{\theta(\wh{p}_m) -\theta(\wh{y}_m) } \leq 2d_m$ and $\theta(\wh{p}_m) =  \Clip_{[d_m,\pi-d_m]}\theta(\wh{y}_{m-1}) \in [d_m,\pi-d_m]$, we can apply \pref{lem:sup_sin_bound} to obtain the claimed inequality. 

If $d_m=\pi/4$, then $\theta(\wh{p}_m)  \in [\pi/4,3\pi/4]$, and thus the desired result trivially holds because $\sin \theta(\wh{p}_m)  \geq 1/\sqrt{2}$ and $\sup_z \sin z \leq 1$.
\end{proof}

\section{Omitted Proofs in \pref{sec:cal1}}
\label{app:proof_Cal1_theorem}

\setcounter{AlgoLine}{0}
\begin{algorithm}[t]
\DontPrintSemicolon
\caption{Calibration Algorithm for $\Caldist_1,\Caldist_2$~\citep{hu2025efficient}}
\label{alg:cal1_knownC}
\textbf{Input}: partition $\Pi=\{J_i\}_{i=1}^N$ of $[0,1]$.

\textbf{Initialize:} An instance of the \textsc{MsMwC} algoriothm of~\citet{chen2021impossible} over $2N$ experts.

\textbf{Initialize:} Prediction grids $\calZ = \{z_i\}_{i=1}^N$ where $\forall i \in [N]: z_i = \sup_{x \in J_i} x$.

\For{$t=1,\ldots,T$}{

Receive all weights $\omega_{t,i,\sigma}$ over all experts $(i,\sigma)$ from \textsc{MsMwC}.

For any $p \in [0,1]$, define
\[
\Phi_t(p) :=  \sum_{(i,\sigma)} \omega_{t,i,\sigma} \sigma  \Ind{ p \in J_i}.
\]

If $\Phi_t(0) > 0$, set $P_t \in \Delta([0,1])$ to put all mass on $0$; else if $\Phi_t(1) \leq 0$, set $P_t$ to put all mass on $1$; else choose $i \in \{0,\ldots,T-1\}$ such that $\Phi_t(i/T)\Phi_t((i+1)/T) \leq 0$ and $P_t$ such that
\begin{align*}
    P_t(i/T) = \frac{|\Phi_t((i+1)/T)| }{ |\Phi_t(i/T)| + |\Phi_t((i+1)/T)|} , \quad   P_t((i+1)/T) = \frac{|\Phi_t(i/T)|}{ |\Phi_t(i/T)| + |\Phi_t((i+1)/T)| }.
\end{align*}

Sample $\wt{p}_{t} \sim P_t$ and predict $p_t =z_{i_t} $, where $i_t \in [N]$ such that $\wt{p}_{t} \in J_{i_t}$.

Observe outcome $y_t \in \{0,1\}$.

For every expert $(i,\sigma)$, feed $ \loss_{t,i,\sigma} =    \E_{u \sim P_t} \sbr{ \sigma \Ind{u \in J_i} (u-y_t) \mid \calF_{t-1} }$ to \textsc{MsMwC}.

}
\end{algorithm}

\subsection{Proofs of Technical Lemmas}

\begin{lemma}[Restatement of \pref{lem:abs_cal_error_byHu}]
For any fixed partition $\Pi=\{J_i\}_{i \in [N]}$ of $[0,1]$ with $N=\order(\text{poly}(T))$, with probability at least $1-\delta/2$, for all $J \in \Pi$,
\begin{align*}
     \abr{ \sum_{t=1}^T \rbr{ y_t - z_{J} } \Ind{p_t = z_{J} } } \leq \order \rbr{  \sqrt{  \iota n_{J}} + \iota + n_{J} \Delta_{J} }.
\end{align*}
\end{lemma}

\begin{proof}
For any fixed interval $J \in \Pi$, we can show that
\begin{align*}
\abr{\sum_{t=1}^T (y_t-p_t)\Ind{p_t = z_J} } & \leq  \abr{\sum_{t=1}^T (y_t-\wt{p}_t)\Ind{p_t = z_J} }
+ \abr{\sum_{t=1}^T (\wt{p}_t-p_t)\Ind{p_t = z_J} } \\
& \leq  \abr{\sum_{t=1}^T (y_t-\wt{p}_t)\Ind{p_t = z_J} }
+ \sum_{t=1}^T \abr{\wt{p}_t-p_t}\Ind{p_t = z_J} \\
& \leq  \abr{\sum_{t=1}^T (y_t-\wt{p}_t)\Ind{p_t = z_J} }
+  \Delta_J n_J \\
& \leq  \order \rbr{ \iota +\sqrt{\iota n_J} }
+  \Delta_J n_J,
\end{align*}
where the third inequality uses the fact that if $p_t= z_J$, then $\abr{\wt{p}_t-p_t} \leq \Delta_J$, and the last inequality uses the same argument in \citet[Section 2.1]{hu2025efficient} to bound $\abr{\sum_{t=1}^T (y_t-\wt{p}_t)\Ind{p_t = z_J} }$ by properly scaling the confidence $\delta \in (0,1)$.

We complete the proof by a union bound over all $J \in \Pi$ and the fact that $N=\order(\text{poly}(T))$.
\end{proof}

\begin{lemma}[Freedman's inequality] \label{lem:freedman}
    Let $X_1,\ldots,X_n$ be a martingale difference sequence adapted to $\{\calF_t\}_{t=0}^n$. Assume that $|X_t| \leq B$ almost surely for all $t$. Define $S_n =\sum_{t=1}^n X_t$ and $V_n=\sum_{t=1}^n \E_t[X_t^2]$. Then, for any $\lambda \in  (0,1/B]$, with probability at least $1-\delta$
\begin{align*}
    |S_n| \leq \lambda V_n + \frac{ \log(2/\delta) }{\lambda}.
\end{align*}

Moreover, with probability at least $1-\delta$,
\begin{align*}
    |S_n| \leq \order \rbr{\sqrt{V_n \log(2/\delta)} + B \log(2/\delta)}.
\end{align*}
\end{lemma}

\begin{lemma}[Restatement of \pref{lem:visit_number_control_Cal1_main}] \label{lem:visit_number_control_Cal1}
For any fixed partition $\Pi=\{J_i\}_{i \in [N]}$ of $[0,1]$ with $N=\order(\text{poly}(T))$, with probability at least $1-\delta/2$, for all $i \in [N] \backslash \{j^*\}$,
\begin{equation}
\sqrt{n_{J_i}} \leq   \order \rbr{  \frac{\sqrt{\iota}}{d_{J_i}} + \sqrt{\frac{C_{J_i}}{d_{J_i}}  }    }.
\end{equation}
\end{lemma}

\begin{proof}
Fix any $i \in [N] \backslash  \{j^*\}$.
We choose
\[
\sigma_i^* = \begin{cases}
      1, & J_i \subseteq (\optp,1], \\
      -1 ,& J_i \subseteq [0,\optp).
   \end{cases}
\]

With this definition, for any $u \in J_i$, we have $\sigma_i^* (u-\optp) =|u-\optp|\geq d_{J_i}$. We further define for any $(i,\sigma)$ and any round $t$,
\begin{align*}
       \wt{\loss}_{t,i,\sigma} &:=    \E_{u \sim P_t} \sbr{ \sigma \Ind{u \in J_i} (u-q_t) \mid \calF_{t-1} },\\
         \loss_{t,i,\sigma} &:=    \E_{u \sim P_t} \sbr{ \sigma \Ind{u \in J_i} (u-y_t) \mid \calF_{t-1} } .
\end{align*}

We have that
\begin{align*}
 \wt{\loss}_{t,i,\sigma_i^*}
    & =  \E_{u \sim P_t} \sbr{ \sigma_i^* \Ind{u \in J_i} (u-\optp) \mid \calF_{t-1} }   +   \sigma_i^* (\optp-q_t)  \calP_t(z_{J_i})\\
    & \geq d_{J_i} \calP_t(z_{J_i})  - c_t \calP_t(z_{J_i}) .
\end{align*}

Further, we have $\{ \loss_{t,i,\sigma_i^*}   -  \wt{\loss}_{t,i,\sigma_i^*} \}_t =\sigma_i^* \calP_t(z_{J_i})(q_t-y_t)$ is a martingale difference sequence.
\[
\rbr{\loss_{t,i,\sigma_i^*}   -  \wt{\loss}_{t,i,\sigma_i^*}}^2 = (q_t-y_t)^2 \calP_t(z_{J_i})^2 \leq  \calP_t(z_{J_i})^2  \leq \calP_t(z_{J_i}).
\]

By Freedman's inequality (see \pref{lem:freedman}), with probability at least $1-\delta/4$, for all $i$
\[
\abr{  \sum_{t=1}^T \rbr{\loss_{t,i,\sigma_i^*}   -  \wt{\loss}_{t,i,\sigma_i^*}} } \leq \order \rbr{ \sqrt{\iota \sum_{t=1}^T \calP_t(z_{J_i})} +\iota  }.
\]

From analysis in \citep[Section 2.1]{hu2025efficient}, which does not rely on any specific partition of $[0,1]$,
$ \sum_{t=1}^T   \loss_{t,i,\sigma_i^*} \leq \order \rbr{\iota+\sqrt{\iota \sum_{t=1}^T \calP_t(z_{J_i})}}.$
Thus,
\begin{align*}
d_{J_i} \sum_{t=1}^T \calP_t(z_{J_i})  - \sum_{t=1}^T  c_t \calP_t(z_{J_i}) & \leq \sum_{t=1}^T     \wt{\loss}_{t,i,\sigma_i^*}    \\
&= \sum_{t=1}^T   \loss_{t,i,\sigma_i^*}   - \sum_{t=1}^T  \rbr{  \loss_{t,i,\sigma_i^*} -  \wt{\loss}_{t,i,\sigma_i^*}  } \\
&\leq \sum_{t=1}^T  \loss_{t,i,\sigma_i^*}   + \abr{  \sum_{t=1}^T  \rbr{  \wt{\loss}_{t,i,\sigma_i^*} -\loss_{t,i,\sigma_i^*}  } }\\
&\leq    \order \rbr{\iota+ \sqrt{\iota \sum_{t=1}^T \calP_t(z_{J_i})}} .
\end{align*}

The above inequality $d_{J_i} \sum_{t=1}^T \calP_t(z_{J_i}) \leq C_{J_i}+  \order \rbr{\iota+ \sqrt{\iota \sum_{t=1}^T \calP_t(z_{J_i})}}$ implies that (here we only consider $d_{J_i}>0$. If $d_{J_i}=0$, we treat $1/d_{J_i}$ as $+\infty$.)
\begin{equation}
   \sum_{t=1}^T \calP_t(z_{J_i}) \leq \order  \rbr{ \frac{\iota}{d_{J_i}^2} + \frac{ C_{J_i}}{d_{J_i}}   }.
\end{equation}

Moreover, let $Y_{t,i}=\Ind{p_t=z_{J_i}}-\calP_t(z_{J_i})$, and $\{Y_{t,i}\}_t$ is a martingale difference sequence.
We have $|Y_{t,i}| \leq 1$ and $\sum_{t=1}^T \E_t[Y_{t,i}^2] \leq \sum_{t=1}^T \calP_t(z_{J_i}) $.
By Freedman's inequality (see \pref{lem:freedman}) and a union bound over all $i$, with probability at least $1-\delta/4$, for all $i$
\[
\abr{ n_{J_i} - \sum_{t=1}^T \calP_t(z_{J_i})  }  \leq \order \rbr{ \sqrt{\iota \sum_{t=1}^T \calP_t(z_{J_i})} +\iota  },
\]
which implies $n_{J_i}  \leq  \order  \rbr{ \frac{\iota}{d_{J_i}^2} + \frac{C_{J_i}}{d_{J_i}}   }.$ Thus, $\sqrt{n_{J_i}} \leq \order{\rbr{  \sqrt{ \frac{\iota}{d_{J_i}^2} + \frac{C_{J_i}}{d_{J_i}} }} } \leq \order \rbr{  \frac{\sqrt{\iota}}{d_{J_i}} + \sqrt{\frac{C_{J_i}}{d_{J_i}}  }    }$.

One can repeat this argument for all $i \in [N] \backslash  \{j^*\}$ to complete the proof.
\end{proof}

\subsection{Proof of \pref{thm:CalL1_knownC}}
For this proof, the specific choice of $N$ is
\begin{equation} \label{eq:choice_N_cal1}
N = \left \lceil \min \cbr{  \sqrt{ \frac{T}{\iota \log T} }, \rbr{   \frac{T^2}{\iota (1+C) \log T} }^{1/3} } \right \rceil.
\end{equation}

Since the computational complexity per round is $\order(N)$ and $N$ is given in \pref{eq:choice_N_cal1}, the claimed computational complexity thus follows.

The following analysis conditions on the nice event that the high probability bounds in both \pref{lem:abs_cal_error_byHu} and \pref{lem:visit_number_control_Cal1} hold simultaneously. Then, such a nice event holds with probability at least $1-\delta$.
One can show
\begin{align*}
\Caldist_1 &= \sum_{i=1}^N      \abr{ \sum_{t=1}^T \rbr{ y_t - z_{J_i} } \Ind{p_t = z_{J_i} } } \\
&\leq  \order \rbr{\sum_{i=1}^N  \sqrt{\iota n_{J_i}}+\iota N +\frac{T}{N}  }  \\
&=  \order \rbr{\sum_{i \in \calN}  \sqrt{\iota n_{J_i}}  +\sum_{i \in [N] \backslash \calN}  \sqrt{\iota n_{J_i}} +\frac{T}{N} +\iota N   } \\
&\leq   \order \rbr{\sqrt{T \iota}  +\sum_{i \in [N] \backslash \calN}  \sqrt{\iota n_{J_i}} +\frac{T}{N}  +\iota N } ,
\end{align*}
where the first inequality uses \pref{lem:abs_cal_error_byHu} and the fact that $\sum_{i =1}^N n_{J_i}=T$, and the last inequality holds since $|\calN|$ is constant level.

Using \pref{eq:bound4ni}, we have
\begin{align*}
    \sum_{i \in  [N] \backslash \calN}  \sqrt{n_{J_i}} \leq \order \rbr{    \sum_{i \in  [N] \backslash \calN}  \frac{\sqrt{\iota}}{d_{J_i}} + \sum_{i \in  [N] \backslash \calN}  \sqrt{\frac{C_{J_i}}{d_{J_i}}}   }.
\end{align*}

On the one hand, we use the fact that for any $i \in  [N] \backslash \calN$, we have $|i-j^*| \geq 2$, and $d_{J_i} \geq \frac{|i-j^*|}{2N}$ to show that
\begin{align} \label{eq:bound_oneoverdi}
    \sum_{i \in  [N] \backslash \calN}  \frac{1}{d_{J_i}}  \leq \order \rbr{     \sum_{i \in  [N] \backslash \calN}  \frac{N}{|i-j^*|}   }  \leq \order \rbr{     \sum_{z=2}^N  \frac{N}{z}   } \leq \order \rbr{ N \log N} \leq \order \rbr{ N \log T}.
\end{align}

On the other hand, we use the Cauchy--Schwarz inequality to show
\begin{align*}
\sum_{i \in  [N] \backslash \calN}  \sqrt{\frac{C_{J_i}}{d_{J_i}}}   \leq   \sqrt{  \sum_{i \in  [N] \backslash \calN} \frac{1}{d_{J_i}}}  \sqrt{  \sum_{i \in  [N] \backslash \calN} C_{J_i}} \leq \order \rbr{   \sqrt{NC \log T}     },
\end{align*}
where the last inequality follows from \pref{eq:bound_oneoverdi} and $\sum_{i \in [N]} C_{J_i} = \sum_{t=1}^T c_t  \sum_{i \in [N]} \calP_t(z_{J_i})  =C$. Putting together, we have
\[
\Caldist_1 \leq \order  \rbr{  \sqrt{T \iota} +\frac{T}{N} +\iota N \log T + \sqrt{ \iota NC \log T}    }.
\]

It remains to plug in the choice of \(N\). If \(\iota\log T>T\), then the
claimed bound is at least order \(T\), and the result is trivial since
\(\Caldist_1\le T\). Thus assume \(\iota\log T\le T\). The ceiling in the
definition of \(N\) only changes the following bounds by constants.

By \pref{eq:choice_N_cal1}, we have
\[
\frac{T}{N}
\le
\order \left(
\sqrt{\iota T\log T}
+
(\iota T(1+C)\log T)^{1/3}
\right)
\le
\order \left(
\sqrt{\iota T\log T}
+
(\iota T C\log T)^{1/3}
\right),
\]
where the last step uses \((1+C)^{1/3}\le 1+C^{1/3}\) and absorbs
\((\iota T\log T)^{1/3}\) into \(\sqrt{\iota T\log T}\). Also,
\[
\iota N\log T
\le
\order \left(
\iota\log T\sqrt{\frac{T}{\iota\log T}}
\right)
=
\order \left(\sqrt{\iota T\log T}\right),
\]
and
\[
\sqrt{\iota N C\log T}
\le
\order \left(
\sqrt{\iota C\log T}
\left(\frac{T^2}{\iota(1+C)\log T}\right)^{1/6}
\right)
\le
\order \left((\iota T C\log T)^{1/3}\right),
\]
where the last inequality uses \(C^{1/2}/(1+C)^{1/6}\le C^{1/3}\).
Putting these three bounds together gives
\[
\Caldist_1
\le
\order \left(
\sqrt{\iota T\log T}
+
(\iota T C\log T)^{1/3}
\right),
\]
which completes the proof.

\section{Omitted Proofs in \pref{sec:cal2_approach1}}

\begin{definition}[$\UnifPart$]\label{def:unifpart}
Given an interval $I=[l,u]$ and an integer $N \geq 1$, the function $\UnifPart(I,N)$ uniformly partitions interval $I$ into a collection of $N$ equal-length subintervals.
If $N=1$, $\UnifPart(I,N)=\{I\}$. If $N \geq 2$, then
\[
\UnifPart(I,N) = \cbr{ \left[  l+\frac{k-1}{N}\Delta_I,l+\frac{k}{N}\Delta_I \right) :k =1,\ldots,N -1  } \cup \cbr{  \sbr{l+\frac{N-1}{N}\Delta_I,u}   }.
\]
\end{definition}

\subsection{Technical Lemmas and Nice Event Construction}

\begin{lemma} \label{lem:corruption_analysis_Cal2_Approach1}
With probability at least $1-\delta/3$, for each $m \geq 2$
    \begin{equation*}
    \abr{\wh{y}_m -\optp} \leq  r_m =    \sqrt{ \frac{ \iota  }{s_{m-1}} }  + \frac{C}{s_{m-1}}.
\end{equation*}
\end{lemma}

\begin{proof}
Fix an epoch $m \geq 2$.
We have
\begin{align*}
 \abr{\wh{y}_m -\optp} & \leq \frac{1}{s_{m-1}} \abr{\sum_{t \in \calT_{m-1}} y_t  -  \sum_{t \in \calT_{m-1}} q_t  } + \frac{1}{s_{m-1}}\abr{ \sum_{t \in \calT_{m-1}} \rbr{q_t  -\optp}} \leq \sqrt{ \frac{ \iota  }{s_{m-1}} }  + \frac{C}{s_{m-1}},
\end{align*}
where the last inequality uses Hoeffding's inequality for zero-mean and $[-1,1]$-bounded random variable $y_t-q_t$.

Using the fact that the total number of epochs is at most $\order(\log T)$ and a union bound over all epochs, we complete the proof.
\end{proof}

\begin{lemma} \label{lem:block_version_bound}
With probability at least $1-\delta/3$, for each $m \geq 2$, we have that for all $J \in \Pi_m$,
\begin{align*}
   \abr{  \sum_{t \in \calT_m } \rbr{ y_t -p_t } \Ind{p_t = z_{J}} } \leq \order \rbr{ \sqrt{ \iota \cdot  n_{m,J}  } + \iota + n_{m,J}  \Delta_J } .
\end{align*}
\end{lemma}
\begin{proof}
For this proof, we condition on the history prior to epoch $m$, under which the partition $\Pi_m$ is fixed. We then repeat the argument of \pref{lem:abs_cal_error_byHu} within epoch $m$. Finally, applying a union bound over all epochs, together with the fact that the total number of epochs is at most $\order(\log T)$, completes the proof.
\end{proof}

\begin{lemma} \label{lem:visit_number_control_Cal2_block_main}
With probability at least $1-\delta/3$, for each $m \geq 2$ and $J \in \Pi_m^{\text{out}}$, we have that
\[
n_{m,J} \leq \order \rbr{ \frac{\iota}{d_{J}^2} +\frac{C_{m,J}}{d_{J}}  }, \ \text{where}\ \forall J \in \Pi: C_{m,J} = \sum_{t\in \calT_m} c_t  \calP_t \rbr{z_{J}} .
\]
\end{lemma}
\begin{proof}
For this proof, we condition on the history prior to epoch $m$, under which the partition $\Pi_m$ is fixed. We then repeat the argument of \pref{lem:visit_number_control_Cal1_main} within epoch $m$ to show $n_{m,J} \leq   \order \rbr{ \rbr{ \frac{\sqrt{\iota}}{d_{J}} + \sqrt{\frac{C_{m,J}}{d_{J}}  } }^2   } \leq \order \rbr{ \frac{\iota}{d_{J}^2} +\frac{C_{m,J}}{d_{J}}  }$. Finally, applying a union bound over all epochs, together with the fact that the total number of epochs is at most $\order(\log T)$, completes the proof.
\end{proof}

\begin{definition}[Nice event $\calE_1$] \label{def:nice_event_Cal2_approach1}
Let $\calE_1$ be the event that all high probability bounds in \pref{lem:visit_number_control_Cal2_block_main}, \pref{lem:corruption_analysis_Cal2_Approach1}, and \pref{lem:block_version_bound} hold simultaneously.
\end{definition}

\begin{corollary} \label{corr:Cal2m_bound_approach1}
Suppose that $\calE_1$ holds where $\calE_1$ is defined in \pref{def:nice_event_Cal2_approach1}.
For each $m \geq 2$, we have
\begin{align*}
     \Caldist_2^{(m)}  = \sum_{J \in \Pi_m: n_{m,J}>0} \frac{1}{n_{m,J}} \abr{ \sum_{t \in \calT_m: p_t =z_{J}} (y_t -z_{J}) }^2 \leq  \order \rbr{ |\Pi_m| \iota + \sum_{ J \in \Pi_m }  n_{m,J}  \Delta_J^2   } .
\end{align*}
\end{corollary}
\begin{proof}
Fix an epoch $m \geq 2$.
Consider an arbitrary interval $J \in \Pi_m$ such that $n_{m,J}> 0$.

Now we consider two cases. If $0<n_{m,J} <\iota$, then
\begin{align*}
    \frac{1}{n_{m,J}} \abr{ \sum_{t \in \calT_m: p_t =z_{J}} (y_t -z_{J}) }^2  \leq    n_{m,J} \leq \iota .
\end{align*}

If $n_{m,J} \geq \iota$, with \pref{lem:block_version_bound}, we can show that
\begin{align*}
& \frac{1}{n_{m,J}} \abr{ \sum_{t \in \calT_m: p_t =z_{J}} (y_t -z_{J}) }^2 \\
&\leq \order \rbr{ \frac{1}{n_{m,J}} \rbr{   \sqrt{ \iota \cdot  n_{m,J}  } + \iota + n_{m,J}  \Delta_J   }^2    } \\
&\leq \order \rbr{ \iota +  \frac{\iota^2}{ n_{m,J}  }  +  n_{m,J}  \Delta_J^2   }  \\
&\leq \order \rbr{ \iota +  n_{m,J}  \Delta_J^2   } ,
\end{align*}
where the second inequality uses $(a+b+c)^2 \leq 3(a^2+b^2+c^2)$, and the last inequality uses $\iota \leq n_{m,J}$.

Summing over all $J \in \Pi_m$, we have the claimed bound for epoch $m$. Repeating this argument for each $m \geq 2$ completes the proof.
\end{proof}

\begin{lemma} \label{lem:outer_bound_Cal2_approach1}
Suppose that $\calE_1$ holds where $\calE_1$ is defined in \pref{def:nice_event_Cal2_approach1}.
For any $m \geq 2$, we have
\begin{align*}
|\Pi_{m}^{\text{out}}| \iota + \sum_{ J \in \Pi_{m}^{\text{out}} }  n_{m,J}  \Delta_J^2  \leq  \order  \rbr{\iota \log T+\iota^{2/3}  (1+C)^{1/3} \log T }.
\end{align*}
\end{lemma}

\begin{proof}
Fix any epoch $m \geq 2$. One can write (note that the summation over $J$ is naturally and implicitly only over those nonempty $J$'s)
\begin{align*}
  \sum_{ J \in \Pi_{m}^{\text{out}} }  n_{m,J}  \Delta_J^2 = \sum_{ q =0}^{Q} \sum_{J \in \calL_{m,q}}  n_{m,J}  \Delta_J^2    +\sum_{ q =0}^{Q} \sum_{J \in \calR_{m,q}}  n_{m,J}  \Delta_J^2  .
\end{align*}

For any non-empty $J \in \calL_{m,q}$, we have $d_J \geq  2^q r_{m}$ and $\Delta_J \leq \frac{2^q r_{m}}{K}$. We write that
\begin{align*}
\sum_{ q =0}^{Q} \sum_{J \in \calL_{m,q}}  n_{m,J}  \Delta_J^2
&\leq  \order \rbr{  \sum_{ q =0}^{Q} \sum_{J \in \calL_{m,q}}  \rbr{   \frac{2^q r_m}{K} }^2 \rbr{ \frac{\iota }{d_{J}^2} +\frac{C_{m,J}}{d_{J}}   }  }\\
& \leq  \order \rbr{ \sum_{ q =0}^{Q} \sum_{J \in \calL_{m,q}}  \rbr{   \frac{2^q r_m}{K} }^2 \rbr{ \frac{\iota }{ (2^q r_m)^2 } +\frac{C_{m,J}}{2^q r_m}   }   }\\
& \leq \order \rbr{ \sum_{ q =0}^{Q} \sum_{J \in \calL_{m,q}}   \rbr{ \frac{\iota }{ K^2 } +\frac{C_{m,J} 2^q r_m}{ K^2}   }} \\
& \leq \order \rbr{   \sum_{ q =0}^{Q}  \rbr{ \frac{\iota }{ K } +\frac{C 2^q r_m}{ K^2}   } } \\
& \leq \order \rbr{ \sum_{ q =0}^{Q}  \rbr{\iota +\frac{C 2^q r_m}{ K^2}   }} \\
& \leq \order \rbr{ \log T  \rbr{\iota +\frac{C}{ K^2}   }},
\end{align*}
where the first inequality uses \pref{lem:visit_number_control_Cal2_block_main}, and the last inequality bounds $2^q r_m \leq d_J \leq 1$.

Similarly, we also have
\begin{align*}
    \sum_{ q =0}^{Q} \sum_{J \in \calR_{m,q}}  n_{m,J}  \Delta_J^2   \leq  \order \rbr{  \log T \rbr{ \iota +\frac{C }{ K^2}   } }.
\end{align*}

Moreover, $|\Pi_{m}^{\text{out}}| = \order \big(  K \log T \big)$. Then, one can show that
\begin{align*}
  |\Pi_{m}^{\text{out}}| \iota + \sum_{ J \in \Pi_{m}^{\text{out}} }  n_{m,J}  \Delta_J^2
 &\leq \order \rbr{ \log T  \rbr{  \iota K +\frac{C }{ K^2} }   } \leq \order \rbr{ \iota \log T + \iota^{2/3}  (1+C)^{1/3} \log T    }.
\end{align*}

The proof is thus complete.
\end{proof}

\begin{lemma} \label{lem:inner_error_Cal2_approach1}
The following holds.
\begin{align*}
\sum_{m=2}^M \rbr{N_m \iota  + s_{m} \frac{|I_{m}|^2}{N_m^2}} \leq  \order \rbr{   \iota^{2/3} C^{1/3}   + \iota \log  T }.
\end{align*}
\end{lemma}
\begin{proof}
Consider any epoch $m \geq 2$.
We can show that
\begin{align*}
& N_m \iota  + s_{m} \frac{|I_{m}|^2}{N_m^2}  \\
& =  \order \rbr{\iota+  s_m^{1/3}  (\iota |I_{m}|)^{2/3}   } \tag{Choice of $N_m$} \\
& \leq  \order \rbr{\iota+   s_m^{1/3} \iota^{2/3} \min \cbr{1,\sqrt{\frac{\iota }{ s_m} } + \frac{C}{s_m} }^{2/3}   } ,
\end{align*}
where the last inequality follows from the definition of $I_m$ and the fact that $s_{m}=2s_{m-1}$:
\begin{equation} \label{eq:bound4Im}
    |I_m| \leq \min \cbr{1,4r_m}  \leq \order \rbr{ \min \cbr{1,\sqrt{\frac{\iota}{ s_{m-1}} } + \frac{C}{s_{m-1}} }  }  \leq \order \rbr{ \min \cbr{1,\sqrt{\frac{\iota}{ s_{m}} } + \frac{C}{s_{m}} }  }.
\end{equation}

We then consider three cases.

\textbf{Case 1: $s_m \leq C$.} In this case, we have
\begin{align*}
     s_m^{1/3}   \min \cbr{1, \sqrt{\frac{\iota }{ s_m} } + \frac{C}{s_m} }^{2/3}  \leq   s_m^{1/3}  .
\end{align*}

\textbf{Case 2: $C< s_m \leq C^2/\iota$.}
This case only exists when $C \geq \iota$.
We have $\sqrt{\iota/s_m} \leq C/s_m$ and then
\begin{align*}
     s_m^{1/3}   \min \cbr{1,\sqrt{\frac{\iota }{ s_m} } + \frac{C}{s_m} }^{2/3} \leq s_m^{1/3}   \min \cbr{1,  \frac{2C}{s_m} }^{2/3}  \leq  2C^{2/3} s_m^{-1/3} ,
\end{align*}
where the last inequality lower bounds $s_m >C$.

\textbf{Case 3: $s_m > \max \{C^2 /\iota,C\} $.} As $s_m >C^2/\iota$, we have $\sqrt{\iota/s_m} > C/s_m$ and then
\begin{align*}
     s_m^{1/3}   \min \cbr{1,\sqrt{\frac{\iota }{ s_m} } + \frac{C}{s_m} }^{2/3}  \leq    s_m^{1/3}  \min \cbr{1,2\sqrt{\frac{\iota }{ s_m} }   }^{2/3}  \leq \order \rbr{ \iota^{1/3}},
\end{align*}

Then, we can write
\begin{align*}
&\sum_{m=2}^M \rbr{N_m \iota  + s_{m} \frac{|I_{m}|^2}{N_m^2} } \\
&\leq \order \rbr{ \sum_{m=2}^M \rbr{\iota+  s_m^{1/3}  (\iota |I_{m}|)^{2/3} } } \\
&\leq \order(\iota \log T)+  \iota^{2/3} \cdot \order \rbr{ \sum_{m: s_m \leq C}  s_m^{1/3}  +\sum_{m: C< s_m \leq C^2/\iota } C^{2/3} s_m^{-1/3} +\sum_{m: s_m > \max \{C, C^2/\iota \}} \iota^{1/3} } \\
&\leq  \order(\iota \log T)+\iota^{2/3} \cdot \order \rbr{ C^{1/3}   +\iota^{1/3} \log  T  } \\
&=  \order \rbr{  \iota^{2/3} C^{1/3}   +\iota \log  T  }.
\end{align*}

\end{proof}

\subsection{Proofs of \pref{thm:cal2_approach1}}
The computational complexity is dominated by that of the base algorithm, which scales linearly with the number of discretization points. Therefore, it suffices to bound $|\Pi_m|$ for each epoch $m$.
For the first epoch, the algorithm may predict arbitrary values, so we can simply force it to deterministically output a fixed prediction, which incurs only $\order(1)$ computational cost. For subsequent epochs, using $|\Pi_m^{\text{out}}| = \order(K \log T)$ together with \pref{eq:bound4Im}, we obtain, for any $m \ge 2$,
\begin{align*}
    |\Pi_{m}| & = |\Pi_{m}^{\text{in}}| + |\Pi_{m}^{\text{out}}| = N_{m}+ |\Pi_{m}^{\text{out}}|\leq \order \rbr{  \rbr{s_{m}}^{1/3} |I_{m}|^{2/3}+ K \log T  }  \leq \otil \rbr{ (1+C)^{1/3}  }.
\end{align*}
where the second inequality uses a similar analysis in \pref{lem:inner_error_Cal2_approach1} to bound $\rbr{s_{m}}^{1/3} |I_{m}|^{2/3}$ and uses the choice of $K$ to bound the remaining term.

The following proof conditions on the nice event $\calE_1$ defined in \pref{def:nice_event_Cal2_approach1}.
\pref{corr:decomp_Cal2_into_block} shows that $\Caldist_2 \leq \sum_{m=1}^M \Caldist_2^{(m)}$, where $\Caldist_2^{(m)}= \sum_{J \in \Pi_m: n_{m,J}>0} \frac{1}{n_{m,J}} \abr{ \sum_{t \in \calT_m: p_t =z_{J}} (y_t -z_{J}) }^2$, and then it suffices to bound each $\Caldist_2^{(m)}$.  \pref{corr:Cal2m_bound_approach1} gives that for any $m \geq 2$,
\begin{align*}
     \Caldist_2^{(m)}  &\leq  \order \rbr{ |\Pi_m| \iota + \sum_{ J \in \Pi_m }  n_{m,J}  \Delta_J^2   }  \\
     &=  \order \rbr{ N_{m} \iota + \sum_{ J \in \Pi_m^{\text{in}} }  n_{m,J}  \Delta_J^2   } +\order \rbr{ |\Pi_m^{\text{out}}| \iota + \sum_{ J \in \Pi_m^{\text{out}} }  n_{m,J}  \Delta_J^2   }.
\end{align*}

Using the facts that $\sum_{J \in \Pi_m^{\text{in}}} n_{m,J} \leq s_m$ for any $m$, $\Delta_J = \frac{|I_{m}|}{N_{m}}$ for any $J \in \Pi_{m}^{\text{in}}$ and $|\Pi_{m}^{\text{in}}|=N_{m}$ to show
\begin{align*}
 N_{m}\iota + \sum_{ J \in \Pi_m^{\text{in}} }  n_{m,J}  \Delta_J^2
& \leq \order   \rbr{N_{m}\iota  +s_m \frac{|I_{m}|^2}{N_{m}^2} } .
\end{align*}

On the other hand, \pref{lem:outer_bound_Cal2_approach1} gives
\begin{align*}
     |\Pi_m^{\text{out}}| \iota + \sum_{ J \in \Pi_m^{\text{out}} }  n_{m,J}  \Delta_J^2  \leq \order  \rbr{\iota \log T + \iota^{2/3}  (1+C)^{1/3}  }.
\end{align*}

Using $M=\order(\log T)$ and bounding $\Caldist_2^{(1)} = \order(1)$, we have
\begin{align*}
    \Caldist_2 &\leq \Caldist_2^{(1)} +\sum_{m=2}^M \Caldist_2^{(m)} \\
    &\leq \order \rbr{\iota \log^2 T+ \log^2 T\iota^{2/3}  (1+C)^{1/3} +\sum_{m=2}^M \rbr{ N_{m}\iota  +s_m \frac{|I_{m}|^2}{N_{m}^2} }  } \\
      &\leq \order \rbr{\iota \log^2 T+ \log^2 T\iota^{2/3}  (1+C)^{1/3}    + \iota  \log  T} \\
        &\leq \order \rbr{\iota \log^2 T+ \log^2 T\iota^{2/3}  C^{1/3}   },
\end{align*}
where the third inequality uses \pref{lem:inner_error_Cal2_approach1}.

The proof is thus complete.

\section{Omitted Details for Achieving $\PCaldist_2 =\otil \big((1+C)^{1/3} \big)$} \label{app:PCal2}

Throughout this section, the algorithm of \citet{fishelson2025full} refers specifically to the algorithm in their paper that achieves the guarantee $\PCaldist_2 = \otil(T^{1/3})$.

For any $s \in \calZ_m$, let $ \ExtReg^{(m)}_s$ be the pseudo external regret of the $s$-th instance of OGD in epoch $m$, which is defined as
\begin{align}  \label{eq:pseudo_ext_reg}
 \ExtReg^{(m)}_s =   \sum_{t \in \calT_m}   \calP_t(s)  \loss(a_{s,t},y_t)   - \inf_{k \in  [0,1]} \sum_{t \in \calT_m}   \calP_t(s) \loss(k,y_t)  .
\end{align}

\subsection{Technical Lemmas and Nice Event Construction}

\begin{lemma}[Restatement of \pref{lem:bound_occupancy_PCal2}]
With probability at least $1-\delta/2$, for any $m \geq 2$, we have
\begin{align*}
\sum_{t \in \calT_m}   \sum_{s \in \calZ_m}  \calP_t(s) (a_{s,t}-\optp)^2 \leq \order \rbr{|\Pi_m|\log T +C_m + \iota },
\end{align*}
where $C_m = \sum_{t \in \calT_m} c_t$.
\end{lemma}

\begin{proof}
For shorthand, we denote
\[
v_t := \sum_{s \in \calZ_m}  \calP_t(s) (a_{s,t}-\optp),\quad \Gamma_m =  \sum_{t \in \calT_m}   \sum_{s \in \calZ_m}  \calP_t(s) (a_{s,t}-\optp)^2.
\]

Our objective is to bound $\Gamma_m$.
For any $s \in \calZ_m$, \citet[Lemma 24]{fishelson2025full} shows that the external regret guarantee of OGD ensures
\begin{align*}
\sum_{t \in \calT_m} \calP_t(s) \rbr{ (a_{s,t}-y_t)^2 - (\optp-y_t)^2  } \leq \ExtReg^{(m)}_s \leq \order \rbr{\log T}.
\end{align*}

We also have
\begin{align*}
(a_{s,t}-y_t)^2 - (\optp-y_t)^2  = (a_{s,t}-\optp)^2  -2(a_{s,t}-\optp)(y_t-\optp) .
\end{align*}

Multiplying the above equality by $\calP_t(s)$, and then summing over all $s \in \calZ_m$ and all $t \in \calT_m$, yields
\begin{align*}
\Gamma_m
& =\sum_{t \in \calT_m}   \sum_{s \in \calZ_m}  \calP_t(s) (a_{s,t}-\optp)^2 \\
&\leq \sum_{s \in \calZ_m} \sum_{t \in \calT_m} \calP_t(s) \rbr{ (a_{s,t}-y_t)^2 - (\optp-y_t)^2  }    +  2 \abr{ \sum_{t \in \calT_m} v_t  (y_t-\optp) } \\
&\leq \sum_{s \in \calZ_m} \sum_{t \in \calT_m} \calP_t(s) \rbr{ (a_{s,t}-y_t)^2 - (\optp-y_t)^2  }    +  2 \abr{ \sum_{t \in \calT_m} v_t  (q_t-\optp) }  +  2 \abr{ \sum_{t \in \calT_m} v_t   (q_t-y_t) } \\
&\leq \order (|\Pi_m|\log T)   +  2 \abr{ \sum_{t \in \calT_m} v_t  (q_t-\optp) }    +  2 \abr{ \sum_{t \in \calT_m} v_t  (y_t-q_t) } .
\end{align*}

Let $Y_t = v_t  (y_t-q_t) $. As $v_t$ is $\calF_{t-1}$-measurable, $\E_t[Y_t] = v_t\E_t [y_t-q_t] =0$. In addition, $|v_t|\leq 1, |y_t-q_t| \in [0,1]$ gives $|Y_t| \leq 1$. One can further show that
\begin{align*}
   \E_t \sbr{ Y_t^2 }    = v_t^2   \E_t [(y_t-q_t)^2] \leq  v_t^2 \leq      \sum_{s \in \calZ_m } \calP_t(s) (a_{s,t}-\optp)^2,
\end{align*}
where the last inequality uses the Cauchy--Schwarz to show
\begin{align} \label{eq:vtsquare_bound}
   v_t^2=
   \rbr{\sum_{s \in \calZ_m }\sqrt{ \calP_t(s)} \sqrt{ \calP_t(s)} (a_{s,t}-\optp)}^2 \leq      \sum_{s \in \calZ_m } \calP_t(s) (a_{s,t}-\optp)^2 .
\end{align}

By Freedman's inequality (see \pref{lem:freedman}) and union bound over all $m$, with probability at least $1-\delta/2$, for all $m$,
\begin{align*}
    \abr{ \sum_{t \in \calT_m} v_t  (q_t-y_t) }  \leq  \frac{1}{8}    \Gamma_m+ 8 \iota.
\end{align*}

On the other hand, we have
\begin{align*}
\abr{ \sum_{t \in \calT_m} v_t (q_t-\optp) } &\leq \sum_{t \in \calT_m} c_t |v_t| \\
&\leq \sqrt{\sum_{t \in \calT_m} c_t  } \sqrt{\sum_{t \in \calT_m} c_t v_t^2} \\
&\leq \sqrt{\sum_{t \in \calT_m} c_t  } \sqrt{\sum_{t \in \calT_m}  v_t^2} \tag{$c_t \leq 1$} \\
 &\leq \sqrt{C_m } \sqrt{\Gamma_m } \tag{by \pref{eq:vtsquare_bound}} \\
  &\leq \frac{1}{8}\Gamma_m  +2 C_m \tag{AM-GM inequality}
\end{align*}

Combining the above results, we have
\begin{align*}
\Gamma_m \leq \order \rbr{|\Pi_m|\log T +C_m +\iota  }  + \frac{1}{2} \Gamma_m  .
\end{align*}

Rearranging the above gives the claimed bound.
\end{proof}

\begin{definition}[Nice event $\calE_2$] \label{def:nice_event_PCal2}
Let $\calE_2$ be the event that all high probability bounds in \pref{lem:bound_occupancy_PCal2} and \pref{lem:corruption_analysis_Cal2_Approach1} hold simultaneously.
\end{definition}

\begin{lemma} \label{lem:bound4_DeltaJ_abs_gap_Cal2}
Suppose that $\calE_2$ holds where $\calE_2$ is defined in \pref{def:nice_event_PCal2} and $K=\big\lceil (1+C)^{1/3} \iota^{-1} \big \rceil$.
For any $m \geq 2$, any $J \in \Pi_{m}^{out}$, and any $a \in J$, we have $\Delta_J \leq \frac{|a-\optp|}{K}$.
\end{lemma}
\begin{proof}
Fix an epoch $m \geq $ and fix an index $q$. Recall that $\Pi_{m}^{\text{out}}= \{\calL_{m,q}\}_{q=0}^{Q} \cup \{\calR_{m,q}\}_{q=0}^{Q}$. We first show that for any $J \in \calL_{m,q}$ and any $a \in J$, $\Delta_J \leq \frac{|a-\optp|}{K}$ holds.
As $\optp \in \sbr{ \wh{y}_m - r_m,\wh{y}_m +r_m  } $ by \pref{lem:corruption_analysis_Cal2_Approach1}, we have $\optp \geq \wh{y}_m - r_m$. Since $J \in \calL_{m,q}$, $a \in J$, and $L_{m,q} = [\wh{y}_m - 2r_m -(2^{q+1}-1) r_m , \wh{y}_m - 2r_m -(2^{q}-1) r_m ) \cap [0,1]$, we have
\[
a \leq \wh{y}_m - 2r_m -(2^{q}-1) r_m .
\]

Then, one can show that
\begin{align*}
\optp -a \geq     \wh{y}_m - r_m  - \rbr{ \wh{y}_m - 2r_m -(2^{q}-1) r_m }     \geq 2^q r_m \geq |L_{m,q}|.
\end{align*}

Therefore, we have
\begin{align*}
\Delta_J \leq \frac{ |L_{m,q}|}{K} \leq \frac{ \optp -a }{K} = \frac{ |\optp -a |}{K}.
\end{align*}

One can symmetrically repeat this argument to show for any $J \in \calR_{m,q}$ and any $a \in J$, $\Delta_J \leq \frac{|a-\optp|}{K}$ holds. Finally, repeating this argument for all $q,m$ completes the proof.
\end{proof}

\begin{lemma} \label{lem:two_points_dist}
Let $\loss(s,y_t)=(s-y_t)^2$ where $y_t \in \{0,1\}$.
For any epoch $m \geq 2$, any interval $J \in \Pi_m$ and any point $q \in J$, we have $\E_{s \sim H(q)}[\loss(s,y_t)]-\loss(q,y_t) \leq \Delta_J^2/4$ for all $t \in \calT_m$.
\end{lemma}

\begin{proof}
Suppose that $J=[a,b]$. By the rounding procedure in the algorithm of \citet{fishelson2025full},
\begin{align*}
   H(q)(a) =   \frac{b-q}{b-a},\qquad    H(q)(b) = \frac{q-a}{b-a}.
\end{align*}

We have $\E_{s \sim H(q)}[s]=q$.
On the one hand, as $y_t$ is independent of $s \sim H(q)$, we have
\begin{align*}
    \E_{s \sim H(q)}[\loss(s,y_t)] &= \E_{s \sim H(q)} \sbr{ (s-y_t)^2 } \\
    &= \E_{s \sim H(q)} \sbr{s^2 -2sy_t +y_t^2}\\
    & = \E_{s \sim H(q)}  \sbr{ s^2  } -2qy_t +y_t^2 .
\end{align*}

Then, we can show that
\begin{align*}
  \E_{s \sim H(q)}[\loss(s,y_t)] -   \loss(q,y_t) &  = \E_{s \sim H(q)}  \sbr{ s^2  } - q^2 \\
  &= \frac{b-q}{b-a} a^2 +  \frac{q-a}{b-a} b^2 -q^2  \\
  &= (q-a)(b-q).
\end{align*}

We further show that
\begin{align*}
    (q-a)(b-q)  \leq \rbr{ \frac{(q-a)+(b-q)}{2}  }^2 =\frac{(b-a)^2}{4}=\frac{\Delta_J^2}{4}.
\end{align*}

The claimed result thus follows.
\end{proof}

\subsection{Main Results for $\PCaldist_2$ and $\Caldist_2$}

\begin{theorem} \label{thm:PCal2_bound}
With $\Alg$ instantiated by the algorithm of \citet{fishelson2025full}, $N_m=\big \lceil (s_m)^{1/3} |I_m|^{2/3} \iota^{-1/3} \big \rceil$, and $K=\big \lceil (1+C)^{1/3} \iota^{-1} \big \rceil$, \pref{alg:framework_non_uniform} ensures that with probability at least $1-\delta$,
$    \PCaldist_2  \leq \order  \rbr{\iota^3   + C^{1/3}  \iota^2  }.$
Moreover, the computational complexity per round is $\otil \big((1+C)^{2/3}\big)$.
\end{theorem}
\begin{proof}
The computational complexity analysis follows that of \pref{thm:cal2_approach1}. The main difference is that the base algorithm requires computing a stationary distribution of $\calQ_t$. Constructing $\calQ_t$ costs $\order(|\Pi_m|^2)$. Computing its stationary distribution via power iteration requires $\tilde{O}(|\Pi_m|)$ per iteration because each row of $\calQ_t$ has at most two non-zero entries, and a logarithmic number of iterations suffices. Thus, the total cost is $\tilde{O}(|\Pi_m|^2)$. Applying a similar argument as in \pref{thm:cal2_approach1} yields the claimed result.

The following proof conditions on the nice event $\calE_2$ where $\calE_2$ is defined in \pref{def:nice_event_PCal2}.
Let $\bar{\rho}_{m,p} = \frac{ \sum_{t \in \calT_m} \calP_t(p) y_t}{\sum_{t \in \calT_m} \calP_t(p)}$.
\pref{corr:decomp_Cal2_into_block} shows that $\PCaldist_2 \leq \sum_{m=1}^M \PCaldist_2^{(m)} $, where $\PCaldist_2^{(m)}:= \sum_{t \in \calT_m} \E_{p \sim \calP_t}\big[ \rbr{\bar{\rho}_{m,p}-p}^2 \big]$. Then, it suffices to bound each $\PCaldist_2^{(m)}$. By \citep[Lemma 26]{fishelson2025full}, choosing squared loss $\loss(a,x) = (a-x)^2$ for $\FSR^{(m)}$ gives that $\PCaldist_2^{(m)}=\FSR^{(m)}$. For shorthand, let $\beta_{s,t}= \E_{a \sim H(a_{s,t})}[\loss(a,y_t)]-\loss(a_{s,t},y_t)$. For any epoch $m \geq 2$, we have
\begin{align*}
\FSR^{(m)} & = \sup_{\phi:  [0,1] \to  [0,1]} \sum_{t \in \calT_m} \E_{s \sim \calP_t} \sbr{ \rbr{\loss(s,y_t)  - \loss(\phi(s),y_t)} }\\
& = \sup_{\phi:  [0,1] \to  [0,1]} \sum_{t \in \calT_m} \sum_{s \in \calZ_m} \calP_t(s) \rbr{\loss(s,y_t)  - \loss(\phi(s),y_t)} \\
& =\sum_{s \in \calZ_m } \sup_{k \in [0,1]} \sum_{t \in \calT_m}  \calP_t(s)\rbr{\loss(s,y_t)  - \loss(k,y_t)}\\
& =\sum_{s \in \calZ_m } \sum_{t \in \calT_m}  \calP_t(s)\loss(s,y_t)  - \sum_{s \in \calZ_m }\inf_{k \in [0,1]} \sum_{t \in \calT_m}   \calP_t(s) \loss(k,y_t) \\
& =\sum_{s \in \calZ_m } \sum_{t \in \calT_m}  \sum_{s' \in \calZ_m } \rbr{\calP_t(s') \calQ_t(s',s) }\loss(s,y_t)  - \sum_{s \in \calZ_m }\inf_{k \in [0,1]} \sum_{t \in \calT_m}   \calP_t(s) \loss(k,y_t) \\
& =  \sum_{t \in \calT_m}  \sum_{s' \in \calZ_m }  \calP_t(s')   \E_{a \sim H(a_{s',t})}[\loss(a,y_t)]  - \sum_{s \in \calZ_m}\inf_{k \in [0,1]} \sum_{t \in \calT_m}   \calP_t(s) \loss(k,y_t) \\
&=  \sum_{t \in \calT_m}  \sum_{s' \in \calZ_m }  \calP_t(s')   \rbr{\loss(a_{s',t},y_t) + \beta_{s',t}   }  - \sum_{s \in \calZ_m }\inf_{k \in  [0,1]} \sum_{t \in \calT_m}   \calP_t(s) \loss(k,y_t)  \\
&=  \sum_{t \in \calT_m}  \sum_{s' \in \calZ_m }  \calP_t(s')   \beta_{s',t}       + \sum_{s \in \calZ_m  } \ExtReg^{(m)}_s \\
&\leq  \sum_{t \in \calT_m}  \sum_{s \in \calZ_m }  \calP_t(s)   \beta_{s,t}  +  \order \rbr{ |\calZ_m| \log T}  \\
&\leq  \sum_{t \in \calT_m}  \sum_{s \in \calZ_m }  \calP_t(s)   \beta_{s,t}  +  \order \rbr{ |\Pi_m| \log T}, \numberthis{} \label{eq:block_swap_reg_bound_step1}
\end{align*}
where the first inequality follows from the regret bound of OGD given by \citep[Lemma 24]{fishelson2025full}, and the last inequality uses the fact that $|\calZ_m|=\order (|\Pi_m|)$.

Recall that $J_{s,t} \in \Pi_{m}$ is the interval such that $a_{s,t} \in J_{s,t}$.
In the following, we focus on $m\geq 2$.
\begin{align*}
 & \sum_{t \in \calT_m}  \sum_{s \in \calZ_m }  \calP_t(s)   \beta_{s,t} \\
 & \leq   \sum_{t \in \calT_m}  \sum_{s \in \calZ_m }  \calP_t(s) \frac{\Delta^2_{J_{s,t}} }{4}  \\
& = \sum_{t \in \calT_m}  \sum_{s \in \calZ_m }  \calP_t(s)  \frac{\Delta_{J_{s,t}}^2  }{4} \Ind{J_{s,t}  \in \Pi_m^{\text{out}} } +\sum_{t \in \calT_m}  \sum_{s \in \calZ_m }  \calP_t(s)  \frac{\Delta_{J_{s,t}}^2  }{4}  \Ind{J_{s,t}  \in \Pi_m^{\text{in}} } \\
& = \sum_{t \in \calT_m}  \sum_{s \in \calZ_m }  \calP_t(s)  \frac{\Delta_{J_{s,t}}^2  }{4} \Ind{J_{s,t}  \in \Pi_m^{\text{out}} } +\frac{1 }{4}\sum_{t \in \calT_m}  \sum_{s \in \calZ_m }  \calP_t(s)  \frac{ |I_{m}|^2  }{ N_{m}^2 }  \Ind{J_{s,t}  \in \Pi_m^{\text{in}} } \\
& \leq \sum_{t \in \calT_m}  \sum_{s \in \calZ_m }  \calP_t(s)  \frac{\Delta_{J_{s,t}}^2  }{4} \Ind{J_{s,t}  \in \Pi_m^{\text{out}} } + \frac{ s_m |I_{m}|^2  }{ 4N_{m}^2 }, \numberthis{} \label{eq:sum_pt_betast_bound}
\end{align*}
where the first inequality uses \pref{lem:two_points_dist} and the second equality follows from the fact that for any interval $J \in \Pi_m^{\text{in}}$, $\Delta_J = \frac{|I_{m}|}{N_{m}}$ holds.

Recall that $C_m = \sum_{t \in \calT_m} c_t$.
Then, we show that
\begin{align*}
& \sum_{s \in \calZ_m }     \sum_{t \in \calT_m}  \calP_t(s)  \Delta_{J_{s,t}}^2 \Ind{J_{s,t}  \in \Pi_m^{\text{out}} } \\
&\leq \frac{1}{K^2}    \sum_{s \in \calZ_m }   \sum_{t \in \calT_m}   \calP_t(s) (a_{s,t}-\optp)^2 \\
&\leq \order \rbr{\frac{|\Pi_m|\log T +C_m +\iota }{K^2}  } \\
&\leq \order \rbr{|\Pi_m|\log T +\iota+\frac{C_m}{K^2}  }, \numberthis{} \label{eq:block_swap_reg_bound_step3}
\end{align*}
where the first inequality applies \pref{lem:bound4_DeltaJ_abs_gap_Cal2}, the second inequality follows from \pref{lem:bound_occupancy_PCal2}, and the third inequality simply bounds $K \geq 1$.

Putting \pref{eq:sum_pt_betast_bound} and \pref{eq:block_swap_reg_bound_step3} into \pref{eq:block_swap_reg_bound_step1}, we have that
\begin{align*}
\PCaldist_2^{(m)} &= \FSR^{(m)} \\
&\leq \order \rbr{ \iota  + |\Pi_m| \iota       + \frac{ s_m |I_{m}|^2  }{ 4N_{m}^2 } +  \frac{C_m }{K^2} } \numberthis{}  \label{eq:absorbed_Pim_iota}\\
&= \order \rbr{ \iota +\frac{ s_m |I_{m}|^2  }{ 4N_{m}^2 }+ N_{m}  \iota   + |\Pi_m^{\text{out}}| \iota   +  \frac{C_m }{K^2}     } \\
&= \order \rbr{ \iota + \frac{ s_m |I_{m}|^2  }{ 4N_{m}^2 }  + N_{m}  \iota     + K \iota^2 + \frac{C_m }{K^2}    } \\
&\leq \order \rbr{\iota^2 +\frac{ s_m |I_{m}|^2  }{ 4N_{m}^2 }  + N_{m} \iota      +  (1+C)^{1/3} \iota   + \frac{C_m \iota^2 }{(1+C)^{2/3}}    },
\end{align*}
where the first equality follows from \citep[Lemma 26]{fishelson2025full}, the second equality uses $|\Pi_m|=|\Pi_m^{\text{in}}| +|\Pi_m^{\text{out}}|$ together with $N_{m} =|\Pi_m^{\text{in}}|$, the third equality follows from $|\Pi_m^{\text{out}}|= \order (K \log T) \leq  \order (K \iota)$, and the last inequality uses the choice of $K$.

On the one hand, applying \pref{lem:inner_error_Cal2_approach1}, we can show that
\[
\sum_{m=2}^M   \rbr{\frac{|I_{m}|^2}{N_{m}^2} s_m+N_{m} \iota  }  \leq \order \rbr{   \iota^{2/3} C^{1/3}   + \iota \log  T }.
\]

On the other hand, using $M=\order(\log T)$, summing over all epochs $m$, and bounding $\PCaldist_2^{(1)} = \order(1)$, we have
\begin{align*}
 \PCaldist_2 & \leq \PCaldist_2^{(1)} + \sum_{m=2}^M \PCaldist_2^{(m)}\\
 &\leq \PCaldist_2^{(1)} + \sum_{m=2}^M \order \rbr{   \iota^2  + (1+C)^{1/3}  \iota  + \frac{C_m \iota^2 }{(1+C)^{2/3}}   }   +\order\rbr{  \iota^{2/3} C^{1/3}   + \iota \log  T}\\
& \leq \order  \rbr{\iota^2   \log T+ (1+C)^{1/3}  \iota^2 + \iota^{2/3} C^{1/3}  }\\
& \leq \order  \rbr{\iota^3   + C^{1/3}  \iota^2  }.
\end{align*}

The proof is thus complete.
\end{proof}

\begin{corollary} \label{corr:Cal2_bound_implied_by_PCal2}
Under the same setting of \pref{thm:PCal2_bound}, with probability at least $1-\delta$,
\begin{align*}
\Caldist_2  \leq \order  \rbr{\iota^3   + C^{1/3}  \iota^2  }.
\end{align*}
\end{corollary}
\begin{proof}
  By \citep[Theorem 3]{luo2025simultaneous}, for an epoch $m$, with probability at least $1-\delta/(2\log_2 T)$
    \[
    \Caldist_2^{(m)}  \leq  \order  \rbr{ \PCaldist_2^{(m)} +  |\calZ_m|\log(|\calZ_m|\log (T)/\delta) }  \leq  \order  \rbr{ \PCaldist_2^{(m)} +  |\Pi_m|\iota },
    \]
    where the last inequality uses the facts that $|\calZ_m|=\order(|\Pi_m|)$ and $|\calZ_m|=\order(T)$.

Then, the claimed bound is immediate since we can put the extra $|\Pi_m|\iota$ term into \pref{eq:absorbed_Pim_iota}, repeat the same analysis together with a union bound over all epochs.
\end{proof}

\section{Omitted Details for Achieving $\PKL=\otil((1+C)^{1/3})$}   \label{app:PKL}

\subsection{Proposed Algorithm for Achieving $\PKL=\otil((1+C)^{1/3})$}

To achieve $\PKL=\otil((1+C)^{1/3})$, we run \pref{alg:framework_non_uniform_PKLCal} with the input confidence $\delta \in (0,1)$, non-stationarity $C$, base algorithm instantiated by Algorithm 1 of \citep{luo2025simultaneous}, $\eta=1/(T+1)$, and $K$ is specified as:
\begin{equation} \label{eq:Kmq_PKL}
    K =  \max \cbr{ \wt{K} , 2 } \quad \text{where}\quad \wt{K} =  \left \lceil \rbr{\frac{1+C}{\log^2 T}}^{1/3}   \right \rceil.
\end{equation}

A similar modification for \pref{alg:framework_non_uniform_PKLCal} is made for $N_m$, defined as:
\begin{equation} \label{eq:Nm_PKL}
    N_m = \max \cbr{\wt{N}_m , \left \lceil  \frac{2|I_m|}{\pi}   \right \rceil },\quad \text{where}\quad \wt{N}_m =\left \lceil \rbr{s_m}^{1/3} |I_{m}|^{2/3} \iota^{-1/3} \right \rceil.
\end{equation}

We again use $\Pi_m$ to denote the set of intervals, partitioned by \pref{alg:framework_non_uniform_PKLCal}.
We make these modifications for technical reasons to ensure $\Delta_J \leq \pi/2$ for all $J \in \Pi_m$. Moreover, taking maximum in \pref{eq:Kmq_PKL} and \pref{eq:Nm_PKL} does not hurt our analysis since $K \in \big[\wt{K},2\wt{K} \big]$ and $N_{m} \in \big[\wt{N}_m,2\wt{N}_m \big]$. We again use $Q=\lceil \log_2 T \rceil$.

Now, we review Algorithm 1 of \citep{luo2025simultaneous} in the context of our non-uniform partition. Let $\psi(z)=\sin^2(z/2)$. Given partition $\Pi_m$, the algorithm selects a set of discrete points, denoted by
\[
\calZ_m =  \cbr{  \psi(z) :z=\sup_{x \in J} x, J \in \Pi_m } \cup \{\eta,1-\eta\}.
\]

Each point $s \in \calZ_m$ is associated with an instance of Exponentially Weighted Online Optimization (EWOO) \citep{hazan2007logarithmic}. At each round $t$, each EWOO instance $s \in \calZ_m$ outputs an action $a_{s,t} \in [0,1]$. For each $s$, let $J_{s,t} \in \Pi_m$ be the interval such that $\theta(a_{s,t}) \in J_{s,t}$, where for any $y \in [0,1]$, $\theta(y) = 2\arcsin(\sqrt{y})$. The algorithm then defines a distribution $H(a_{s,t}) \in \Delta(\calZ_m)$ supported on the two endpoints of $J_{s,t}$, with probabilities proportional to their weighted distances to $a_{s,t}$. Specifically, suppose that $\theta(a_{s,t}) \in J$ and $J=[u,v]$ (or $[u,v)$). Let $d=\psi(u),b=\psi(v)$. Then for all $s \in \calZ_m \backslash \{d,b\}$, $H(a_{s,t})(s)=0$, and
\begin{align*}
H(a_{s,t})(d) =  \frac{ \frac{b-a_{s,t}}{b(1-b)}  }{ \frac{b-a_{s,t}}{b(1-b)} +\frac{a_{s,t}-d}{d(1-d)}    } , \quad
H(a_{s,t})(b) =  \frac{ \frac{a_{s,t}-d}{d(1-d)}  }{ \frac{b-a_{s,t}}{b(1-b)} +\frac{a_{s,t}-d}{d(1-d)}    } .
\end{align*}

All other points receive zero mass. Using these distributions, the algorithm constructs a row-stochastic matrix $\calQ_t \in \mathbb{R}^{|\calZ_m|\times|\calZ_m|}$, where $\calQ_t(s,\cdot)=H(a_{s,t})$ and computes a stationary distribution $\calP_t$ induced by $\calQ_t$.
Finally, the algorithm samples a prediction $p_t \sim \calP_t$. Each EWOO instance $s \in \calZ_m$ is then updated using the weighted loss $\calP_t(s)\loss(a_{s,t},y_t)$ where $\loss$ is the log loss $\loss(p,y) = -y\log p -(1-y)\log(1-p)$.

\setcounter{AlgoLine}{0}
\begin{algorithm}[t]
\DontPrintSemicolon
\caption{Modified framework for $\PKL$ to adapt non-stationarity}
\label{alg:framework_non_uniform_PKLCal}
\textbf{Input}: confidence $\delta \in (0,1)$, non-stationarity $C$, base algorithm $\Alg$, number of grids $\{N_m\}_m$ for inner regions, the number of grids $K$ for outer bands, and $\eta \in (0,1/2]$.

Let $s_m = 2^m$. Predict arbitrarily for the first epoch that lasts for $s_1$ rounds.

\For{epoch $m=2,\ldots$}{

Let $\wh{y}_m$ be the average outcome of the previous epoch.

Define inner region $I_m = [a_m,b_m] =  \sbr{ \theta(\wh{y}_m) - 2r_m, \theta(\wh{y}_m) + 2r_m  } \cap [\theta(\eta),\theta(1-\eta)]$ where $r_m=2\pi \sqrt{\frac{ (\iota+C) }{s_{m-1}}}$.

Define inner partition $\Pi_{m+1}^{\text{in}} \gets \UnifPart(I_m,N_m) $ where $\UnifPart$ is from \pref{def:unifpart}.

Define left and right outer bands as (for each $q=0,\ldots,Q$ where $Q=\lceil \log_2 T \rceil$)
\begin{equation} \label{eq:outer_division}
\begin{aligned}
        L_{m,q} &= [a_m -(2^{q+1}-1) r_m , a_m -(2^{q}-1) r_m ) \cap [\theta(\eta),\theta(1-\eta)],\\
        R_{m,q} &=  (  b_m +( 2^{q} -1)r_m , b_m + (2^{q+1}-1) r_m ] \cap [\theta(\eta),\theta(1-\eta)] .
\end{aligned}
\end{equation}

Define $\calL_{m,q} = \UnifPart(L_{m,q},K),\calR_{m,q} = \UnifPart(R_{m,q},K)$ for $q=0,\ldots,Q$.

Update outer partition $\Pi_{m}^{\text{out}} =  \{\calL_{m,q}\}_{q=0}^{Q} \cup \{\calR_{m,q}\}_{q=0}^{Q}$

Define final overall partition$\Pi_{m} = \Pi_{m}^{\text{in}} \cup \Pi_{m}^{\text{out}}$.

Run $\Alg$ over the partition $\Pi_m$ from scratch for $s_m=2^m$ rounds.

}
\end{algorithm}

\subsection{Supporting Lemmas and Nice Event Construction}

\begin{lemma} \label{lem:range_ast}
For all $s,t$, $a_{s,t} \in \sbr{\frac{1}{T+1},1-\frac{1}{T+1}}$.
\end{lemma}
\begin{proof}
For shorthand, we define $T_m^{<t}:=\{\tau\in T_m:\tau<t\}$.
According to \citep[Appendix C.1]{luo2025simultaneous}, $a_{s,t}$ has a closed-form solution
\[
a_{s,t} = \frac{ \sum_{\tau \in  T_m^{<t}} \calP_{\tau}(s) y_{\tau} +1 }{ \sum_{\tau \in  T_m^{<t}} \calP_{\tau}(s) +2}.
\]

Then, one can simply show that
\[
a_{s,t} \geq \frac{1 }{ \sum_{\tau \in  T_m^{<t}} \calP_{\tau}(s) +2} \geq \frac{ 1 }{|T_m^{<t}|+2} \geq \frac{1}{T+1}.
\]

On the other hand, we similarly have
\[
1-a_{s,t} \geq \frac{1 }{ \sum_{\tau \in  T_m^{<t}} \calP_{\tau}(s) +2} \geq \frac{ 1 }{|T_m^{<t}|+2} \geq \frac{1}{T+1}.
\]

The claimed result thus follows.
\end{proof}

\begin{lemma} \label{lem:concentration_optp_PKL}
With probability at least $1-\delta/2$, for any $m \geq 2$,
\begin{equation} \label{eq:def_rm_PKL}
     \abr{ \theta(\wh{y}_m) - \theta(\optp) } \leq r_m = 2\pi \sqrt{\frac{ (\iota+C) }{s_{m-1}}}.
\end{equation}
\end{lemma}
\begin{proof}
Let $\bar{q}_m = \frac{1}{s_{m-1}} \sum_{t \in \calT_{m-1}} q_t$.
We fix an epoch $m \geq 2$ and use \pref{lem:theta_diff_basic_inequality} to write that
\begin{align*}
    \abr{ \theta(\wh{y}_m) - \theta(\optp) }^2 & \leq 2\abr{ \theta(\wh{y}_m) - \theta(\bar{q}_m) }^2+2\abr{ \theta(\bar{q}_m) - \theta(\optp) }^2\\
    & \leq 2\pi^2 \KL(\wh{y}_m,\bar{q}_m) +  4\pi^2 \abr{  \bar{q}_m -\optp} . \numberthis{} \label{eq:concentration_PKL_pre}
\end{align*}

On the one hand, we have
\begin{align} \label{eq:concentration_PKL_corruption}
   \abr{  \bar{q}_m -\optp}  \leq \frac{1}{s_{m-1}} \sum_{t \in \calT_{m-1}} \abr{q_t - \optp }    = \frac{1}{s_{m-1}} \sum_{t \in \calT_{m-1}} c_t \leq \frac{C}{s_{m-1}}.
\end{align}

On the other hand, as $y_t \sim \texttt{Ber}(q_t)$, for any $\lambda>0$, we have
\begin{align*}
    \E \sbr{  \exp \rbr{ \lambda \sum_{t \in \calT_{m-1}}  y_t  } } = \prod_{t \in \calT_{m-1}} (1-q_t +q_te^{\lambda}) \leq \rbr{ 1 -  \bar{q}_m +  \bar{q}_m e^{\lambda} }^{s_{m-1}},
\end{align*}
where the inequality follows from the fact that Jensen's inequality gives $\frac{1}{s_{m-1}} \sum_{t \in \calT_{m-1}} \log(1-q_t +q_te^{\lambda}) \leq  \log( 1 -  \bar{q}_m +  \bar{q}_m e^{\lambda} )$.

Thus, for any $u > \bar{q}_m$, Chernoff's bound gives
\begin{align*}
\P \rbr{ \wh{y}_m \geq u } \leq \inf_{ \lambda >0 } \exp \rbr{ -\lambda s_{m-1} u  +s_{m-1} \log(1 -  \bar{q}_m +  \bar{q}_m e^{\lambda}) } = \exp \rbr{  -s_{m-1} \KL(u, \bar{q}_m)}.
\end{align*}

Similarly, for any $u \leq \bar{q}_m$, $\P \rbr{ \wh{y}_m \leq u } \leq  \exp \rbr{  -s_{m-1} \KL(u, \bar{q}_m)}$. Thus, with probability at least $1-\delta'$,
\begin{equation} \label{eq:concentration_PKL_KL_term}
  \KL(\wh{y}_m,\bar{q}_m)  \leq       \frac{2\log(2/\delta
  ')}{s_{m-1}}.
\end{equation}

Putting \pref{eq:concentration_PKL_corruption} and \pref{eq:concentration_PKL_KL_term} into \pref{eq:concentration_PKL_pre}, we have
\begin{align*}
     \abr{ \theta(\wh{y}_m) - \theta(\optp) }^2 \leq \frac{4\pi^2 \log(1/\delta')}{s_{m-1}} +\frac{4\pi^2 C}{s_{m-1}}.
\end{align*}

Choosing $\delta'$ properly, applying a union bound over all epochs together with the fact that the number of epochs is at most $\order(\log T)$, we complete the proof.
\end{proof}

\begin{lemma} \label{lem:bound_occupancy_PKL}
Suppose that $\eta \in (0,1/2]$ and $a_{s,t} \in [\eta,1-\eta]$ for all $s,t$. With probability at least $1-\delta/2$, for each $m \geq 2$,
\begin{align*}
\sum_{t \in \calT_m}   \sum_{s \in \calZ_m}  \calP_t(s) \KL(\optp,a_{s,t}) \leq \order \rbr{|\Pi_m|\log T +\log(1/\eta)C_m +\iota \log(1/\eta) +\eta s_m},
\end{align*}
where $C_m = \sum_{t \in \calT_m} c_t$.
\end{lemma}

\begin{proof}
For shorthand, we denote $\optq =\Clip_{[\eta,1-\eta]}(\optp)$,
\[
V_t := \sum_{s \in \calZ_m}  \calP_t(s) \log \rbr{ \frac{a_{s,t}(1-\optq)}{\optq(1-a_{s,t})}  } ,\quad \Gamma_m =  \sum_{t \in \calT_m}   \sum_{s \in \calZ_m}  \calP_t(s) \KL(\optp,a_{s,t}) .
\]

By direct expansion, we have for any $a_{s,t} \in (0,1)$ and any $x \in \{0,1\}$,
\begin{align*}
\loss(a_{s,t},x)-\loss(\optq ,x) =\KL(\optp,a_{s,t})-\KL(\optp,\optq ) -(x-\optp)  \log \rbr{ \frac{a_{s,t}(1-\optq)}{\optq (1-a_{s,t})} } .
\end{align*}

For the above equality, we set $x=y_t$ and multiply both sides by $\calP_t(s)$ for each $t$, then sum over $s \in \calZ_m ,t \in \calT_m$, and finally rearrange, we have
\begin{align*}
   \Gamma_m &= \underbrace{s_m \KL(\optp,\optq )}_{\Term_1} + \underbrace{ \sum_{s \in \calZ_m}  \sum_{t \in \calT_m} \calP_t(s) \rbr{ \loss(a_{s,t},y_t) -\loss(\optq,y_t) } }_{\Term_2}\\
   &\quad +\underbrace{ \sum_{t \in \calT_m} (y_t-q_t)V_t }_{\Term_3}+ \underbrace{ \sum_{t \in \calT_m} (q_t-\optp)V_t }_{\Term_4} .
\end{align*}

\textbf{Bounding $\Term_1$.} If $\optp \in [\eta,1-\eta]$, then $\optp=\optq$, and thus $\KL(\optp,\optq)=0$. If $\optp <\eta$, then $\optq=\eta$, and
\[
\KL (\optp,\optq) =\KL (\optp,\eta) \leq  \KL (0,\eta) =\log \rbr{ \frac{1}{1-\eta}  }.
\]

If $\optp >1-\eta$, then $\optq=1-\eta$, and similarly
\[
\KL (\optp,\optq) =\KL (\optp,1-\eta) \leq  \KL (1,1-\eta) =\log \rbr{ \frac{1}{1-\eta}  }.
\]

Combining two cases, we have $\KL (\optp,\optq) \leq -\log(1-\eta)$.
For any $\eta \in (0,1/2]$, we have $-\log(1-\eta) \leq 2\eta$, which gives $\KL (\optp,\optq) \leq 2\eta$. Therefore,
\begin{equation} \label{eq:bound_term1_logloss_lemma}
    \Term_1 \leq 2 \eta s_m .
\end{equation}

\textbf{Bounding $\Term_2$.} One can show that
\begin{equation}\label{eq:bound_term2_logloss_lemma}
    \Term_2  = \sum_{s \in \calZ_m}  \sum_{t \in \calT_m} \calP_t(s) \rbr{ \loss(a_{s,t},y_t) -\loss(\optq,y_t) } \leq    \sum_{s \in \calZ_m}  \ExtReg^{(m)}_s \leq \order \rbr{|\Pi_m|\log T},
\end{equation}
where the last inequality uses the external regret guarantee of EWOO together with $|\calZ_m|= \order(|\Pi_m|)$.

\textbf{Bounding $\Term_4$.} We have
\begin{align} \label{eq:abs_Vt_logloss}
    |V_t| = \abr{ \sum_{s \in \calZ_m}  \calP_t(s) \log \rbr{ \frac{a_{s,t}(1-\optq)}{\optq(1-a_{s,t})}  } } \leq  \sum_{s \in \calZ_m}  \calP_t(s)  \abr{\log \rbr{ \frac{a_{s,t}(1-\optq)}{\optq(1-a_{s,t})}  } }  \leq 2 \log(1/\eta),
\end{align}
where the last inequality uses the fact that for any $\optq,a_{s,t} \in [\eta,1-\eta]$, $\abr{\log \rbr{ \frac{a_{s,t}(1-\optq)}{\optq(1-a_{s,t})}  }} \leq 2\log(1/\eta)$.
Then, we can show that
\begin{align} \label{eq:bound_term4_logloss_lemma}
\Term_4   \leq  \abr{ \sum_{t \in \calT_m} V_t (q_t-\optp) } \leq \sum_{t \in \calT_m} c_t |V_t| \leq 2 \log(1/\eta)C_m.
\end{align}

\textbf{Bounding $\Term_3$.}
By the Cauchy--Schwarz inequality, we have
\begin{align} \label{eq:vtsquare_bound_PKLCal}
   V_t^2=
   \rbr{  \sum_{s \in \calZ_m}  \sqrt{\calP_t(s)} \sqrt{\calP_t(s)} \log \rbr{ \frac{a_{s,t}(1-\optq)}{\optq(1-a_{s,t})}  }    }^2 \leq      \sum_{s \in \calZ_m } \calP_t(s) \log^2 \rbr{ \frac{a_{s,t}(1-\optq)}{\optq(1-a_{s,t})}  } .
\end{align}

Let $Y_t = V_t  (y_t-q_t) $. We have $\E_t[Y_t] = V_t\E_t [y_t-q_t] =0$, and use \pref{eq:vtsquare_bound_PKLCal} to show
\begin{align*}
\E_t \sbr{ Y_t^2 }   & = V_t^2   \E_t [(y_t-q_t)^2] \leq  V_t^2 q_t \rbr{1-q_t} \\
&\leq     \sum_{s \in \calZ_m } \calP_t(s) q_t \rbr{1-q_t} \log^2 \rbr{ \frac{a_{s,t}(1-\optq)}{\optq(1-a_{s,t})}  } \\
&\leq     \sum_{s \in \calZ_m } \calP_t(s) \rbr{ \optp \rbr{1- \optp } +c_t} \log^2 \rbr{ \frac{a_{s,t}(1-\optq)}{\optq(1-a_{s,t})}  }  \\
&\leq     \sum_{s \in \calZ_m } \calP_t(s)\optp \rbr{1- \optp } \log^2 \rbr{ \frac{a_{s,t}(1-\optq)}{\optq(1-a_{s,t})}  } + 4c_t \log^2(1/\eta), \numberthis{} \label{eq:second_order_Yt_squared}
\end{align*}
where the second inequality follows from the following:
\[
q_t(1-q_t) = \optp \rbr{1- \optp } +  (q_t -\optp)(1-q_t-\optp) \leq \optp \rbr{1- \optp } +|q_t -\optp| \leq \optp \rbr{1- \optp } +c_t.
\]

We can further rewrite the first term in \pref{eq:second_order_Yt_squared} as
\begin{align*}
&    \sum_{s \in \calZ_m } \calP_t(s) \optp \rbr{1-\optp} \log^2 \rbr{ \frac{a_{s,t}(1-\optq)}{\optq(1-a_{s,t})}  } \\
& =    \sum_{s \in \calZ_m } \calP_t(s) \Variance_{Y \sim \text{Ber}(\optp)} \rbr{ \loss(a_{s,t},Y) -\loss(\optq,Y) } \\
& =    \sum_{s \in \calZ_m } \calP_t(s) \Variance_{Y \sim \text{Ber}(\optp)} \rbr{ \loss(a_{s,t},Y) -\loss(\optp,Y)+\loss(\optp,Y) -\loss(\optq,Y) } \\
& \leq    2\sum_{s \in \calZ_m } \calP_t(s)  \rbr{\E_{Y \sim \text{Ber}(\optp)} \sbr{ \rbr{\loss(a_{s,t},Y) -\loss(\optp,Y)}^2}  +\E_{Y \sim \text{Ber}(\optp)} \sbr{ \rbr{ \loss(\optp,Y) -\loss(\optq,Y)}^2 } } \\
& \leq    4 \rbr{\log(1/\eta)+1} \sum_{s \in \calZ_m } \calP_t(s)  \rbr{ \KL(\optp,a_{s,t})  +\KL(\optp,\optq)} ,
\end{align*}
where the last inequality uses \pref{lem:bound_diff_log_loss_square}.

Summing over all $t \in \calT_m$, we have
\begin{align*}
    \sum_{t \in \calT_m}  \E_t \sbr{ Y_t^2 } & \leq   4 \rbr{\log(1/\eta)+1}   \rbr{s_m \KL(\optp,\optq)+  \Gamma_m   }  +4C_m \log^2(1/\eta) \\
    &\leq   4 \rbr{\log(1/\eta)+1}   \rbr{2\eta s_m +  \Gamma_m   }+4C_m \log^2(1/\eta).
\end{align*}

From \pref{eq:abs_Vt_logloss}, we have $|Y_t|=|(q_t-y_t)V_t| \leq |V_t| \leq 2\log(1/\eta)$. By Freedman's inequality (\pref{lem:freedman} with $\lambda=\Theta(1/\log(1/\eta))$) and a union bound over all $m$, with probability at least $1-\delta/2$
\begin{align} \label{eq:bound_term3_logloss_lemma}
 \Term_3 \leq    \abr{ \sum_{t \in \calT_m} Y_t }  \leq  \frac{1}{2}    \Gamma_m+ \order \rbr{ \iota \log(1/\eta) +C_m \log(1/\eta) +\eta s_m}.
\end{align}

\textbf{Putting together.}
Combining \pref{eq:bound_term1_logloss_lemma}, \pref{eq:bound_term2_logloss_lemma}, \pref{eq:bound_term3_logloss_lemma}, and \pref{eq:bound_term4_logloss_lemma}, we have
\begin{align*}
\Gamma_m \leq \order \rbr{|\Pi_m|\log T +\log(1/\eta)C_m +\iota \log(1/\eta) +\eta s_m}  + \frac{1}{2} \Gamma_m  .
\end{align*}

Rearranging the above gives the claimed bound.
\end{proof}

\begin{definition}[Nice event $\calE_3$] \label{def:nice_event_PKL}
Let $\calE_3$ be the event that all high probability bounds in \pref{lem:concentration_optp_PKL} and \pref{lem:bound_occupancy_PKL} hold simultaneously.
\end{definition}

\begin{lemma} \label{lem:inner_error_PKLCal}
The following holds.
\begin{align*}
\sum_{m=1}^M \rbr{\wt{N}_{m} \log T + s_{m+1} \frac{|I_{m}|^2}{\wt{N}_{m}^2}} \leq  \order \rbr{  \log^2 T (\iota+C)^{1/3}    }.
\end{align*}
\end{lemma}
\begin{proof}
We can show that
\begin{align*}
& \wt{N}_{m} \log T  + s_{m+1} \frac{|I_{m}|^2}{\wt{N}_{m}^2}  \\
&=  \order \rbr{\wt{N}_{m}\log T  + s_{m} \frac{|I_{m}|^2}{\wt{N}_{m}^2} } \tag{$s_{m+1}=2s_m$} \\
& =  \order \rbr{ \log T \rbr{1+ s_m^{1/3}   |I_{m}|^{2/3} }  } \tag{Choice of $\wt{N}_{m}$} \\
& \leq  \order \rbr{ \log T \rbr{1+ s_m^{1/3}  \min \cbr{1,\sqrt{\frac{\iota+C }{ s_m} } }^{2/3} }  }  \\
& =  \order \rbr{ \log T \rbr{1+ \min \cbr{ s_m^{1/3}, (\iota+C)^{1/3} }    }  }  \\
& \leq   \order \rbr{ \log T (\iota+C)^{1/3}  },
\end{align*}
where the first inequality follows from the definition of $I_m$ that for any epoch $m$
\begin{equation} \label{eq:bound4Im_PKL}
    |I_m| \leq 4r_m  \leq \order \rbr{ \min \cbr{1,\sqrt{\frac{\iota +C}{ s_{m-1}} } }  } \leq \order \rbr{ \min \cbr{1,\sqrt{\frac{\iota +C}{ s_{m}} } }  }.
\end{equation}

Then, we can write
\begin{align*}
\sum_{m=1}^M \rbr{\wt{N}_{m} \log T  + s_{m+1} \frac{|I_{m}|^2}{\wt{N}_{m}^2} }\leq  \sum_{m=1}^M  \order \rbr{ \log T (\iota+C)^{1/3}  } \leq  \order \rbr{ \log^2 T (\iota+C)^{1/3}  }.
\end{align*}

\end{proof}

\begin{lemma} \label{lem:bound4_DeltaJ_abs_gap}
Suppose that $\calE_3$ holds where $\calE_3$ is defined in \pref{def:nice_event_PKL}.
For any $m \geq 2$, any $J \in \Pi_{m+1}^{out}$, and any $\theta(a) \in J$, we have $\Delta_J \leq \frac{|\theta(a)-\theta(\optp)|}{\wt{K}}$.
\end{lemma}
\begin{proof}
Fix an epoch $m\geq 2$ and an index $q$.
Recall that $\Pi_{m}^{(\text{out})}= \{\calL_{m,q}\}_{q=0}^{Q} \cup \{\calR_{m,q}\}_{q=0}^{Q}$.
We first show that for any $J \in \calL_{m,q}$ and any $\theta(a) \in J$, we have $\Delta_J \leq \frac{|\theta(a)-\theta(\optp)|}{\wt{K}}$.
As $\theta(\optp) \in \sbr{ \theta(\wh{y}_m) - r_m,\theta(\wh{y}_m) +r_m  } $ by \pref{lem:concentration_optp_PKL}, we have $\theta(\optp) \geq \theta(\wh{y}_m) - r_m$. Since $J \in \calL_{m,q}$, $\theta(a) \in J$, and $L_{m,q} = [a_m -(2^{q+1} -1)r_m , a_m -(2^{q}-1) r_m ] \cap [\theta(\eta),\theta(1-\eta)]$, where $a_m = \max \cbr{\theta(\eta),  \theta(\wh{y}_m) - 2r_m}$, we have
\[
\theta(a) \leq a_m - (2^{q}-1) r_m .
\]

Consider two cases. If $a_m = \theta(\eta)$, then interval $J $ is empty, and the claimed result holds trivially. It remains to consider $a_m = \theta(\wh{y}_m) - 2r_m$.
We have
\begin{align*}
\theta(\optp) - \theta(a) \geq    \theta( \wh{y}_m) - r_m  - \rbr{\theta( \wh{y}_m) - 2r_m -(2^{q}-1) r_m }  \geq  2^q r_m \geq |L_{m,q}|.
\end{align*}

Using $K \geq \wt{K}$, we can show that
\begin{align*}
\Delta_J \leq \frac{ |L_{m,q}|}{K} \leq \frac{ \theta(\optp) - \theta(a) }{K} = \frac{ |\theta(\optp) - \theta(a) |}{K} \leq \frac{ |\theta(\optp) - \theta(a) |}{\wt{K}}.
\end{align*}

One can symmetrically repeat this argument to show that for any $J \in \calR_{m,q}$ and any $\theta(a) \in J$, $\Delta_J \leq \frac{|\theta(\optp) - \theta(a)|}{\wt{K}}$ holds. Finally, repeating this argument for all $q,m$ completes the proof.
\end{proof}

\begin{lemma} \label{lem:two_point_dist_theta}
Let $\loss(\cdot,\cdot)$ be the log loss.
For any interval $J \in \Pi_m$, any point $q \in [0,1]$ such that $\theta(q) \in J$, and any $x \in \{0,1\}$, if $\Delta_J \leq \pi/2$, then $\E_{s \sim H(q)}[\loss(s,x)]-\loss(q,x) \leq \Delta_J^2$.
\end{lemma}

\begin{proof}
Suppose that $J=[u,v] \subseteq [\theta(\eta),\theta(1-\eta)]$.
Let $a=\theta^{-1}(u)$ and $b=\theta^{-1}(v)$.
Then,
\begin{align*}
H(q)(a) =  \frac{ \frac{b-q}{b(1-b)}  }{ \frac{b-q}{b(1-b)} +\frac{q-a}{a(1-a)}    } \quad
H(q)(b) =  \frac{ \frac{q-a}{a(1-a)}  }{ \frac{b-q}{b(1-b)} +\frac{q-a}{a(1-a)}    } .
\end{align*}

For shorthand, we define
\[
F(q) =q+ab-q(a+b).
\]

Then, we consider two cases for $x =1$ and $x=0$.

\textbf{Case 1: $x=1$.} In this case, $\loss(s,1)=-\log(s)$, which gives
\[
\E_{s \sim H(q)}[\loss(s,1)]-\loss(q,1)=\E_{s \sim H(q)} [\log(q/s)] \leq q \E_{s \sim H(q)} \sbr{ \frac{1}{s} } -1 =  \frac{(q-a)(b-q)}{F(q)} ,
\]
where the inequality uses $\log(x) \leq x-1$, and the last equality follows from $\E_{s \sim H(q)} \sbr{ \frac{1}{s} }=\frac{1-q}{F(q)}$.

\textbf{Case 2: $x=0$.} In this case, $\loss(s,0)=-\log(1-s)$, which gives
\begin{align*}
\E_{s \sim H(q)}[\loss(s,0)]-\loss(q,0) &=\E_{s \sim H(q)} \sbr{\log \rbr{ \frac{1-q}{1-s}  }} \\
&\leq (1-q) \E_{s \sim H(q)} \sbr{ \frac{1}{1-s} } -1 =  \frac{(q-a)(b-q)}{F(q)} ,
\end{align*}
where the inequality uses $\log(x) \leq x-1$, and the last equality follows from $\E_{s \sim H(q)} \sbr{ \frac{1}{1-s} }=\frac{q}{F(q)}$.

Combining both cases, we have that for any $y \in \{0,1\}$,
\[
\E_{s \sim H(q)}[\loss(s,y)]-\loss(q,y) \leq   \frac{(q-a)(b-q)}{F(q)} =:G(q)  .
\]

Differentiating function $G(q)$ with respect to $q$ and setting to zero gives the unique maximizer
\[
q^* =  \frac{ \sqrt{ab}   }{ \sqrt{ab} +\sqrt{(1-a)(1-b)}  }.
\]

Using the facts that $\sqrt{a/(1-a)}=\tan (\theta(a)/2)$ and $\sqrt{b/(1-b)}=\tan (\theta(b)/2)$, we have
\begin{align*}
\E_{s \sim H(q)}[\loss(s,y)]-\loss(q,y)  \leq G(q^*) &=   \rbr{    \frac{ \sqrt{ \frac{b}{1-b}}  - \sqrt{\frac{a}{1-a} }   }{ 1+ \sqrt{ \frac{a}{1-a} \frac{b}{1-b}  }   }   }^2  =  \tan^2 \rbr{  \frac{\theta(b)-\theta(a)}{2} }.
\end{align*}
where the second equality follows from the tangent subtraction formula.

Finally, using the fact that for any $x \in [0,\pi/4]$, $\tan x \leq 2x$ and the definitions of $a,b$, we show that
\[
\E_{s \sim H(q)}[\loss(s,y)]-\loss(q,y)  \leq \tan^2 \rbr{  \frac{\Delta_J}{2} } \leq \Delta_J^2,
\]
which thus completes the proof.
\end{proof}

\begin{lemma} \label{lem:theta_diff_basic_inequality}
Let $\theta(p)=2\arcsin(\sqrt{p})$ for any $p \in [0,1]$.
For any $p,q \in [0,1]$,
\begin{align*}
& \abr{  \theta(p) -\theta(q)}^2  \leq  \min \cbr{ 2\pi^2  \abr{ p-q },\pi^2  \KL(p,q),\pi^2  \KL(q,p) }.
\end{align*}
\end{lemma}

\begin{proof}
Let $H^2(p,q)$ be (squared) Hellinger distance for two Bernoulli distributions with mean $p,q$, defined as
\[
H^2(p,q):=(\sqrt p-\sqrt q)^2+(\sqrt{1-p}-\sqrt{1-q})^2.
\]

For shorthand, we define
\[
u=\frac{\theta(p)}{2}=\arcsin(\sqrt p),
\qquad
v=\frac{\theta(q)}{2}=\arcsin(\sqrt q).
\]

With these definitions, we have $\sqrt p=\sin u, \sqrt{1-p}=\cos u$, and $\sqrt q=\sin v, \sqrt{1-q}=\cos v.$
Then, we write
\[
H^2(p,q)
=(\sin u-\sin v)^2+(\cos u-\cos v)^2
=2-2(\sin u\sin v+\cos u\cos v)
=2-2\cos(u-v).
\]

Therefore, we have
\[
H^2(p,q)=2-2\cos(u-v)=4\sin^2 \rbr{\frac{u-v}{2}}
=4\sin^2 \rbr{\frac{\theta(p)-\theta(q)}{4}}.
\]

Since $\theta([0,1])=[0,\pi]$, we have $\left|\frac{\theta(p)-\theta(q)}{4}\right|\leq \frac{\pi}{4}.$
Using the fact that $\sin x \ge \frac{2}{\pi}x$,
for all $x\in [0,\frac{\pi}{2}]$, we obtain
\begin{align} \label{eq:abs_diff_theta_hellinger}
    H^2(p,q)
=
4\sin^2\!\left(\frac{|\theta(p)-\theta(q)|}{4}\right)
\ge
4\left(\frac{2}{\pi}\cdot \frac{|\theta(p)-\theta(q)|}{4}\right)^2
=
\frac{|\theta(p)-\theta(q)|^2}{\pi^2}.
\end{align}

Using the facts that $H^2(p,q)=H^2(q,p)$ and $H^2(p,q) \leq \KL(p,q)$,
\begin{equation}\label{eq:theta_hellinger}
|\theta(p)-\theta(q)|^2 \le \pi^2 \min \cbr{ \KL(q,p),\KL(p,q) }.
\end{equation}

Next we bound \(H^2(p,q)\) by \(2|p-q|\). Indeed, using $(\sqrt p+\sqrt q)^2 \geq p+q+2\sqrt{pq} \geq p+q \geq |p-q|$ to show
\[
(\sqrt p-\sqrt q)^2
=
\frac{(p-q)^2}{(\sqrt p+\sqrt q)^2}
\le |p-q|.
\]

Moreover, we can show
\[
(\sqrt{1-p}-\sqrt{1-q})^2 = \frac{(p-q)^2}{ \rbr{  \sqrt{1-p}+\sqrt{1-q} }^2  }\le |p-q|,
\]
where the last inequality follows from
\[
\rbr{  \sqrt{1-p}+\sqrt{1-q} }^2= (1-p)+(1-q)+2\sqrt{(1-p)(1-q)} \geq (1-p)+(1-q) \geq |p-q|.
\]

Therefore, we have
\begin{equation}\label{eq:hellinger_l1}
H^2(p,q)\le 2|p-q|.
\end{equation}
Combining \pref{eq:abs_diff_theta_hellinger} and \pref{eq:hellinger_l1} yields
\[
|\theta(p)-\theta(q)|^2 \le 2\pi^2 |p-q|.
\]

Combining all the above, we conclude the proof.
\end{proof}

\begin{lemma} \label{lem:bound_diff_log_loss_square}
Let $\loss(\cdot,\cdot)$ be the log loss. If $Y \sim \texttt{Ber}(p)$ and $x \in [\eta ,1-\eta]$, then
\[
\E_{Y \sim \mathrm{Ber}(p) } \sbr{ \rbr{ \loss(x,Y) -\loss(p,Y) }^2 } \leq 2(\log(1/\eta)+1) \KL(p,x).
\]
\end{lemma}
\begin{proof}
For any $y \leq \log(1/\eta)$,
\begin{equation} \label{eq:tool_ineq_1}
    y^2 \leq 2(\log(1/\eta)+1)(e^{-y}+y-1).
\end{equation}

We first verify $\loss(x,Y) -\loss(p,Y) \leq \log(1/\eta)$. To this end, it suffices to show that $\loss(x,0) -\loss(p,0) \leq \log(1/\eta)$ and $\loss(x,1) -\loss(p,1) \leq \log(1/\eta)$.
Indeed, one can use $x \in [\eta,1-\eta]$ to show that $\loss(x,0) -\loss(p,0)=\log( (1-p)/(1-x) ) \leq \log(1/(1-x)) \leq \log(1/\eta)$ and $\loss(x,1) -\loss(p,1)=\log(p/x ) \leq \log(1/x) \leq \log(1/\eta)$.

Then, we apply $ y =  \loss(x,Y) -\loss(p,Y)$ in \pref{eq:tool_ineq_1} and take expectation on both sides to get
\begin{align*}
    \E \sbr{   (\loss(x,Y) -\loss(p,Y))^2  } \leq 2(\log(1/\eta)+1) \E \sbr{ e^{ \loss(p,Y) - \loss(x,Y)}  + \loss(x,Y) -\loss(p,Y)   - 1  }.
\end{align*}

Let $\P_u$ be the Bernoulli law with parameter $u$.  One can show that
\[
\E \sbr {  \loss(x,Y) -\loss(p,Y) }= \E \sbr { \log\frac{\P_p(Y)}{\P_x(Y)} }= \KL(p,x),
\]
and
\[
\E \sbr{ e^{\loss(p,Y) -\loss(x,Y)  }  } = \E \sbr{ e^{ \log\frac{\P_x(Y)}{\P_p(Y)} }  } = \E \sbr{  \frac{\P_x(Y)}{\P_p(Y)}   }  =1.
\]

Combining the above equalities, we get the claimed bound.
\end{proof}

\subsection{Main Result for $\PKL$}

\begin{theorem} \label{thm:PKLCal_bound_knownC}
With $\Alg$ instantiated by Algorithm 1 of \citet{luo2025simultaneous}, $N_m$ given in \pref{eq:Nm_PKL}, $K$ given in \pref{eq:Kmq_PKL}, and $\eta=1/(T+1)$, \pref{alg:framework_non_uniform_PKLCal} ensures that with probability at least $1-\delta$,
\begin{align*}
\PKL  \leq \order \rbr{\iota \log^2 T+ (1+C)^{1/3}  \log^{\frac{7}{3}}  T }.
\end{align*}

Moreover, the computational complexity per round is $\otil \big((1+C)^{2/3}\big)$.
\end{theorem}

\begin{proof}
The computational complexity analysis follows that of \pref{thm:PCal2_bound}. The only difference is that the base algorithm employs the EWOO subroutine, which incurs $\order(1)$ cost per round by the analysis in \citep[Appendix C.1]{luo2025simultaneous}. Thus, the claimed computational complexity is immediate.

The following proof conditions on the nice event $\calE_3$, defined in \pref{def:nice_event_PKL}.
Let $\bar{\rho}_{m,p} = \frac{ \sum_{t \in \calT_m} \calP_t(p) y_t}{\sum_{t \in \calT_m} \calP_t(p)}$.
\pref{corr:decomp_Cal2_into_block} shows $ \PKL \leq \sum_{m=1}^M \PKL^{(m)}$, where $    \PKL = \sum_{t \in \calT_m} \E_{p \sim \calP_t } \sbr {   \KL \rbr{ \bar{\rho}_{m,p} , p } }$. Then, it suffices to bound each $\PKL^{(m)}$. By \citep[Proposition 1]{luo2025simultaneous}, minimizing $\PKL^{(m)}$ is equivalent to minimizing $\FSR^{(m)}$ with log loss $\loss(p,y) = -y\log p -(1-y)\log(1-p)$.
Recall that $\beta_{s,t}= \E_{a \sim H(a_{s,t})}[\loss(a,y_t)]-\loss(a_{s,t},y_t)$ and recall $\ExtReg_s^{(m)}$ from \pref{eq:pseudo_ext_reg}. Then, we write for any $m \geq 2$
\begin{align*}
\FSR^{(m)}& = \sup_{\phi:  [0,1] \to  [0,1]} \sum_{t \in \calT_m} \E_{s \sim \calP_t}\sbr{\loss(s,y_t)  - \loss(\phi(s),y_t)}\\
& = \sup_{\phi:  [0,1] \to  [0,1]} \sum_{t \in \calT_m} \sum_{s \in \calZ_m} \calP_t(s) \sbr{\loss(s,y_t)  - \loss(\phi(s),y_t)}\\
& =\sum_{s \in \calZ_m } \sup_{k \in [0,1]} \sum_{t \in \calT_m}  \calP_t(s)\rbr{  \loss(s,y_t)  - \loss(k,y_t)  }\\
& =\sum_{s \in \calZ_m } \sum_{t \in \calT_m}  \calP_t(s) \loss(s,y_t)  - \sum_{s \in \calZ_m }\inf_{k \in [0,1]} \sum_{t \in \calT_m}   \calP_t(s) \loss(k,y_t) \\
& =\sum_{s \in \calZ_m } \sum_{t \in \calT_m}  \sum_{s' \in \calZ_m } \rbr{\calP_t(s') \calQ_t(s',s) }\loss(s,y_t)  - \sum_{s \in \calZ_m }\inf_{k \in [0,1]} \sum_{t \in \calT_m}   \calP_t(s) \loss(k,y_t)\\
& =  \sum_{t \in \calT_m}  \sum_{s' \in \calZ_m }  \calP_t(s')   \E_{a \sim H(a_{s',t})}[\loss(a,y_t)]  - \sum_{s \in \calZ_m}\inf_{k \in [0,1]} \sum_{t \in \calT_m}   \calP_t(s) \loss(k,y_t) \\
&=  \sum_{t \in \calT_m}  \sum_{s' \in \calZ_m }  \calP_t(s')   \rbr{\loss(a_{s',t},y_t) + \beta_{s',t} }  - \sum_{s \in \calZ_m }\inf_{k \in  [0,1]} \sum_{t \in \calT_m}   \calP_t(s) \loss(k,y_t)  \\
&=   \sum_{t \in \calT_m} \sum_{s' \in \calZ_m  }    \calP_t(s') \beta_{s',t}   + \sum_{s \in \calZ_m  } \ExtReg^{(m)}_s \\
&\leq    \sum_{t \in \calT_m}   \sum_{s \in \calZ_m  }  \calP_t(s)\beta_{s,t}   +  \order \rbr{ |\Pi_m| \log T}  , \numberthis{} \label{eq:block_swap_reg_bound_step1_PKL}
\end{align*}
where the last inequality follows from \citep[Lemma 4]{luo2025simultaneous} and the fact that $|\calZ_m|=\order (|\Pi_m|)$.
Using \pref{lem:two_point_dist_theta} (our partition ensures $\Delta_J \leq \pi/2$ for all $J \in \Pi_m$) and a similar argument of \pref{eq:sum_pt_betast_bound}, we have that for any $m \geq 2$
\begin{align*}
 \sum_{t \in \calT_m}   \sum_{s \in \calZ_m }  \calP_t(s) \beta_{s,t}   \leq \sum_{t \in \calT_m}  \sum_{s \in \calZ_m }  \calP_t(s) \Delta_{J_{s,t}}^2 \Ind{J_{s,t}  \in \Pi_m^{\text{out}} } + \frac{ s_m |I_{m}|^2  }{ N_{m}^2 }.  \numberthis{} \label{eq:block_swap_reg_bound_step2_PKL}
\end{align*}

Recall that $C_m = \sum_{t \in \calT_m} c_t$.
Then, we show that
\begin{align*}
&\sum_{t \in \calT_m}  \sum_{s \in \calZ_m }  \calP_t(s)  \Delta_{J_{s,t}}^2  \Ind{J_{s,t}  \in \Pi_m^{\text{out}} } \\
& \leq \order \rbr{\frac{\sum_{t \in \calT_m}  \sum_{s \in \calZ_m }  \calP_t(s) |\theta(a_{s,t})-\theta(\optp)|^2}{\wt{K}^2} }  \\
& \leq \order \rbr{\frac{|\Pi_m|\log T +\log(T)C_m +\iota \log(T) }{\wt{K}^2}  } \\
& \leq \order \rbr{|\Pi_m|\log T +\iota \log(T) + \frac{\log(T)C_m }{\wt{K}^2}  },\numberthis{} \label{eq:block_swap_reg_bound_step3_PKL}
\end{align*}
where the first inequality uses \pref{lem:bound4_DeltaJ_abs_gap}, the second inequality follows from \pref{lem:theta_diff_basic_inequality} and \pref{lem:bound_occupancy_PKL} with $\eta=\frac{1}{T+1}$ verified by \pref{lem:range_ast}, and the third inequality bounds $\wt{K}\geq 1$.

Putting \pref{eq:block_swap_reg_bound_step2_PKL} and \pref{eq:block_swap_reg_bound_step3_PKL} into \pref{eq:block_swap_reg_bound_step1_PKL}, we have that for any $m\geq 2$
\begin{align*}
\PKL^{(m)} &=\FSR^{(m)}\\
&\leq \order \rbr{ \iota \log T + |\Pi_m| \log T       +\frac{|I_{m}|^2}{N_{m}^2} s_m +  \frac{C_m \log(T)}{\wt{K}^2} } \\
&= \order \rbr{ \iota \log T + \frac{|I_{m}|^2}{N_{m}^2} s_m+N_{m} \log T  + |\Pi_m^{\text{out}}| \log T   + \frac{C_m \log(T)}{\wt{K}^2}     } \\
&\leq \order \rbr{ \iota \log T+ \frac{|I_{m}|^2}{N_{m}^2} s_m+N_{m} \log T     + \wt{K}\log^2 T + \frac{C_m \log(T)}{\wt{K}^2}    } \\
&\leq \order \rbr{ \iota \log T + \frac{|I_{m}|^2}{N_{m}^2} s_m+N_{m} \log T     +  (1+C)^{1/3}  \log^{\frac{4}{3}}   T  + \frac{C_m \log^{\frac{7}{3}} T}{(1+C)^{2/3}}    },
\end{align*}
where the first inequality follows from \citep[Proposition 1]{luo2025simultaneous}, the second equality uses $|\Pi_m|=|\Pi_m^{\text{in}}| +|\Pi_m^{\text{out}}|$ together with $|\Pi_m^{\text{in}}|=N_{m}$, the third equality follows from $|\Pi_m^{\text{out}}|=\order (\wt{K}\log T )$, and the last inequality uses the choice of $\wt{K}$.

On the one hand, we have
\begin{align*}
   \sum_{m=2}^M   \rbr{\frac{|I_{m}|^2}{N_{m}^2} s_m+N_{m}\log T } &\leq \sum_{m=2}^M   \rbr{\frac{|I_{m}|^2}{\wt{N}_{m}^2} s_m+\wt{N}_{m}\log T } \\
   &\leq \order \rbr{   \log^2 T (\iota+C)^{1/3}    }.
\end{align*}
where the first inequality uses $N_{m} \in [\wt{N}_{m},2\wt{N}_{m}]$, and the last inequality applies \pref{lem:inner_error_PKLCal}.

We write
\begin{align*}
   \PKL & \leq \PKL^{(1)}+\sum_{m=2}^M\PKL^{(m)} .
\end{align*}

We first bound $\PKL^{(1)}= \sum_{t \in \calT_1} \E_{p \sim \calP_t} [\KL(\bar{\rho}_p^{(1)},p)]$, where $\bar\rho^{(1)}_p
=
\frac{\sum_{t\in\mathcal \calT_1} \calP_t(p)y_t}
{\sum_{t\in\mathcal \calT_1}\calP_t(p)} $. Using the fact that for any $p \in [\eta,1-\eta]$ and any $q \in [0,1]$, we have $\KL(q,p) \leq \log(1/\eta)$. Since the sampled prediction $p \sim \calP_t$ is supported on $\calZ_1 \subseteq [\eta,1-\eta]$ at each round,
\[
\PKL^{(1)}=\sum_{t \in \calT_1} \E_{p \sim \calP_t} [\KL(\bar{\rho}_p^{(1)},p)] \leq  s_1 \log(T+1) \leq \order(\log T).
\]

On the other hand, using $M=\order(\log T)$, summing over all epochs $m$, and bounding $\PKL^{(1)}=\order(\log T)$ we have,
\begin{align*}
 \PKL & \leq \PKL^{(1)}+\sum_{m=2}^M\PKL^{(m)}  \\
 &\leq  \order \rbr{    \log^2 T (\iota+C)^{1/3}   } +\sum_{m=2}^M \rbr{\iota \log T+ (1+C)^{1/3}  \log^{\frac{4 }{3}}  T  + \frac{C_m \log^{\frac{7}{3}} T}{(1+C)^{2/3}}   } \\
& \leq \order \rbr{\iota \log^2 T+ (1+C)^{1/3}  \log^{\frac{7}{3}}  T }.
\end{align*}

The proof is thus complete.
\end{proof}

\section{Omitted Details for $\Caldist_1$ with Improved Time Complexity}

\subsection{Technical Lemmas}

\begin{lemma} \label{lem:cal1_better_inner}
We have
\begin{align*}
   \sum_{m=2}^M\rbr{\iota N_m  + \sqrt{ \iota N_m s_m } + s_m \frac{|I_m|}{N_m}}  \leq \order \rbr{ \iota \log T+  \sqrt{\iota T} + (\iota TC)^{1/3} }.
\end{align*}
\end{lemma}

\begin{proof}
Using the choice of $N_m=\left \lceil \rbr{s_m}^{1/3} |I_{m}|^{2/3} \iota^{-1/3} \right \rceil$, we have for each $m$
\[
\iota N_m  + \sqrt{ \iota N_m s_m } + s_m \frac{|I_m|}{N_m} \leq \order \rbr{ \iota^{1/3}s_m^{2/3} |I_m|^{1/3} +\iota^{2/3} s_m^{1/3} |I_m|^{2/3} +\sqrt{ \iota s_m} +\iota  },
\]
where the last two terms are used to handle ceiling effects.
Notice that $\sum_{m \geq 2} (\sqrt{\iota s_m} +\iota) \leq \order (\sqrt{\iota T} +\iota \log T)$. We then focus on the remaining two terms.
Since $s_{m+1}=2s_m$, we have
\[
|I_m| \leq \order \rbr{ \min \cbr{1,\sqrt{ \frac{\iota}{s_{m-1}} }  +\frac{C}{s_{m-1}} } }  \leq \order \rbr{ \min \cbr{1,\sqrt{ \frac{\iota}{s_{m}} }  +\frac{C}{s_{m}} } } .
\]

For each $m \geq 2$, we consider three cases for $\iota^{1/3}s_m^{2/3} |I_m|^{1/3} $.

\textbf{Case 1: $s_m\leq C$.} We simply bound $|I_m| \leq 1$ to show $\iota^{1/3}s_m^{2/3} |I_m|^{1/3}  \leq \iota^{1/3}s_m^{2/3}$. Then,
\begin{align*}
\sum_{m:s_m \leq C}\iota^{1/3}s_m^{2/3} |I_m|^{1/3}  \leq \order  \rbr{ \iota^{1/3}C^{2/3}  }  \leq \order  \rbr{ (\iota TC)^{1/3}  }.
\end{align*}

\textbf{Case 2: $C< s_m\leq C^2/\iota$.} This case only exists when $C \geq \iota$, and then $\sqrt{\iota/s_m} \leq C/s_m$. Hence, $|I_m| \leq \order(C/s_m)$, which gives
\begin{align*}
\sum_{m:C< s_m\leq C^2/\iota}\iota^{1/3}s_m^{2/3} |I_m|^{1/3}  \leq \order  \rbr{ \sum_{m:C< s_m\leq C^2/\iota} \iota^{1/3}s_m^{2/3} \rbr{ \frac{C}{s_m}   }^{1/3}  }  \leq \order  \rbr{ (\iota TC)^{1/3}  }.
\end{align*}

\textbf{Case 3: $ s_m >\max\{ C^2/\iota,C\}$.} In this case, $C/s_m \leq \sqrt{\iota/s_m}$, and thus $|I_m| \leq \order(\sqrt{\iota/s_m})$.
\begin{align*}
\sum_{m:s_m >\max\{ C^2/\iota,C\}}\iota^{1/3}s_m^{2/3} |I_m|^{1/3}  \leq \order  \rbr{\sum_{m:s_m >\max\{ C^2/\iota,C\}}  \iota^{1/3}s_m^{2/3} \rbr{\sqrt{ \frac{\iota}{s_m}}   }^{1/3}  }  \leq \order  \rbr{ \sqrt{\iota T}  }.
\end{align*}

Thus,
\[
\sum_{m=2}^M \iota^{1/3}s_m^{2/3} |I_m|^{1/3} \leq \order  \rbr{ \sqrt{\iota T} + (\iota TC)^{1/3}   }.
\]

For $\iota^{2/3} s_m^{1/3} |I_m|^{2/3}$, we consider two cases. If $s_m \leq \iota$, then $\iota^{2/3} s_m^{1/3} |I_m|^{2/3} \leq \iota$. Otherwise, $\iota^{2/3} s_m^{1/3} |I_m|^{2/3} \leq \iota^{1/3} s_m^{2/3} |I_m|^{1/3}$. Combining all the above, we get the claimed bound.
\end{proof}

\subsection{Main Results}

\begin{theorem} \label{thm:cal1_better}
With  $\Alg$ instantiated as \pref{alg:cal1_knownC}, $N_m=\big \lceil \rbr{s_m}^{1/3} |I_{m}|^{2/3} \iota^{-1/3} \big \rceil$,
and $K =\big\lceil \rbr{(1+C)^2/(\iota T) }^{1/3}  \big \rceil$,
\pref{alg:framework_non_uniform} ensures that with probability at least $1-\delta$, $\Caldist_1 = \order \rbr{  \iota \log T+  \sqrt{\iota T} + (\iota TC)^{1/3}   }.$ The computational complexity per round is $\otil \big( (1+C)^{1/3}\big)$.
\end{theorem}

\begin{proof}
Since the same base algorithm is used in \pref{alg:framework_non_uniform} and the choice $K=\big\lceil \rbr{(1+C)^2/(\iota T) }^{1/3}  \big \rceil \leq \order(\big\lceil \rbr{(1+C)/\iota }^{1/3}  \big \rceil)$ (recall in \pref{thm:cal2_approach1}, the algorithm uses $K= \big\lceil \rbr{(1+C)/\iota }^{1/3}  \big \rceil$), repeating the argument in \pref{thm:cal2_approach1} directly gives the claimed result.

The following proof conditions on the nice event $\calE_1$ defined in \pref{def:nice_event_Cal2_approach1}.
 From \pref{corr:decomp_Cal2_into_block}, we have $\Caldist_1 \leq \sum_{m=1}^M \Caldist_1^{(m)}$, where
\begin{align*}
    \Caldist_1^{(m)} = \sum_{J \in \Pi_m} \abr{   \sum_{t \in \calT_m} (y_t-z_J) \Ind{p_t=z_J}    }.
\end{align*}

Consider any epoch $m\geq 2$.
By \pref{lem:block_version_bound}, for each $J \in \Pi_m$,
\begin{align*}
    \abr{   \sum_{t \in \calT_m} (y_t-z_J) \Ind{p_t=z_J}    } \leq \order \rbr{ \sqrt{ \iota \cdot  n_{m,J}  } + \iota + n_{m,J}  \Delta_J } ,
\end{align*}
which implies
\begin{align*}
  \Caldist_1^{(m)} &\leq  \order \rbr{ \iota |\Pi_m| + \sum_{J \in \Pi_m} \rbr{ \sqrt{ \iota \cdot  n_{m,J}  }+ n_{m,J}  \Delta_J  } } \\
   &\leq  \order \rbr{\iota |\Pi_m^{\text{in}}| + \sum_{J \in \Pi_m^{\text{in}}} \rbr{ \sqrt{ \iota \cdot  n_{m,J}  }+ n_{m,J}  \Delta_J  }  } \\
   &\quad + \order \rbr{ \iota |\Pi_m^{\text{out}}| + \sum_{J \in \Pi_m^{\text{out}}} \rbr{ \sqrt{ \iota \cdot  n_{m,J}  }+ n_{m,J}  \Delta_J  }  } .
\end{align*}

For each $J \in \Pi_m^{\text{in}}$, $\Delta_J = \frac{|I_m|}{N_m}$, and thus
\begin{align*}
&\sum_{ m \geq 2 } \rbr{ \iota |\Pi_m^{\text{in}}| + \sum_{J \in \Pi_m^{\text{in}}} \rbr{ \sqrt{ \iota \cdot  n_{m,J}  }+ n_{m,J}  \Delta_J  }}\\
&  \leq \order  \rbr{\sum_{ m \geq 2 } \rbr{ \iota N_m  + \sqrt{ \iota N_m s_m } + s_m \frac{|I_m|}{N_m} }} \\
& \leq \order \rbr{ \iota \log T+  \sqrt{\iota T} + (\iota TC)^{1/3} },
\end{align*}
where the last inequality applies \pref{lem:cal1_better_inner}.

On the other hand, we have
\begin{align*}
   & \iota |\Pi_m^{\text{out}}| + \sum_{J \in \Pi_m^{\text{out}}} \rbr{ \sqrt{ \iota \cdot  n_{m,J}  }+ n_{m,J}  \Delta_J  } \\
   &\leq \order \rbr{   \iota \sum_{q=0}^Q K + \sum_{J \in \Pi_m^{\text{out}}} \rbr{ \sqrt{ \iota \cdot  n_{m,J}  }+ n_{m,J}  \Delta_J  }  } \\
    &\leq \order \rbr{    \sum_{q=0}^Q  \rbr{ \iota K +  \sum_{J \in \calL_{m,q}} \rbr{ \sqrt{ \iota \cdot  n_{m,J}  }+ n_{m,J}  \Delta_J  }
    +  \sum_{J \in \calR_{m,q}} \rbr{ \sqrt{ \iota \cdot  n_{m,J}  }+ n_{m,J}  \Delta_J  }    } }.
\end{align*}

As the idea of bounding left and right outer regions is the same, we here take the left outer region as an example. For any non-empty $J \in \calL_{m,q}$, we have $d_J \geq  2^q r_{m}$ and $\Delta_J = \frac{2^q r_{m}}{K}$.
Recall that $C_{m,J} = \sum_{t\in \calT_m} c_t  \calP_t \rbr{z_{J}}$. Based on this, we further define
\begin{align} \label{eq:def_Cmq}
    C_{m,q}: = \sum_{J \in \calL_{m,q}} \sum_{t\in \calT_m} c_t  \calP_t \rbr{z_{J}} + \sum_{J \in \calR_{m,q}} \sum_{t\in \calT_m} c_t  \calP_t \rbr{ z_{J}}.
\end{align}

We use \pref{lem:visit_number_control_Cal2_block_main} to write
\begin{align*}
&  \sum_{J \in \calL_{m,q}} \rbr{ \sqrt{ \iota \cdot  n_{m,J}  }+ n_{m,J}  \Delta_J  }  \\
&\leq \order \rbr{  \sum_{J \in \calL_{m,q}} \rbr{ \sqrt{ \iota  \rbr{  \frac{\iota}{ (2^q r_m)^2  }  +\frac{C_{m,J}}{2^qr_m}  }     }+ \rbr{  \frac{\iota}{ 2^q r_m K  }  +\frac{C_{m,J}}{K}  }     }  }  \\
&\leq \order \rbr{ \frac{\iota K}{ 2^q r_m  } +   \frac{\iota}{ 2^q r_m  }  +\frac{C_{m,q}}{K} +  \sum_{J \in \calL_{m,q}}  \sqrt{  \frac{\iota  C_{m,J}}{2^qr_m}       }        } \\
&\leq \order \rbr{ \frac{\iota K}{ 2^q r_m  } +   \frac{\iota}{ 2^q r_m  }  +\frac{  C_{m,q}}{K} +    \sqrt{  \frac{\iota KC_{m,q}}{2^qr_m}       }        }
\end{align*}

Thus,
\begin{align*}
 &   \iota K + \sum_{J \in \calL_{m,q}} \rbr{ \sqrt{ \iota \cdot  n_{m,J}  }+ n_{m,J}  \Delta_J  } \\
  & \leq   \order \rbr{ \frac{\iota K}{ 2^q r_m  } +   \frac{\iota}{ 2^q r_m  }  +\frac{C_{m,q}}{K} +    \sqrt{  \frac{\iota  KC_{m,q}}{2^qr_m}       } + \iota K       } \\
    & \leq   \order \rbr{ \frac{\iota K^2}{ 2^q r_m  }   +\frac{C_{m,q}}{K} +    \sqrt{  \frac{\iota  KC_{m,q}}{2^qr_m}       }       } \\
        & \leq   \order \rbr{ \frac{\iota K^2}{ 2^q r_m  }   +\frac{C_{m,q}}{K}        },
\end{align*}
where the last inequality uses $2\sqrt{ab}\leq a+b$ with $a=\frac{C_{m,q}}{K}$ and $b=\frac{\iota K^2}{ 2^q r_m  }$.

Hence, we have
\begin{align*}
&\sum_{ m \geq 2 }\rbr{ \iota |\Pi_m^{\text{out}}| + \sum_{J \in \Pi_m^{\text{out}}} \rbr{ \sqrt{ \iota \cdot  n_{m,J}  }+ n_{m,J}  \Delta_J  }} \\
&\leq \order \rbr{\sum_{ m \geq 2 } \sum_{q=0}^Q \rbr{   \frac{\iota K^2}{ 2^q r_m  }   +\frac{C_{m,q}}{K}    }  } \\
&\leq \order \rbr{\sum_{ m \geq 2 }   \frac{\iota K^2}{ r_m  }   +\frac{C}{K}      } \\
&\leq \order \rbr{ \sum_{ m \geq 2 }   \frac{\iota K^2}{ r_m  }  + (\iota TC)^{1/3}  },
\end{align*}
where the second inequality uses $\sum_q (2^q r_m)^{-1} \leq \order(r_m^{-1})$, and the last inequality follows from the choice of $K$.
Then, we bound $\sum_{ m \geq 2 }   \frac{\iota K^2}{ r_m  }$ by considering two cases.
If $C \leq \sqrt{\iota T}$, then $K=\order(1)$. Since $r_m \geq \sqrt{\iota/s_{m-1}}$, using $\sum_q (2^q r_m)^{-1} \leq \order(1/r_m)$ gives
$\sum_{m=2}^M \frac{\iota K^2}{r_m} \leq \order \rbr{\sum_{m=2}^M   \frac{\iota}{r_m}  } \leq \order (\sqrt{\iota T}).$
If $C >\sqrt{\iota T}$, we lower-bound $r_m \geq \frac{C}{s_{m-1}} = \frac{2C}{s_m}$, which implies $\sum_{m} r_m^{-1} \leq \order(T/C)$. Then,
\[
\sum_{m=2}^M \frac{\iota K^2}{  r_m}   \leq \order \rbr{ \frac{\iota K^2T}{C}  } \leq \order \rbr{ (\iota TC)^{1/3}  },
\]
where the last inequality uses the choice of $K$.
Combining both cases, we have $\sum_{ m \geq 2 }   \frac{\iota K^2}{ r_m  } \leq \order(\sqrt{\iota T} +(\iota TC)^{1/3})$.

Putting all the above and bounding $\Caldist_1^{(1)}=\order(1)$, we get the claimed bound on $\Caldist_1$.
\end{proof}

\end{document}